\documentclass[11pt]{article}

\usepackage{mathtools,bbm,xspace}
\usepackage{bm}
\usepackage{xcolor}
\usepackage{thmtools}
\usepackage{thm-restate}
\usepackage{hyperref}
\usepackage{float}
\usepackage[capitalise,nameinlink]{cleveref}
\usepackage{algorithm}
\usepackage{algpseudocode}
\usepackage{amsmath}
\usepackage{amssymb}
\usepackage{natbib}
\usepackage{graphicx}
\usepackage{url}
\usepackage{amsthm}
\usepackage{natbib}
\usepackage{authblk}
\usepackage[left=0.8in, right=0.8in]{geometry}

\newtheorem{definition}{Definition}
\newtheorem{lemma}{Lemma}
\newtheorem{theorem}{Theorem}

\crefname{algorithm}{Algorithm}{Algorithms}
\crefname{theorem}{Theorem}{Theorems}
\crefname{lemma}{Lemma}{Lemmas}
\crefname{proposition}{Proposition}{Propositions}
\crefname{corollary}{Corollary}{Corollaries}
\crefname{claim}{Claim}{Claims}
\crefname{definition}{Definition}{Definitions}
\crefname{example}{Example}{Examples}
\crefname{remark}{Remark}{Remarks}

\hypersetup{
    colorlinks=true,
    linkcolor=blue,
    filecolor=blue,
    urlcolor=blue,
    citecolor=blue
}

\newcommand{\e}{\mathbb{E}}
\newcommand{\eps}{\varepsilon}
\newcommand{\p}{\mathbb{P}}
\newcommand{\C}{C}

\newcommand{\ind}{\mathbbm{1}}
\DeclareMathOperator{\sign}{sign}
\newcommand{\dlh}{\Delta^{t}(\cH)}

\DeclareMathOperator{\ls}{\mathcal{L}}
\newcommand{\as}{\text{AdaBoostSample }}
\newcommand{\erm}{\mathrm{ERM}}
\newcommand{\loss}{\epsilon}
\newcommand{\alphar}{\bm{\alpha}}

\DeclareMathOperator{\Utrain}{U_{T}}
\DeclareMathOperator{\Uinf}{U_{I}}
\newcommand{\rh}{\mathbf{\tilde{h}}}
\newcommand{\rhh}{\mathbf{h}}
\newcommand{\rf}{\tilde{\mathbf{f}}}
\newcommand{\cs}{c_s} 

\newcommand{\cnn}{40} 

\newcommand{\cff}{40}

\newcommand{\rr}{\mathbf{r}}
\newcommand{\cu}{C_{U}}
\newcommand{\ah}{\mathcal{\hat{A}}}

\newcommand{\cX}{\mathcal{X}}
\newcommand{\cH}{\mathcal{H}}
\newcommand{\cD}{\mathcal{D}}
\newcommand{\cA}{\mathcal{A}}
\newcommand{\rS}{\mathbf{S}}
\newcommand{\cS}{\mathcal{S}}
\newcommand{\cY}{\mathcal{Y}}
\newcommand{\rx}{\mathbf{x}}
\newcommand{\rw}{\mathbf{w}}

\newcommand{\rz}{\mathbf{z}}
\newcommand{\rX}{\mathbf{X}}
\newcommand{\ry}{\mathbf{y}}
\newcommand{\cB}{\mathcal{B}}
\newcommand{\ri}{\mathbf{i}}
\newcommand{\rN}{\mathbf{N}}
\newcommand{\cW}{\mathcal{W}}
\newcommand{\cG}{\mathcal{G}}
\newcommand{\rJ}{\mathbf{J}}
\newcommand{\cF}{\mathcal{F}}
\newcommand{\rD}{\mathbf{D}}

\newcommand{\rB}{\mathbf{B}}

\title{Efficient Optimal PAC Learning}
\date{}
\author{ \textcolor{white}{ blaba}\\Mikael {M\o ller H\o gsgaard}\thanks{We thank Kasper Green Larsen for valuable conversations and his support. We also thank the anonymized reviewers at ALT 2025 for their comments helping improve the paper. Supported by Independent
Research Fund Denmark (DFF) Sapere Aude Research Leader Grant No. 9064-00068B. } \\ Aarhus University\\hogsgaard@cs.au.dk }

\begin{document}
\maketitle

\begin{abstract}
  Recent advances in the binary classification setting by \cite{hannekeoptimal} and \cite{baggingoptimalPAClearner} have resulted in optimal PAC learners. These learners leverage, respectively, a clever deterministic subsampling scheme and the classic heuristic of bagging \cite{Breiman1996BaggingP}. Both optimal PAC learners use, as a subroutine, the natural algorithm of empirical risk minimization. 
Consequently, the computational cost of these optimal PAC learners is tied to that of the empirical risk minimizer algorithm.

In this work, we seek to provide an alternative perspective on the computational cost imposed by the link to the empirical risk minimizer algorithm.
To this end, we show the existence of an optimal PAC learner, which offers a different tradeoff in terms of the computational cost induced by the empirical risk minimizer.
\end{abstract}

\section{Introduction}
The PAC model (Probably Approximately Correct) was introduced by \cite{valiant1984theory,vapnik1964class,vapnik74theory} and has since been a cornerstone in the theory of machine learning. The model is based on the idea that a learning algorithm should be able to learn from a set of labelled training examples such that it can predict the label of unseen examples with high accuracy. In the following, we will consider the realizable setting of PAC learning, that is, we assume that the true concept $ c\in\left\{ -1,1 \right\}^{\cX} $  is contained in a hypothesis class $ \cH \subset \left\{ -1,1 \right\}^{\cX}$, and that the hypothesis class has finite VC-dimension $ d $. Finite VC-dimension $ d $ of $ \cH $ implies that there exists a point set $ x_1,\ldots,x_{d} $, where every pattern in $ \left\{ -1,1 \right\}^{d} $ on the points $ x_{1},\ldots,x_{d} $, can be realized by a hypothesis in the hypothesis class $ \cH $. However, for any point set of more than $ d $ points, this is not the case (\cite{vapnik1971uniform}). Formally, we have the following definition of PAC learning:
\begin{definition}
    Let $\cX$ be an input space and $\cH\subset \left\{ -1,1 \right\}^{\cX}$ a hypothesis class of VC-dimension $ d.$ A learning algorithm $\cA$ is a PAC learner for $\cH$ if there exists a sample complexity function $m_{\cA}:(0,1)^{2}\times \mathbb{N}\rightarrow\mathbb{N}$ such that for every $\epsilon,\delta\in(0,1)$, every distribution $\cD$ over $\cX$, every target concept $ c\in \cH $ and $m\geq m_{\cA}(\epsilon,\delta,d),$ it holds with probability at least $1-\delta$ over a training sequence $\rS=(\rS_1,c(\rS_1),\ldots,\rS_{m},c(\rS_{m}))\sim\cD_{c}^{m}$ with $\rS_{i} \sim\cD ,$ that $ f=\cA(\rS) \in\left\{ -1,1 \right\}^{\cX}$ is such that the error $$\ls_{\cD_{c}}(f):=\p_{\rx\sim \cD_{c}}\left[f(\rx)\not=c(\rx)\right]\leq\epsilon.$$
\end{definition}
The first, and also natural, approach tried out to get PAC-learners with good sample complexity was to do empirical risk minimization ($\erm$). In this approach, given a training sequence $   \rS, $ the PAC learner outputs a hypothesis $ h $  in $ \cH $ that is consistent with the training sequence, i.e., for every $ (\rx,\ry)\in \rS $, the hypothesis $ h $ satisfies $ h(\rx)=\ry $. The following lemma, due to \cite{vapnik1968algorithms,vapnik1971uniform,blumer1989learnability}, gives a uniform error bound on consistent functions in $ \cH$, and thus ensures that any $ \erm $ algorithm guarantees that error bound. Furthermore, error bounds give rise to determining a sample complexity for a given $ \eps $ by solving for $ m $ such that the error bound implies $ \ls_{\cD_{c}}(\cA(\rS))\leq \eps $.
\begin{lemma}\label{uniformconvergencemain}[\cite{vapnik74theory}, \cite{Blumeruniformconvergence} from \cite{Simons}[Theorem 2]]
    For $0<\delta,\eps<1$, hypothesis class $\cH$ of VC-dimension $d$, target concept $c\in\cH$, and distribution $\cD$ over $\cX$, we have with probability at least $1-\delta$ over $\rS\sim \cD_{c}^{m}$, that all $h\in \cH$ consistent with $ S $, i.e.,  $\ls_{\rS}(h):= \tfrac{
        1 }{m}\sum_{(x,y)\in \rS} \ind\left\{ h(x)\not=y \right\}  =0$,
       have error  $$
       \ls_{\cD_{c}}(h)\leq 2(d\log_{2}{\left(2em/d \right)}+\log_{2}{\left(2/\delta \right)})/m.
       $$
   \end{lemma}
   As described above \cref{uniformconvergencemain}, we can transform this error bound into a sample complexity bound of $ O(\left(d\ln{\left(1/\eps \right)}+\ln{\left(1/\delta \right)}\right)/\eps).$  
   This sample complexity bound is known to be tight for general $ \erm $ algorithms, in the sense that for any $0< \eps ,\delta<1/100$ and $ d\in\mathbb{N} $, there exists a hypothesis class $ \cH $ of VC-dimension $ d $, input space $ \cX,$ and distribution $ \cD_{c} $  over $ \cX $  for $ c\in \cH $, such that the sample complexity is $ \Omega((d\ln{\left(1/\eps \right)}+\ln{\left(1/\delta \right)})/\eps) $, where this lower bound construction is due to \cite{bousquet2020proper}[Theorem 11] (See also \cite{haussler1994predicting,auer2007new,simon2015almost,hanneke2016refined}). The above lower bound by \cite{bousquet2020proper} holds in general for any learner outputting a hypothesis in $ \cH $, known as proper learners. In contrast, the lower bound for general learners (not necessarily proper) is $ \Omega((d+\ln{\left(1/\delta \right)})/\eps),$ due to \cite{EHRENFEUCHT1989247}. A natural question to ask is thus whether there is a gap between the performance of proper and improper learners.

This long-standing open problem was finally resolved by \cite{hannekeoptimal}, who showed that there exists an improper learner that, with probability at least $ 1-\delta $ over the given training sequence $ \rS\sim \cD_{c}^{m} $,  produces a function $ f\in \left\{ -1,1 \right\}^{\cX} $, which, obtain the optimal PAC generalization error bound 
$$ \ls_{\cD_{c}}(f) =\Theta\left((d+\ln{\left(1/\delta \right)})/m \right),$$
or equivalently a sample complexity of $O( (d+\ln{\left(1/\delta \right)})/\eps).$   
 The improper learner of \cite{hannekeoptimal} is a  majority vote of $ m^{\log_{4}(3)} \approx m^{0.79}$ voters, with the voters being obtained by running a $ \erm $ algorithm on $m^{\log_{4}(3)}$ carefully chosen sub training sequences of $ \rS $, each of size $ 2m/3 $. Furthermore, \cite{baggingoptimalPAClearner} showed that Bagging, introduced by \cite{Breiman1996BaggingP}, is also an optimal PAC learner. The PAC learner is a majority vote of the output of a $ \erm $ algorithm trained on $ 18\ln{\left(2m/\delta \right)} $ many bootstrap samples of any size between $ 0.02m $ and $ m $, giving a more efficient optimal PAC learner. From the structure of these optimal PAC learners, we notice that their computational efficiency is closely tied to that of the $\erm$ learner run on $ \Theta(m) $ training examples. Thus in the case that, the computational cost of the $\erm$ learner scales poorly in the input size, this carries over to the computational cost of the aforementioned PAC learners, scaling with the number of examples $\Theta(m)$ provided for learning. This raises two natural questions: 
\vspace{-0.075cm}
\begin{enumerate}
    \item Can one obtain an optimal PAC learner that always queries the $ \erm $ learner with fewer than $ m $  training examples?\label{question1}
    \item\vspace{-0.2cm} If possible, can it be done to such a degree that the overall computational complexity of the optimal PAC learner becomes linear in the number of training examples $ m $?\label{question2}
\end{enumerate}
\vspace{-0.075cm}

To address these questions we now introduce some notation. In what follows we assume that the following operations cost one unit of computation: reading a number, comparing two numbers, adding, multiplying, calculating $\exp(\cdot)$ and $\ln{\left(\cdot \right)}$, and renaming. 

Furthermore, we assume the learner has query access to a $ \erm $-algorithm. To compare different learners computational cost, we associate to the $ \erm $ algorithm two cost functions: the worst-case cost of training $ \erm(S) $  on a consistent training sequence $ S $ of length $ m $, denoted  $ \text{U}_{\text{Train}}(m):=\Utrain(m)$
$$\Utrain(m):=\sup_{\stackrel{S\in (\cX\times\cY)^{m}}{S \text{ consistent with } \cH}} \# \{ \text{Operations to}\text{ find } \erm(S)  \},$$
and the worst-case inference cost of calculating the value $ h(x) $  for any point $ x\in \cX $ of any outputted hypothesis, $ h=\erm(S) $, with $ S $ being a consistent training sequence,  denoted  $ \text{U}_{\text{Inference}}:= \Uinf$\footnote{We have for simplicity defined the inference cost as the worst case output for any realizable $ \rS $, so note depending on the number of examples $  m$  in $ \rS $, as in training complexity. If we have gone with this more refined notion the inference cost of \cite{hannekeoptimal} and \cite{baggingoptimalPAClearner} would be $ U_{I}(\Theta(m)) $ and that of \cref{introductionmaintheorem} would be $ U_{I}(550d). $    },

$$\Uinf=\sup_{\stackrel{h=\erm(S), S\in (\cX\times\cY)^{*}}{S \text{ consistent with } \cH,\text{ }x\in\cX}} \# \{ \text{Operations to calculate } h(x)  \}.$$

For a learning algorithm $ \cA:(\cX\times \{ -1,1 \} )^{*}\rightarrow \{ -1,1 \}^{\cX} $, we define the training complexity of $ \cA $ for an integer $ m $ as the worst case number of operations made by the learning algorithm when given a realizable training sequence by $ \cH $ of length $ m $ , i.e.  $$\sup_{\stackrel{S\in(\cX\times \{-1,1\})^{m}}{ S \text{ realizable by } \cH  }} \# \{\text{Operations to find } \cA(S)\}.$$ The inference complexity of a learning algorithm $ \cA $ for an integer $ m $ we define as the worst case cost of predicting a new point $ x\in \cX $ for the learned mapping $f= \cA(S) $, given a realizable training sequence by $ \cH $ of length $ m $   i.e. $$ \sup_{\stackrel{f=\cA(S),S\in (\cX\times \{-1,1\})^{m}}{ S \text{ is realizable by } \cH, \text{ } x\in \cX}} \#\{\text{Operations to calculate } f(x)\}.$$

Using this notation, \cite{hannekeoptimal} has a training complexity of at least $ m^{0.79}\Utrain(2m/3) $ and an inference complexity of at least $ m^{0.79}\Uinf $, whereas \cite{baggingoptimalPAClearner} has a training complexity of at least $ 18\ln{\left(2m/\delta \right)}\cdot\Utrain(0.02m) $ and an inference complexity of $ 18\ln{\left(2m/\delta \right)}\Uinf$. We now state our main result which gives a different tradeoff in terms of training and inference complexity.

\begin{theorem}[Informal statement of \cref{maintheorem}]\label{introductionmaintheorem}
There exists an algorithm $ \ah $ such that for $ 0< \delta <1 $, hypothesis class $ \cH $ of VC-dimension $ d $, target concept $ c\in \cH $, and access to a $ \erm $-learner, that with probability at least $ 1-\delta $ over $ \rS \sim \cD_{c}^{m}$ and the randomness of $ \ah $, it holds that $ \ah $ has, error $ \ls_{\cD_{c}}(\ah)=O((d+\ln{\left(1/\delta \right)})/m),$ inference complexity $ O(\ln{[m /\delta(d+\ln{\left(1/\delta \right)}) ]}) \Uinf $, and  training complexity 
\begin{align*}
    O\bigg(\negmedspace\ln{\bigg(\frac{m}{\delta(d+\ln{\left(1/\delta \right)})} \bigg)}\negmedspace\cdot\negmedspace\ln{\bigg(\frac{m}{\delta} \bigg)}\negmedspace\bigg)\negmedspace\cdot \negmedspace\Big(O\Big(m+d\ln{\left(m \right)}\Big)\negmedspace+\negmedspace\Utrain(550d)\negmedspace+\negmedspace 3m\Uinf\negmedspace\Big).
\end{align*}
         
\end{theorem}

As a result of the above, \textbf{we answer the first question affirmatively} by achieving PAC optimality, using only $550d$ points in each call to the $\erm$ algorithm.
Furthermore, we also notice that if we consider our training complexity as a function of $ m $, thinking of $ \delta $, $ d $, $ \Utrain(550d) $,  and $ \Uinf $  as fixed, our training complexity is up to a $ \ln^{2}{\left(m \right)} $-factor linear in $ m $. Consequently, \textbf{we almost answer the second question affirmatively} under these assumptions. However, it is not always the case that these quantities are fixed as a function of $ m,$ since we are in the distribution-free setting, as we will see in the following paragraph for perceptron. 
Additionally, we also note that the inference complexity of the algorithm is \textbf{asymptotically better} than that of \cite{hannekeoptimal} and \cite{baggingoptimalPAClearner}\footnote{Looking at the proof of \cite{baggingoptimalPAClearner} we think that \cite{baggingoptimalPAClearner} could obtain the same inference complexity.}.  

Our training complexity is a multiplicative factor of  $ O(\ln{[m/(\delta(d+\ln{(1/\delta )}))]}\cdot(m+d\ln{(m )}+\Utrain(550d)+3m\Uinf))/\Utrain(0.02m)$ different from that of \cite{baggingoptimalPAClearner}, and only a smaller multiplicative factor different from that of \cite{hannekeoptimal}. Thus, if $ \delta$ is some polynomial in $\Omega(1/m^{C}) $ for some fixed constant $ C \geq 1$, the term in the denominator is $O( \ln{\left(m \right)}(m+d\ln{\left(m \right)}+\Utrain(550d)+3m\Uinf) )$, which in the case $ \Utrain(0.02m)=\omega(\ln{\left(m \right)}\cdot\max ( m , d\ln^{2}{\left(m \right)} , \Utrain(550d), 3m \Uinf )) $,  implies that we get an asymptotically better training complexity. If we think of $ d,\Utrain(550d) $  and $ \Uinf $ as fixed, $ \Utrain(0.02m) $ has to be $ \omega(m\ln{\left(m \right)}) $ for \cref{introductionmaintheorem} to be better, however, as commented above this is not always the case.

We now give two examples of where \cref{introductionmaintheorem} might have a better training complexity. In the following we will for simplicity set $ \delta $ equal to a small constant, thus the training complexity of \cref{introductionmaintheorem} holds with this probability.

\paragraph{Search over a hypothesis space with finite VC-dimension:}             
Consider a hypothesis class $\cH$ of VC-dimension $d$. Furthermore, for a training sequence $S = ((x_{1}, y_{1}), \ldots, (x_{n}, y_{n}))$, we let $\cH|S = \{(h(x_{1}), \ldots, h(x_{n})) \mid h \in \cH\}$ denote the possible projections of $\cH$ on $S$. We assume that the black-box $\erm$ algorithm checks one projection of $\cH|S$ at a time and, if the projection is realizable, outputs any hypothesis that realizes this projection. 
We assume that performing inference on one point $x_{i}$ for a projection has a cost of $U_{I} = O(1)$. Thus, the training complexity becomes $U_{T}(n) = O(n (n/d)^{O(d)})$ in the worst case, as the $\erm$ might find a projection that realizes $S$ among the last projections in $\cH|S$, which can have size $(n/d)^{O(d)}$. 
In this case, the training complexity of \cref{introductionmaintheorem} becomes $O(\ln{\left(\frac{m}{d}\right)} \ln{\left(m\right)} (d 500^{d} + m))$, and the training complexity of \cite{baggingoptimalPAClearner} becomes $O(\ln{\left(m\right)} m (m/d)^{d})$, where \cite{baggingoptimalPAClearner} is better for $m \leq \Theta(550d \ln^{1/d}{\left(m\right)})$.
\paragraph{Perceptron:}
In the \cref{appendix:perceptron}, we show that for $m \in \mathbb{N}$ and $m \geq 2200$, there exists a realizable distribution $\cD$ where \cite{baggingoptimalPAClearner}, run with the perceptron algorithm as a black box $\erm$, has a training complexity of $\Omega(m^2)$ with probability at least $1/2$. Similarly, \cref{introductionmaintheorem}, except with the small constant probability set in this section, has a training complexity of $\Omega(\ln{\left(m/d\right)}\ln{\left(m\right)}m)$. This example also shows that $U_{T}$ might depend on $m$, as mentioned above, which can happen in the distribution-free setting.

The intuition behind the example is to consider a universe consisting of the points $x_{i} = (0, 1 - \frac{i}{m^3}), y_{i} = -1$ for $i = 1, \ldots, m-1$, and $x_{m} = (\sqrt{\frac{1}{m}}, 1), y_{m} = 1$, which has a margin of $\gamma = \Theta(\sqrt{\frac{1}{m}})$. If the perceptron is run on the whole universe, it makes $2$ mistakes when passing over the data, i.e., when the label switches sign. One can then show that for the perceptron to converge, one has to pass over the universe $\Omega(m)$ times, where each pass takes $\Omega(m)$ operations, leading to a total complexity of $\Omega(m^2)$.

Using a distribution $\cD$ that assigns a small mass to $x_{m}$ and uniformly to the remainder of the universe ensures this pattern also holds with probability at least $1/2$ for a random bootstrap sample, leading to $U_{T}(0.02m) \geq \Omega(m^{2})$.

For the training complexity of \cref{introductionmaintheorem}, we consider a sub-training sequence of size $1650$, since the VC-dimension of the sign of a hyperplane is at most $3$ in $\mathbf{R}^{2}$. For a sub-training sequence denoted $(x'_{1}, y'_{1}), \ldots, (x'_{1650}, y'_{1650})$ of $(x_{1}, y_{1}), \ldots, (x_{m}, y_{m})$, the claimed training complexity follows by noting that the margin is at least $\gamma = \Theta(\sqrt{\frac{1}{m}})$ and that the perceptron only makes $\frac{\max_{i=1, \ldots, 1650} ||x'_{i}||^{2}}{\gamma^{2}} = O(m)$ updates/mistakes when repeatedly passing over $(x'_{1}, y'_{1}), \ldots, (x'_{1650}, y'_{1650})$. Since there must be one mistake for each pass over $(x'_{1}, y'_{1}), \ldots, (x'_{1650}, y'_{1650})$ and one pass takes $O(1)$ operations, this leads to a training complexity of $O(\ln{\left(m/d\right)}\ln{\left(m\right)}m)$, except with the small constant probability considered in this section.

In \cref{appendix:perceptron}, we formalize the above argument for the perceptron and also take into consideration the hyperplane including a bias term.
\newline \linebreak
Before moving on to the next two sections, where we in the first section provide a high-level proof sketch and in the latter section a detailed proof sketch and explain the relation of our work to previous work, we would like to mention some other related work on optimal learning. 

\cite{optimalwithoutuniformconvergence} gives an interesting alternative to the above optimal PAC learners based on a majority vote of $ m/4 $ one-inclusion graphs, given $ \Theta(m) $ points as input. When the one-inclusion graphs have to be computed this require searching through the hypothesis class $ \cH $, which makes it hard to compare to the above setting where the learner is only assumed access to the hypothesis class through a $ \erm $ algorithm. 

We also want to mention \cite{majorityofthree}, which shows, that the majority vote of 3 $ \erm $-calls on 3 disjoint training sequences has optimal in expectation error $ O(d/m) $. However, \cite{majorityofthree} were not able to extend the result to the optimal error bound in the PAC model.

\section{High-level proof sketch}
Before presenting the related work in the following section, we provide a high-level description of our algorithm and analysis. The intent of this high-level description is firstly to provide the reader intuition about our proof, and secondly to help establish connections with related work presented in the following section. In the detailed proof sketch, which follows the subsection of related work, we specifically point to how relevant steps in our analysis relate to previous work. For a reader mainly interested in understanding the proof of \cref{maintheorem} one can skip \cref{sec:proofoverview} and jump to \cref{sec:optimalitya} which introduces the necessary lemmas to prove \cref{maintheorem}, and show how they imply \cref{maintheorem}. We now give a high-level description of the algorithm giving the guarantee of \cref{introductionmaintheorem}.

\begin{algorithm}
\caption{High level description of Efficient optimal PAC learner.}\label{alg:algorithmhighlevel}
\begin{algorithmic}[1]
\State Sample $l=\Theta\big(\ln\big(\tfrac{m}{\delta\left(d+\ln{\left(1/\delta \right)}\right)}\big)\big)$ structured subtraining sequences $\rS_1, \ldots, \rS_{l}$ of $ \rS $, with  $ |\rS_{i}| =\Theta(m).$ \label{alg:highlevel1}
\State Generate $ l $  majority voters by running a boosting algorithm $\cB$ on each $\rS_{1},\ldots,\rS_{l}$. \quad \quad \quad \quad High level description of $ \cB(\rS_{i}) $:\label{alg:highlevel2}
\begin{algorithmic}[1]
    \State$ \cB $ uses $\erm$ as a weak learner, and makes at most $\Theta(\ln{\left(\tfrac{m}{\delta} \right)} )$ training invocation of $ \erm $ to get hypothesis $ h_{i,1},\ldots,h_{i,t},$ with $ t=\Theta(\ln{\left( m\right)}) $ .
    \State On each of the training evocation the $ \erm $ algorithm, is provided $550d$ training examples and returns a hypothesis $ h $. After each training evocation of the $ \erm $ algorithm, $ \cB $ evaluates $ h$ on $ \rS_{i} $. Based upon the evaluation $ \cB $  updates a distribution on $ \rS_{i}$, or discards the hypothesis $ h.$     
    \State The output of $ \cB(\rS_{i}) $  is a voting classifier $\cB(\rS_{i}) = \sum_{j=1}^{t} h_{i,j}/t$
\end{algorithmic}
\State \negmedspace \negmedspace \negmedspace \negmedspace \negmedspace \negmedspace \negmedspace For each $ i=1,\ldots,l $  sample a voter $ \rhh_i $  from the voters $ \{ h_{i,j} \}_{j=1}^{t} $ in the voting classifier $ \cB(\rS_{i})=\sum_{j=1}^{t} h_{i,j}/t $. Output $f=\sign(\sum_{i=1}^{l} \rhh_i) $ as the final predictor.\label{alg:highlevel3}
\end{algorithmic}
\end{algorithm}

We first give the rough analysis for the training and inference complexity of \cref{introductionmaintheorem} based on \cref{alg:algorithmhighlevel}. From Line~\ref{alg:highlevel2} we see that the boosting algorithm is invoked $ \Theta\big(\ln\big(\tfrac{m}{\delta\left(d+\ln{\left(1/\delta \right)}\right)}\big)\big)$ times and that the boosting algorithm on each of these evocations, makes at most $ \Theta\left(\ln{\left(\tfrac{m}{\delta} \right)} \right)$ training evocations of the $ \erm $, with $ 550d $ training points in each evocation, implying a training complexity of $ O\big(\ln\big(\frac{m}{\delta\left(d+\ln{\left(1/\delta \right)}\right)}\big)  \ln{\left(\tfrac{m}{\delta} \right)}   U_{T}(550d)\big)$. After each training evocation of the $ \erm$, the boosting algorithm evaluates the just trained $ \erm $ on all of $ \rS $, amounting to $ O\big(\ln\big(\frac{m}{\delta\left(d+\ln{\left(1/\delta \right)}\right)}\big)  \ln{\left(\tfrac{m}{\delta} \right)}   mU_{I}\big)$ operations over all the evocations of the boosting algorithm, which gives the high-level analysis of the training complexity. The claimed inference complexity follows from the final predictor in Line~\ref{alg:highlevel3} being an majority vote of  $ \Theta\big(\ln\big(\tfrac{m}{\delta\left(d+\ln{\left(1/\delta \right)}\right)}\big)\big) $ hypothesis.

The optimal PAC error bound on a high level (see \cref{fig:first_figure}) follows from showing that with probability at least $ 1-\delta $ over $ \rS $, it holds that: At least a $ 1-O\left(\left(d+\ln{\left(\delta\right)}\right)/m \right)$ fraction of new examples $ (\rx,\ry) $, is such that with probability at least $ 3/4 $ over the randomness drawing the structured sub training sequences $ \rS_{i} $ in Line~\ref{alg:highlevel1}, the majority vote $ \cB(\rS_{i})=\sum_{j=1}^{t}h_{i,j}/t $ has at least $ 3/4 $ of the voters $ h_{i,j} $ being correct on $ (\rx,\ry) $ (see \cref{fig:first_figure}). Thus for such new examples when drawing $ \rhh_{i} $ in Line~\ref{alg:highlevel3} it is correct with probability at least $ (3/4)^{2}\approx 0.56 $ - slightly better than guessing - so drawing enough of these hypotheses $ \rhh_{i} $, enough being $\Theta\big(\ln\big(\tfrac{m}{\delta\left(d+\ln{\left(1/\delta \right)}\right)}\big)\big) $ many, concentration inequalities implies, that the majority vote $f=\sign(\sum_{i=1}^{l} \rhh_i) $, except on a $ O\left(d+\ln{\left(1/\delta \right)}/m\right) $ fraction of such $ (\rx,\ry) $ examples, will contain more voters being correct than wrong and hence output the correct answers. As there was only a $ O\left(d+\ln{\left(1/\delta \right)}/m\right) $ fraction of new examples $ (\rx,\ry) $  that did not have the above property the optimal PAC bound of $ O\left(d+\ln{\left(1/\delta \right)}/m\right) $ follows.

\begin{figure}[h]
    \centering
    \includegraphics[width=\textwidth]{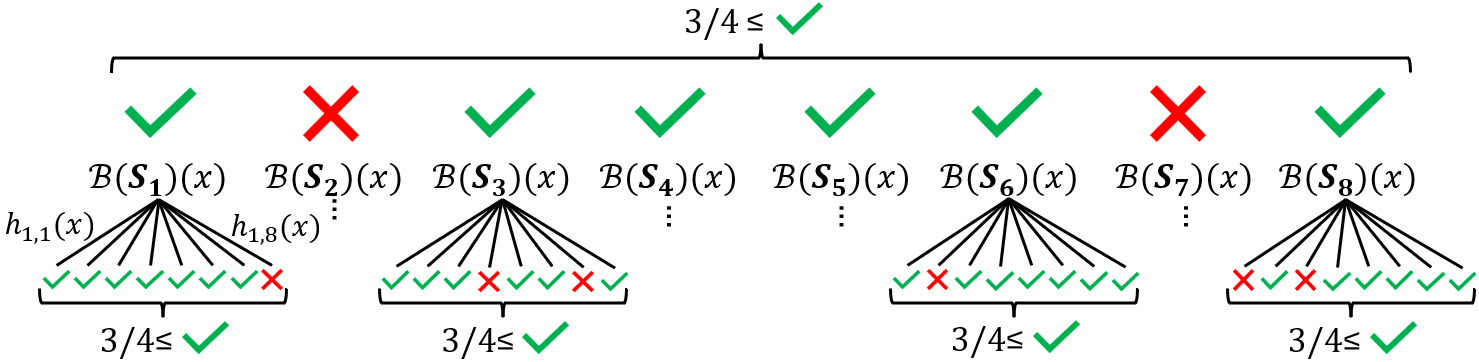}
    \caption{The figure illustrates outputs of the boosting algorithm on $ 8 $  structured sub training sequences, each producing a majority vote consisting of $ 8 $ voters depicted by lines coming out of $ \cB(\rS_{i}) $ (to not overload the figure we only include the name of $ h_{1,1}$ and $h_{1,8}$ on the lines). For most new examples $ (\rx,\ry) $ with probability at least $ 3/4 $ over the draw of $ \rS_{i} $  the majority vote $ \cB(\rS_{i}) $ has $ 3/4 $ of it voters correct. Boosting calls $ \cB(\rS_{i}) $ with a green check mark over has $ 3/4 $ of its voters correct on the new example $ (\rx,\ry) $, else a red cross. Lines with a green checkmark at the end correspond to a voter being correct on $ (\rx,\ry) $, and if incorrect a red cross. For instance the call $ \cB(\rS_{1}) $ has all of it voters expect $ h_{1,8}(\rx) $ being equal to $ \ry $, thus $ \cB(\rS_{1})(\rx) $ has a green check mark as $ 7/8\geq 3/4 $ of its voters are correct.}
    \label{fig:first_figure}
\end{figure}

\section{Previous work and detailed proof sketch}\label{sec:proofoverview}

In this section, we give a detailed proof sketch and describe how our result compares to previous optimal PAC learners. To this end, we describe the works \cite{hannekeoptimal}, \cite{Optimalweaktostronglearning}, and \cite{baggingoptimalPAClearner}. We will write random variables  $ \rx $ with boldface letters, and non-random variables and realizations of random variables as $ x $. 
\subsection{Previous work} 
\paragraph{Approach of \cite{hannekeoptimal}:}
To understand the above-mentioned works, we start with the work \cite{hannekeoptimal}, which the two other works use ideas from. For simplicity, we assume in this paragraph that we have a training sequence $ S $ with size some power of $4$. Furthermore, given training sequences $S,T\in (\cX \times \{-1,1\})^{*}$ we let $ S\sqcup T $ be the concatenation of the two training sequences $ S $ and $ T $, i.e.   $ [S^{T},T^{T}]^{T} $ (multiplicities of examples are important). In the following we will always assume that the training sequences are realizable by the target concept $ c.$  With this in place we now describe the deterministic subsampling scheme by \cite{hannekeoptimal}. 
\begin{algorithm}[H]
  \caption{$\cS'(S,T)$}\label{alg:Subsamplehanneke}
  \begin{algorithmic}[1]
    \State \textbf{Input:} Training sequences $S ,T\in (\cX\times \{-1,1\})^{*}$ 
    \State \textbf{Output:}{ A collection of training sequences}
      \State \textbf{if $ |S|\leq 3 $ then}
      \State\hspace{0.5cm}\textbf{return $ S\sqcup T $ }
      \State Split $ S $ into  $S_0,S_{1},S_{2},S_{3}$  where $ S_i $ contains the examples from $i|S|/4+1$ to $ (i+1)|S|/4$
      \State \textbf{return}$[\cS'\left(S_0 , S_2 \sqcup S_3 \sqcup T\right) , \cS'\left(S_0 , S_1 \sqcup S_3 \sqcup T\right) , \cS'\left(S_0 , S_1 \sqcup S_2 \sqcup T\right)]$
  \end{algorithmic} 
\end{algorithm}
\vspace{-0.3cm}
The above algorithm, given a training sequence $S$ of size $m=4^{k}$, generates $3^{\log_{4}(m)}\approx m^{0.79}$ training sequences, which have some non-overlapping parts, as it leaves out a fourth of the training examples in each recursive call. Leaving out training examples will be key later. Using this algorithm, 
\citeauthor{hannekeoptimal} produces his optimal PAC learner as the $ \sign(\cdot) $ of the majority voter, consisting of the output of the $\erm$-algorithm run on the sub training sequences of $S$ that $\cS'(S,T)$ returns for $ T=\emptyset $, i.e. $ \sign $ of  $\sum_{S'\in \cS'(S,T)} \erm(S)/|\cS'(S,T)|$. However, for the analysis, \citeauthor{hannekeoptimal} considers the above majority vote for general $ T\in(\cX\times\{-1,1\})^{*} $,  which we will denote $\erm(S,T)=\sum_{S'\in \cS'(S,T)} \erm(S)/|\cS'(S,T)|$ from now on. One of the key insights of \citeauthor{hannekeoptimal} is that by the recursive structure of the sub training sequences in $\cS'(S,T)$, $ \erm(S,T) $ can be written as
\begin{align*}
    \erm(S,T)=\frac{\erm(S_{0},S_{1}\sqcup S_{2}\sqcup T)}{3|\cS'(S_{0},S_{1}\sqcup S_{2}\sqcup T)|}+
\frac{\erm(S_{0},S_{1}\sqcup S_{2}\sqcup T)}{3|\cS'(S_{0},S_{1}\sqcup S_{3}\sqcup T)|}+
\frac{\erm(S_{0},S_{1}\sqcup S_{2}\sqcup T)}{3|\cS'(S_{0},S_{2}\sqcup S_{3}\sqcup T)|},
\end{align*}
where we in the following will write $\erm(3,T)$ for $\erm(S_{0},S_{1}\sqcup S_{2}\sqcup T)$, i.e., the $\erm$ trained on the sub training sequences where $S_{3}$ is left out, and likewise $\erm(2,T)$ and $\erm(1,T)$. 

Now, for the majority vote $\erm(\cS'(S,T))$ to fail on a point $x\in\cX$, it must be the case that half of its voters are incorrect, i.e., one of the three recursive calls $\erm(i,T)$ on the right-hand side of the above equation also fails on this point. Furthermore, there is still at least $1/2-1/3$ of the remaining voters in $\erm(\cS'(S,T))$ not in $\erm(i,T)$ that also fail, which corresponds to $1/4$ of the voters in $\erm(j,T)$ and $\erm(j',T)$ for $ j\not= j' $ and $ j,j'\not=i.$ Thus, if we pick uniformly at random $\ri\in\left\{ 1,2,3 \right\}$ and $\rh$ uniformly at random from $\erm(j,T)$ and $\erm(j',T)$ for $j'\not=j$ and $j',j\not=\ri$, we have that the chosen $\erm(\ri,T)$ and $\rh$ both fails on a point $x$ where $\erm(S,T)$ fails with probability at least $1/12$ over $\ri$ and $\rh$. Therefore, \citeauthor{hannekeoptimal} concludes that 
$ \p_{\rS\sim\cD_c^m}[\ls_{\rx\sim\cD_{c}}(\erm(\rS,T))\geq C(d+\ln{\left(1/\delta \right)})/m]$ is upper bounded by $\p_{\rS\sim\cD_c^m}[12\p_{\rx\sim\cD,\ri,\rh}[\erm(\ri,T)(\rx)\not=c(\rx),\rh(\rx)\not=c(\rx)]\geq C(d+\ln{\left(1/\delta \right)})/m],
$ where $C>1$ is some constant. By this relation, \citeauthor{hannekeoptimal} further argues that it is sufficient to show that $
    \max_{i\in\left\{ 1,2,3 \right\}}\max_{h\in \tiny\erm(\rS,T)\backslash\erm(i,T)}   12\p_{\rx\sim\cD}[\erm(i,T)(\rx)\not=c(\rx),h(\rx)=c(\rx)]$ is less than $ C(d+\ln{(1/\delta )})/m $
 with probability at least $ 1-\delta $ over $\rS$, which would follow by showing for each $i\in\left\{ 1,2,3 \right\}$ that 
\begin{align}\label{eq:intro1}
  \vspace{-0.6cm}
\max_{h\in \erm(\rS,T)\backslash\erm(i,T)} \negmedspace \negmedspace \negmedspace \negmedspace   12\p_{\rx\sim\cD}\left[\erm(i,T)(\rx)\not= c(\rx),h(\rx)\not=c(\rx)\right]\leq C(d+\ln{\left(1/\delta \right)})/m 
\end{align}
\vspace{-0.4cm}\newline holds with probability at least $1-\delta/3$, and then applying a union bound. 
To argue why this last statement is true \citeauthor{hannekeoptimal} continues by induction over the size of $\rS$ for $m=4^{k}$ where $ k\in\mathbb{N} $. The induction base $k=1$ follows by the right-hand side of the above for $m=|\rS|=4$ being larger than $1$. Now, for the induction step, let us consider the case $i=1$, which can be done without loss of generality since $\rS_{0},\rS_{1},\rS_{2},\rS_{3}$ are i.i.d.. \citeauthor{hannekeoptimal}, now uses the fact that, the induction base applied to $\erm(i,T)$ (which recurses on $\rS_{0}$ of size $4^{k-1}$) implies that,
\begin{align}\label{eq:intro2}
 \p_{\rx\sim\cD}\left[\erm(1,T)(\rx)\not=c(\rx)\right]\leq 4C(d+\ln{\left(9/\delta \right)})/m \leq 4\ln{\left(9e \right)}C(d+\ln{\left(1/\delta \right)})/m, 
\end{align}
holds with probability at least $1-\delta/9$ over $\rS_{0},\rS_{2},\rS_{3}$.
Furthermore, notice that if $\rS_{0},\rS_{2},\rS_{3}$ are such that $
    \p_{\rx\sim\cD}\left[\erm(1,T)(\rx)\not=c(\rx)\right]\leq C(d+\ln{\left(1/\delta \right)})/m $  then by the monotonicity of measures, we have
    $\max_{h\in \erm(2,T)\sqcup\erm(3,T)}    12\p_{\rx\sim\cD}\left[\erm(1,T)(\rx)\not= c(\rx),h(\rx)\not=c(\rx)\right] \leq C(d+\ln{\left(1/\delta \right)})/m.$ 
Thus, the problem reduces to analyzing the outcomes of $\rS_{0},\rS_{2},\rS_{3}$ such that \cref{eq:intro2} holds and is also lower bounded by $ C(d+\ln{\left(1/\delta \right)})/m .$ 
\citeauthor{hannekeoptimal} now uses the law of total probability to conclude that for such outcomes $ S_0,S_2,S_3 $, it holds for any $h\in \erm(2,T)\sqcup\erm(3,T)$, that the error $12\p_{\rx\sim\cD}[\erm(1,T)(\rx)\not= c(\rx),h(\rx)=c(\rx)]$ is less than 
\begin{align}\label{eq:intro4}
48\ln{\left(9e \right)}C \p_{\rx\sim\cD(\cdot \mid \erm(1,T)(\rx)\not=c(\rx))}\left[h(\rx)\not=c(\rx)\right] (d+ \ln{\left(1/\delta \right)}) /m.
    \end{align}
Thus, it suffices to show a uniform error bound, of $1/(48\ln{\left(9e \right)})$, under the conditional distribution, given that $\erm(1,T)$ errs, for all $h\in \erm(2,T)\sqcup\erm(3,T)$.

Now, assuming that \cref{eq:intro2} is lower bounded by $ C(d+\ln{\left(1/\delta \right)})/m $  and using the fact that $\rS_{1}$ contains, in expectation,
$\p_{\rx\sim\cD}[\erm(1,T)(\rx)\not=c(\rx)]m/4$, points from $\{x: \erm(1,T)(x)\not=c(x) \}$ (which, by the lower bounds is  greater than $C\ln{( e/\delta )}/4$) it follows by the Chernoff bound and for $C>64$ that $\rN_{1}:=|\rS_{1}\sqcap \{x: \erm(1,T)(x)\not=c(x) \}|\geq C(d+ \ln{(1/\delta )} )/8$ with probability at least 
$1-\exp(\p_{\rx\sim\cD}[\erm(1,T)(\rx)\not=c(\rx)]m/4^2)\geq 1-(\delta/e)^{4}\geq 1-\delta/9$  over $\rS_{1}$. 

Conditioned on this event, and since every $h\in \erm(2,T)\sqcup\erm(3,T)$ is a $ \erm $ trained on training sequence which contains $\rS_{1}$, it is consistent with $ \rS_{1}$, especially with the points in $\rS_{1}\sqcap \{x: \erm(1,T)(x)\not=c(x) \}$. The $ \erm $-bound (\cref{uniformconvergencemain}) on the points $\rS_{1}\sqcap \{x: \erm(1,T)(x)\not=c(x) \}$ (which is drawn according to the conditional distribution $\rx\sim\cD(\mid \erm(1,T)(\rx)\not=c(\rx)))
$ then gives, with probability at least $1-\delta/9$, for all $h\in \erm(2,T)\sqcup\erm(3,T)$ simultaneously, we have
$\p_{\rx\sim\cD(\cdot \mid \erm(1,T)(\rx)\not=c(\rx))}\left[h(\rx)\not=c(\rx)\right]\leq 2(d\ln{\left(2e \rN_{1}/d \right)}+\ln{\left(18/\delta \right)})/(\ln{\left(2 \right)}\rN_{1}).$ Since we had a lower bound of $\rN_{1}\geq C(d+ \ln{\left(1/\delta \right)} )/8$ and this is a decreasing function in $\rN_{1}$, we see that for $C$ sufficiently large, this is less than $1/(48\ln{\left(9e \right)})$. Therefore, by a union bound over the three above events (each of which occurs with probability at least $1-\delta/9$), \cref{eq:intro1} holds with probability at least $1-\delta/3$ as claimed. 

We are now ready to explain one of our key insights leading to \cref{introductionmaintheorem}. Using \cref{uniformconvergencemain}, one can see that a $ \erm $ trained on $ \Theta(d) $ training examples has a small constant error with probability at least $ 1-\exp(-d) $. This observation has also been used, for instance, in \cite{samplecompressionschemesamirmoran} to show existence of sample compression schemes for VC-classes, independent of the sample size. At first glance, it might seem like training on $ \Theta(d) $ examples would work straight out of the box with the above argument, as the argument required a uniform error bound on $h\in \erm(2,T)\cup \erm(3,T) $ of some small constant error. However, the above uniform error bound must hold under the conditional distribution $ \cD(\cdot \mid \erm(1,T)(\rx)\not=c(\rx))$, so a training sequence of $\Theta(d) $ examples would by the above argument only contain $ \Omega(d(d+\ln{\left(1/\delta \right)})/m) $ examples from this distribution so not necessarily any or enough to guarantee a small constant error. 

However, as we noted, using $ \Theta(d) $ training examples, the $ \erm $-learner can achieve a small error under the distribution $ \cD $, from which the training examples are drawn from, with probability at least $ 1- \exp{\left(-d \right)}  $. Furthermore, since we know the labels on the training sequence $ S $ we can for distributions $\cD$ over $S$ always check if the error is small. Thus, we can with this observation employ boosting with resampling (see, e.g., \cite{boostingbookSchapireF12}[Section 3.4]) to learn a voting classifier that is correct on the entire sample $ S $  using only calls of size $ \Theta(d) $ to the $ \erm $-learner. In fact, a voting classifier can even be found with good margins on $ S $, where a voting classifier $ f =\sum_{h\in\cH}\alpha_{h}h$ (with $ \alpha_{h} $ summing to 1) is said to have margin $ \gamma $ on $ S $ if for every $ (x,y)\in S $,  $f(x)y\geq \gamma  $. We now present the following work by \cite{Optimalweaktostronglearning}, which indicates how boosting might help improve training complexity.
\vspace{-0.3cm}
\paragraph{Approach of \cite{Optimalweaktostronglearning}:}
\cite{Optimalweaktostronglearning} used the fact that given a $\gamma$-weak learner $ \cW $, i.e., for all training sequences $\rS\sim \cD_{c}^{m}$ and any distribution $D$ over the examples in $\rS$, the $ \gamma $-weak learner $\cW$, given a sample of size $m_{0}$ from $D$, outputs a hypothesis $h$ from some base class $\cH$ with VC-dimension $d$ such that $\ls_{D_{c}}(h)\leq 1/2-\gamma$ with probability at least $1-\delta_{0}$ for $\gamma\leq 1/2$ and $\delta_{0}< 1$ - one can run a boosting algorithm $\cG$ that, given query access to $\cW$, returns a voting classifier $\cG(\rS)=\sum_{h\in \cH} \alpha_{h}h$ which has margins $\cG(\rS)(\rx)\ry=\Omega(\gamma)$ for all $(\rx,\ry)\in\rS$. 

The key insight of \citeauthor{Optimalweaktostronglearning} is that redoing the analysis of \cite{hannekeoptimal} with $\sign(\cG(\rS'))$, run on the training sequences in $\rS'\in \cS'(\rS)$ until the point \cref{eq:intro4}, uses no fact about which learning algorithm the training sequences in $\cS'(\rS)$ are trained on. Therefore, showing the claim of \citeauthor{Optimalweaktostronglearning}, similarly boils down to showing a uniform error bound for $h\in \sign(\cG(2,T))\sqcup\sign(\cG(3,T))$ under the conditional distribution $ \cD_{c}(\cdot\mid \cG(1,T)\not=c(\rx)) $ of  $1/(48\ln{\left(9e \right)})$ to get their target error of $C(d\gamma^{-2}+\ln{\left(1/\delta \right)})/m$. By a similar argument to that of \cite{hannekeoptimal}, \citeauthor{Optimalweaktostronglearning} concludes that $\rS_{1}$ contains $C(d\gamma^{-2}+\ln{\left(1/\delta \right)})/8$ points on which $\sign(\cG(1,T))$ fails with probability at least $ 1-\delta/9 $. Thus, reducing the problem to showing a uniform bound for the error of $ \sign(\cdot) $ of voting classifiers with margin $\Omega(\gamma)$ on an i.i.d. training sequence of size $C(d\gamma^{-2}+\ln{\left(1/\delta \right)})/8$, of at most $1/(48\ln{\left(9e \right)})$ with probability at least $1-\delta/9$. To this end, \citeauthor{Optimalweaktostronglearning}[See Theorem 4] show a uniform error bound that, informally, says with probability $1-\delta$, all voting classifier $ g $  with $\Omega(\gamma)$ margin on a training sequence $|\rS|\geq C(d\gamma^{-2}+\ln{\left(1/\delta \right)})/8$ has at most 
 $ 1/200 $ error, which is sufficient to complete their analysis. 

Thus, using our observation about $ \erm $ from before this paragraph, it follows that it can be plugged in as, for instance, a $ 1/4 $-weak learner with $ m_{0}=\Theta(d) $ and $ \delta_{0}= \exp{\left(-\Theta(d) \right)}  $  in the above algorithm, with a boosting algorithm like AdaBoost \cite{Adaboost}, run until the output of AdaBoost is consistent with the sub training sequence, i.e., $ \Theta(\ln{\left(m \right)}) $ rounds. Following the above method by \citeauthor{Optimalweaktostronglearning}, assuming the $ \erm $-algorithm never fails and running AdaBoost on all sub training sequences in $\cS'(S,\emptyset) $, i.e., $ \Omega(m^{0.79}) $ times, and assuming that AdaBoost is run for $ \Theta(\ln(m)) $ rounds and in each round update the distribution over $ \rS $ by evaluating the just-received hypothesis on the whole $ \rS $, this gives a training complexity of at least $ \Omega(m^{0.79}\ln{\left(m \right)})(\Utrain(\Theta(d))+3m\Uinf) $ and an inference complexity of $\Omega( m^{0.79}\ln{(m )} \Uinf)$, since it has to query $ \Omega(m^{0.79})$ many voting classifiers consisting of $ \Omega(\ln{(m )}) $ hypothesis. We want to stress that we do not claim that the above approach attains this lower bound or that one could not do something different with the above result by \citeauthor{Optimalweaktostronglearning} which would be more efficient. With that being said we notice that compared to \cite{hannekeoptimal}  the training complexity term $ \Utrain(\cdot) $ is now dependent on $ d $ not $ m $, however, there is an $ \Omega(m^{1.79}) $ term showing up in the training complexity and the inference complexity have become worse. Thus, the above indicates that boosting is helping to remove the $ m $ dependency in $ \Utrain(\cdot) $. We now move on to describe the work of \cite{baggingoptimalPAClearner}.

\paragraph{Approach of \cite{baggingoptimalPAClearner}:}
The analysis of \cite{baggingoptimalPAClearner} consists of two key steps, which we will briefly describe. To this end let $(\rS,\rB_{i})$ denote a sample with replacement from $\rS$ of size $ m'\in[0.02m,m]$. \citeauthor{baggingoptimalPAClearner} then outputs the $ \sign(\cdot) $ of the majority voter $\cB(\rS)=\sum_{i=1}^{n'} \erm((\rS,\rB_{i}))/n'$ for $n'=\Theta(\ln{\left(m/\delta \right)})$. Let $\binom{\rS}{m'}$ denote all possible sequences with replacement from $\rS$ of size $m'$.  

The first step of the analysis in \cite{baggingoptimalPAClearner} is to "derandomize" the bagging step, which \citeauthor{baggingoptimalPAClearner} accomplish by relating the error of the majority voter $\cB(\rS)$ to the purely analytical structure $\bar{\cB}(\rS)$ which is defined as $\sum_{\rS'} \ind\{\rS'\in \binom{\rS}{m'}\}\erm(\rS')/|\tbinom{\rS}{m'}|$. To this end \citeauthor{baggingoptimalPAClearner} observe that $ (\rS,\rB_{i}) $ can be viewed as a drawn with replacement from $\tbinom{\rS}{m'}$, which implies that $ \cB(\cS) $ can be seen as the majority vote of hypotheses drawn with replacement from the voters in $ \bar{\cB} $. To the end of using this observation, \citeauthor{baggingoptimalPAClearner} defines the event $E=\{ x\in \cX: \bar{\cB}(x)c(x)/|\binom{\rS}{m'}|\geq 1/3\}$ and bounds the error of $\cB(\rS)  $ by the following two terms using the law of total probability:
\begin{align}\label{eq:splittingerror}
  \vspace{-0.5cm}
 \ls_{\rx\sim\cD_{c}}[\cB(\rS) ] \leq \p_{\rx\sim\cD}[\bar{E}]+\ls_{\rx\sim\cD_{c}(\cdot \mid E)}(\cB(\rS)).
\end{align}
\vspace{-0.5cm}\newline
Now, using the observation that the voters in $ \cB $ can be viewed as drawn with replacement from the voters of $ \bar{\cB}(\rS) $, \citeauthor{baggingoptimalPAClearner} concludes that for any realization $S$ of $\rS$ and $x\in E$ the expectation of $c(x)\cB(S)$ is at least $1/3$. By applying Hoeffding's inequality and using $ n'=\Theta(\ln{(m/\delta )})$, it follows that at least half of the bootstrap samples are correct on $x$ with probability at least $1-\delta/(4m)$ over the bootstrap sample. An application of Markov's inequality over $\p_{\cB}[\ls_{\rx\sim\cD_{c}(\cdot|E)}(\cB(S))\geq 1/m]$, shows that with probability at least $1-\delta/4$ over the bagging step, the error is at most $1/m$. Since this holds for any realization $S$ of $\rS$ and the bagging step and $\rS$ are independent, this also holds for random $ \rS $, bounding the second term in \cref{eq:splittingerror} by $ 1/m $ with probability at least $ 1-\delta/4 $. 

The second step of \citeauthor{baggingoptimalPAClearner} is to upper bound the probability of $\bar{E}$, i.e., the margin error $\{ x\in \cX: \bar{\cB}(x)c(x)< 1/3\}$ of the classifier $\bar{\cB}$. This part of the analysis is highly non-trivial, and we will give only a high-level overview here. One of the key observations by \citeauthor{baggingoptimalPAClearner} in the second step, is that the training sequences in $ \binom{\rS}{m'} $ can be seen as being created recursively by a splitting algorithm. The splitting algorithm first splits the training sequence $\rS$ into $20$ disjoint training sequences followed by $19$ recursively calls to the same splitting algorithm. 
One reason for the large number of buckets is to ensure that $\bar{\cB}$ can be shown to have good margins- an idea we will use later. Furthermore, using the viewpoint of the training sequences in $ \binom{\rS}{m'} $ being created recursively \citeauthor{baggingoptimalPAClearner} gets some structure similar to \cite{hanneke2016refined}, with nonoverlapping training sequences. Similar to the step below \cref{eq:intro4} in the above overview of \cite{hannekeoptimal}, \citeauthor{baggingoptimalPAClearner} also uses at some point of the analysis that the training sequences are sufficiently nonoverlapping to see $\Theta(d+\ln{\left(1/\delta \right)})$ training examples under some conditional distribution and then calling \cref{uniformconvergencemain} for $\erm$'s trained on these $\Theta(d+\ln{\left(1/\delta \right)})$ training examples to get a uniform error bound of some sufficiently small constant. 

From the above, we deduce that the training complexity would be $\Omega(\ln{\left(m/\delta \right)})\Utrain(O(0.02m))$, since $\Theta(\ln{\left(m/\delta \right)})$ bagging training sequences are created, and a $\erm$ is trained on each of them. Inference on a new point has computational complexity $O(\ln{\left(m/\delta \right)})\Uinf$, as it requires querying all trained $ \erm $'s and taking a majority vote of their answers.

\citeauthor{baggingoptimalPAClearner} also says that combining Bagging and Boosting yields an optimal weak to strong learner. By combining this with our observation that the $ \erm $ learner provides a $ 1/4 $-weak learner, and omitting the fail probability, making $ \Theta(\ln{(m/\delta )} )$ bootstrap sample, training AdaBoost on them  $ \Theta(\ln{(m )}) $ rounds, and assuming that the distribution used by AdaBoost over the sample $ S $ in training, is updated each round by querying the just received hypothesis on the whole $ S $, this approach would yield a training complexity of at least $\Omega(\ln{\left(m/\delta \right)}\ln{\left( m\right)})(\Utrain(O(d))+3m\Uinf)$ and an inference complexity of $\Omega(\ln{\left(m/\delta \right)}\ln{\left(m \right)})\Uinf$. Where we again want to stress that we are not claiming that these lower bounds actually could be attained with the above approach or that one could not do something better with \citeauthor{baggingoptimalPAClearner} Bagging + Boosting.
With that being said we notice the above mirrors the situation in \cite{Optimalweaktostronglearning}, which also indicated that the dependency on \( \Utrain(m) \) could be reduced to \( \Utrain(d) \), but also indicated that it would come at a cost of an increase in inference complexity due to the necessity of querying all voters for each majority vote when doing inference.

With the above-related work explained, we now present our approach.
\vspace{-0.3cm}

\subsection{Detailed proof sketch}

Our algorithm will, as a subroutine, run an algorithm that we name $ \as $ \cref{alg:AdaBoostSample} and denote $ \cA. $ To describe $ \cA $, we let $  \rr=(\rr_{1},\ldots,\rr_{\Theta(\ln{(m/\delta )})}) \sim([0:1]^{550d})^{\Theta(\ln{(m/\delta )})} $ denote a random string, where the $ \rr_{i} $'s are i.i.d. sequences of length $ 550d $, and the $ \rr_{i,j} $'s are i.i.d. and uniformly distributed on the interval $ 0 $ to $ 1 $. The algorithm $ \cA$ on input training sequence $S$ of size $m$, random string $ \rr \sim([0:1]^{550d})^{\Theta(\ln{(m/\delta )})}$, and query access to $\erm$, is informally described AdaBoost run with a fixed learning rate, and early stopping ensuring that $ \cA $  has a specific number of voters, $t= \Theta(\ln{\left(m \right)}) $, in its majority vote $\cA(S)= \sum_{i=1}^{t} h_i/t $. Furthermore, if $ \cA $ has not reached its early stopping criteria after $ n=\Theta(\ln(m/\delta)) $ calls to the weak learner/$ \erm,$  $ \cA $  terminates by outputting $ \erm(S) $ (we will also see this as a majority vote of $ t $ copies of $ \erm(S) $). We show that the latter happens with probability at most $O(\delta/m)$. 

Furthermore, we show that $ \cA $ guarantees that $\ls_{S}^{3/4}(\cA(S)):= \sum_{x\in S} \ind\{\cA(S)(x)c(x)\leq 3/4 \}/m <1/m $, meaning all the $m$ points in $ S $  has $ 3/4 $  margins. This will be used later in the proof to obtain a uniform error bound, similarly as in the other proofs -- however will also play a key role in getting our inference complexity that does not suffer from a blow-up due to the boosting step as indicated by the sketched lower bounds in the previous paragraphs.

The reason why we can guarantee that the outputted majority voter $ \cA(S) $ satisfies
$ \ls_{S}^{3/4}(\cA(S)) <1/m $ is due to \cref{AdaBoostSampleMarginLemma3}, which says that if $ \cA $ in each boosting call, upon querying $ \erm $  with a sample from $ D_{i} $ over $ S $, receives a hypothesis $ h_{i} $ from $ \erm $ such that  $ \ls_{D_{i}}(h_{i})\leq 1/2-\gamma  $ for $ \gamma=9/20 $,  we have that after $ t $ such rounds the in sample 3/4-margin loss is  $\p_{(\rx,\ry)\sim S}\left[\ry \sum_{i=1}^{t} h_{i}(\rx)\leq 3/4 \right]\leq (24/25)^{t}$. Now as we use a $ \erm $ as a weak learner, we can make $ \gamma <1/2$  arbitrarily close to $ 1/2 $  with enough samples from $ D_{i} $, except with a failure probability by \cref{uniformconvergencemain}. For the purpose of showing $ \ls_{S}^{3/4}(S)<1/{m} $ we will end up giving the $ \erm $ a sample of size $ 550d $, which will allow us to set $ \gamma=9/20 $ and since $t= \Theta(\ln{\left(m\right)}) $ the bound $ \ls_{S}^{3/4}(\cA(S))<1/m $ follows if the $ \erm $ do not fail. 

We now resolve the case of the $ \erm $ failing to get a smaller loss than $ 1/2-\gamma $. We remark that since we know $ D_{i} $ over $ S $ and $ S $, we can always check whether a hypothesis returned by $ \erm $ succeeds or not, and if it fails, we can skip that hypothesis and query $ \erm $ again with a new sample from $ D_{i}$. Thus, for the above argument to go through, we just have to ensure that $ \erm $ succeeds $ t $  times with probability at least $ 1-O(\delta/m) $. Since a boosting round can fail, we must do more than $ t $ rounds. Thus, we 
run a for-loop of size $n= \Theta(\ln{\left(m/\delta\right)} )$, and try to train a hypothesis for boosting in each iteration.  We will have that each call succeeds with probability at least $1-\delta_{0}=1-2^{-d}$, so the expected number of success is at least $(1-2^{-d})\Theta(\ln{\left(m/\delta \right)} )$. By the multiplicative Chernoff bound, the probability of seeing fewer than $(1-2^{-d})\Theta(\ln{\left(m/\delta\right)} )/2 \geq t$ success, is at most $ \exp((1-2^{-d})\Theta(\ln{\left(m/\delta \right)} )/8)=O(\delta/m)$. In the above argument we are omitting dependencies of the success which is defined in terms of the distributions $ D_1,\ldots $  that $ \cA $  is updating iteratively. In \cref{adaboostsamplefewroundslemma}, we take this into account, however as the entries of $ \rr $ are i.i.d. the argument will be close to the above. 

For the training complexity of our algorithm, we need the training complexity of $\cA$, as it runs $ \cA $ as a sub-routine. As stated above, with probability at least $1-O((\delta/m))$ $ \cA $ runs a for loop of size $ n $ and outputs a majority vote $\cA(S)= \sum_{i=1}^{t} h_i/t $ with $ 3/4 $-margins on $ S.$ We show in \cref{adaboostsampleruntime} that each iteration of the for loop takes $ O(m+d\ln{\left(m \right)})+\Utrain(d)+3m\Uinf $ operations. The $ O(m)+3m\Uinf$ arises from updating the distribution $ D_{i} $, leading to evaluating the just-received hypothesis $ h_i $  on all $ m $ examples in $ S.$ The $O(d\ln{\left(m \right)} ) $ term comes from sampling $ 550d $ times from the distribution $D_{i}$, where each sample from $ D_{i} $  requires $ \ln{\left(m \right)} $ operations to generate, as a binary search over the buckets of the cumulative distribution function of $ D_{i} $ is made, to find which bucket the uniform random variable $ \rr_{j,l} {\sim} [0,1]$ generating the sample landed in. The $ \Utrain(550d) $ term comes from the $ \erm $-algorithm training on the generated sample from $ D_{i} $ of size $ 550d $. Thus, as the for loop has size $\Theta(\ln{\left(m/\delta \right)} ) $,  this leads to a training complexity of $ O(\ln{\left(m/\delta \right)} )(O(m+d\ln{\left(m \right)})+\Utrain(d)+3m\Uinf ) $ with probability at least $ 1-O(\delta/m)$. With the properties of $ \cA $ introduced, we now present the splitting algorithm inspired by \cite{hannekeoptimal}, which we use in our algorithm. We denote this algorithm as $ \cS $.
\begin{algorithm}[H]
    \caption{Subsampling algorithm $\cS$}\label{alg:Subsample}
    \begin{algorithmic}[1]
        \State \textbf{Input:} Training sequences $S,T \in (\cX \times \cY)^{*}$.
        \State \textbf{Output:} A sequence of sub-training sequences of $S \sqcup T$.
        \State \textbf{if $|S| \geq 6$ then}
            \State\hspace{0.4cm}Split $S$ into $S_0, S_1, S_2, S_3, S_4, S_5$ where $S_i$ contains examples from  $i|S|/6+1$ to $(i+1)|S|/6$.\label{alg:Subsample:recursion}
            \State\hspace{0.4cm}\textbf{return} $[\cS(S_0, S_1 \sqcup T), \cS(S_0, S_2 \sqcup T), \cS(S_0, S_3 \sqcup T), \cS(S_0, S_4 \sqcup T), \cS(S_0, S_5 \sqcup T)].$\label{alg:Subsample:recursion2}
        \State \textbf{else}
            \State\hspace{0.4cm}\textbf{return} $S \sqcup T.$\label{alg:Subsample:recursion3}
    \end{algorithmic}
\end{algorithm}

We notice differences from \cite{hannekeoptimal} in that the recursive calls only overlap in $ S_{0} $, the number of splits being $ 6 $ instead of $ 4 $,  and the recursions being $ 5 $ instead of $ 3 $. The reason for the former is that we found it easier notation-wise to let $ \cS(i,T)=\cS(S_{0},S_{i}\sqcup T) $. The reason for the latter is more interesting and follows the same reasoning as \cite{baggingoptimalPAClearner} to achieve good margins. However, we do more than just obtain good margins for the majority vote of the majority voters, as done in \cite{baggingoptimalPAClearner} with bagging and boosting. This "more" is our key observation to prevent our inference complexity from blowing up due to the boosting step. Now let $ \cA(\cS(S,T)) =[\cA(S')]_{S'\in \cS(S,T)} $ be the family of hypothesis outputted by $ \cA $ when run on each sub training sequence of $ \cS(S,T) $. What we then show is that with probability at least $ 1-\delta $ over $ \rS $,          
\begin{align}\label{eq:deterministic error}
  \vspace{-0.4cm}
    \p_{\rx\sim\cD} \Bigg[  
      \sum_{\rf\in \cA(\cS(\rS,\emptyset))} \negmedspace \negmedspace \negmedspace \negmedspace \ind\{\sum_{s=1}^{t}  \ind\{h_{\rf,s}(\rx)=c(\rx)\}/t\geq 3/4\}/|\cS(\rS,\emptyset)|
       <3/4\Bigg]\negmedspace \negmedspace \leq \cs \frac{d+\ln{\left(1/\delta \right)}}{m},
  \end{align}
  \vspace{-0.4cm}\newline
where we have used that $\rf\in  \cA(\cS(S,T)) $ is a majority vote over $ t $ hypothesis, $ \rf=\sum_{s=1}^{t} h_{\rf,s}.$     
Before explaining why this holds, we provide the rationale for why this is what we want to show, and how it gives us the inference complexity that does not suffer an increase of $ \Theta(\ln{(m )}) $  from the boosting step. To this end, we draw inspiration from \cite{baggingoptimalPAClearner} \cref{eq:splittingerror} and the idea of "derandomzing" $ \cB $ by $ \bar{\cB} $, with the above classifier/event now taking the place of $ \bar{\cB} $. Now, let our algorithm be denoted $ \ah $ and let $ E $ denote the event $ \{x\in \cX \mid  \sum_{\rf\in \cA(\cS(\rS,\emptyset))} \ind\{\sum_{s=1}^{t} \ind\{h_{\rf,s}(x)=c(x)\}/t\geq 3/4\}/|\cS(\rS,\emptyset)|
\geq 3/4 \}$, we then have,  as in \cref{eq:splittingerror} that,
$\ls_{\rx\sim\cD_{c}}(\ah(\rS)) = \p_{\rx\sim\cD}[\bar{E}]+\ls_{\rx\sim\cD_{c}(\cdot \mid E)}(\ah(\rS)).$

Now given an $ x\in E $, we observe by the definition of $ E $ that if we sample a row/majority voter of $ \cA(\cS(\rS,\emptyset)) $ and then sample one of its voters, we have with probability at least $ (3/4)^{2}=9/16 $ that the sampled voter is such that $ h(x)=c(x) $, thus correct on the point $ x $. Since this probability is $1/16 $ greater than $ 1/2 $,  we get by repeating the above way of sampling voters $l= \Theta(\ln{\left(m/(\delta(d+\ln{\left(1/\delta \right)})) \right)} )$ times, that for $ x\in E $, letting $ \rX_{x}=\sum_{i=1}^{l}\rX_{x,i} $, where $ \rX_{x,i} $ indicates if the $ i $'th drawn voter is correct on $ x $, a Chernoff bound implies that $\p\left[\rX_{x}\leq l/2\right]= \p\left[\rX_{x}\leq (1-2/18)9l/16\right]\leq \exp((2/18)^{2}\cdot(9l/16)/4)=O((\delta(d+\ln{\left(1/\delta \right)}))/m )$. Thus, by Markov's inequality, $ \p\left[\p_{\rx\sim\cD_{c}(\cdot\mid E)}\left[\rX_{x}\leq l/2\right]\geq (d+\ln{\left(1/\delta \right)})/m \right]=O(\delta)$. That is with probability at least $ 1-\delta $ over the repeated random sampling over voters in the above fashion, we have that $\ls_{\rx\sim\cD_{c}(\cdot \mid E)}(\ah(\rS))\leq \p_{\rx\sim\cD(\cdot|E)}\left[\rX_{x}\leq l/2\right]=O((d+\ln{\left(1/\delta \right)})/m)  $ (The above argument is also depicted in \cref{fig:first_figure}). 

Thus, with probability at least $ 1-\delta $ over the above way of random sampling of voters, the majority vote of them, on a new example $ \rx $ drawn from $ \rx\sim\cD(\cdot|E) $ has $O((d+\ln{\left(1/\delta \right)})/m)$ error. Thus, if we let $ \ah $ be the algorithm that samples $l= \Theta(\ln{\left(m/(\delta(d+\ln{\left(1/\delta \right)})) \right)} )$ many hypotheses/voters in this fashion and takes the majority vote of them, we get, by the above splitting of $ \ls_{\cD_{c}}(\hat{\cA}(\rS)) $  and \cref{eq:deterministic error} (which upper bounds $\bar{E}$),  that with probability at least $ 1-\delta/2 $ over $ \rS $ and $ \ah $  that    $\ls_{\rx\sim\cD_{c}}[\ah(\rS) ] =O(d+\ln{\left(1/\delta \right)}/m)$, which shows that $ \ah $ obtain the optimal error bound of realizable PAC learning. Further, we get that the inference complexity will be evaluating each of the  $l= \Theta(\ln{\left(m/(\delta(d+\ln{\left(1/\delta \right)})) \right)} )$ voters sampled in the above fashion, and thus the inference complexity becomes $ \Uinf(\Theta(\ln{\left(m/(\delta(d+\ln{\left(1/\delta \right)})) \right)}))$ as claimed without the blow-up from the boosting step, since we can sample on the voters level.  

Thus, what we still need to argue about is that \cref{eq:deterministic error} holds with probability at least $ 1-\delta $. Inspired by \cite{hannekeoptimal} we show the claim by induction in the size $ m=6^{k} $ of $ \rS $, and consider \cref{eq:deterministic error} with an arbitrary training sequence $ T $ instead of $ \emptyset $. The induction base follows from the right-hand side of \cref{eq:deterministic error} being greater than $ 1 $ for $ m=6 $. Now we first notice that for the event inside of \cref{eq:deterministic error} to happen, it must be the case that there exists an $ i =1,2,3,4,5$ such that 
$ \sum_{\rf\in \cA(\cS (i,T))} \ind\{\sum_{s=1}^{t} \ind\{h_{\rf,s}(x)=c(x)\}/t\geq 3/4\}/|\cS(i,T)|
       <3/4.
$
Since $ \cA(\cS(i,T)) $ has $ 1/5  $ of the voters in $ \cA(\cS(\rS,T)) $, we have that there are still $ 1/4-1/5 =1/20$ of the voters in $  \cA(\cS(\rS,T))$ not in $\cA(\cS(i,T))$ that do not have $ 3/4 $ correct answers on $ x $. This is a $ (5/4)(1/20)=1/16 $ fraction of the majority voters in $ \cA(\cS(\rS,t))\backslash \cA(\cS(i,T)) $, and thus we get that a randomly chosen majority voter $ \rf $ from $ \cA(\cS(\rS,T))\backslash \cA(\cS(i,T)) $ is such that  $\sum_{s=1}^{t} \ind\{h_{\rf,s}(x)=c(x)\}/t< 3/4$ with probability at least $ 1/16 $. Thus, by drawing a uniform random $ \ri\sim {1,\ldots,5} $ and uniform random $ \rf \in  \cA(\cS(\rS,T))\backslash \cA(\cS(\ri,T)) $,  we conclude similarly to \cref{eq:intro1} that it suffices to show that for $ i,j\in \left\{ 1,2,3,4,5  \right\}$  $ i\not=j $   with probability at least $ 1-\delta/20 $ over $ \rS $ that 
\begin{align}\label{eq:ourrecursion}
  80 \negmedspace \negmedspace \negmedspace \negmedspace \negmedspace \negmedspace \max_{\rf\in\cA(\cS(\rS_{j},T))} \negmedspace \negmedspace \negmedspace \negmedspace \negmedspace \negmedspace \p_{\rx\sim\cD_{c}}\Bigg[ \negmedspace \negmedspace \negmedspace \negmedspace \negmedspace \negmedspace \negmedspace \negmedspace \negmedspace \negmedspace \negmedspace \negmedspace \negmedspace \negmedspace   \sum_{ \quad\quad \quad \rf\in \cA(\cS (i,T))} \negmedspace \negmedspace \negmedspace \negmedspace \negmedspace \negmedspace \negmedspace \negmedspace \negmedspace \negmedspace \negmedspace \negmedspace \negmedspace \negmedspace \frac{\ind\{\sum_{s=1}^{t} \ind\{h_{\rf,s}(\rx)=c(\rx)\}/t\geq \frac{3}{4}\}}{|\cS(i,T)|}
  \negmedspace  <\negmedspace \frac{3}{4}, \sum_{s=1}^{t} \ind\{h_{\rf,s}(\rx)=c(\rx)\}/t\negmedspace<\negmedspace\frac{3}{4}\Bigg]
\end{align}
is upper bound by $ \cs (d+\ln{\left(1/\delta \right)})/m $ (our \cref{induktionlemma}), and doing a union bound over all the $ 20 $ combinations of $ i,j $ one gets that \cref{eq:deterministic error} holds with probability at least $1- \delta $. We notice that if we had made less than $ 5 $ recursive calls, the above would not have worked since we then would not be able to guarantee that there were any number of majority voters left in $ \cA(\cS(\cS,T))\backslash \cA(\cS(i,T)) $ that failed to have $ 3/4 $ correct as $ 1/4-1/a \leq0$, for $ a<5 $. 

Now to show the above holds with probability at least $ 1-\delta/20 $, we first use the induction hypothesis that holds for $m=6^{k-1},$ which we know is the size of $ \rS_{i} $. Thus, we get that with probability at least $ 1-\delta/60 $ over $ \rS_{0},\rS_{i} $ that 
\vspace{-0.2cm}
\begin{align}\label{eq:errorinduktionstep}
    \p_{\rx\sim\cD}\Bigg[ \sum_{\rf\in \cA(\cS (i,T))} \frac{\ind\{\sum_{s=1}^{t} \ind\{h_{\rf,s}(x)=c(x)\}/t\geq 3/4\}}{|\cS(i,T)|}
    <3/4\Bigg]\leq 6\cs \frac{d+\ln{\left(60/\delta \right)}}{m}.
  \end{align} 
  \vspace{-0.4cm} \newline We can further restrict to the setting where the above is at least $ \cs (d+\ln{\left(1/\delta \right)})/(80m) $, since otherwise we have by the monotonicity of measures that \cref{eq:ourrecursion} is upper bounded by $ \cs(d+\ln{\left(1/\delta \right)})/m $ as noted by \cite{hannekeoptimal}. Further, we can with this lower bound, as in the above sketch for \cite{hannekeoptimal} show by a Chernoff bound that $ \rS_{j} $ contains at least $ \cs (d+\ln{\left(1/\delta \right)})/(2\cdot6\cdot80) $ points (we denote these points $ \rS_{j}\sqcap E' $) where $ \sum_{\rf\in \cA(\cS (i,T))} \ind\{\sum_{s=1}^{t} \ind\{h_{\rf,s}(x)=c(x)\}/t\geq 3/4\}/|\cS(i,T)|
<3/4 $ with probability at least $ \exp(-\cs (d+\ln{\left(1/\delta \right)})/(8\cdot 6\cdot80)) $, which for $ \cs $ sufficiently large is less than $ \delta/60 $. Thus, if we now let $ E' $ denote the following event  $\{x\in \cX: \sum_{\rf\in \cA(\cS (i,T))}\ind\{\sum_{s=1}^{t} \ind\{h_{\rf,s}(x)=c(x)\}/t\geq 3/4\}|\cS(i,T)|
<3/4 \}$,  we can use the law of conditional probability to bound \cref{eq:ourrecursion} using \cref{eq:errorinduktionstep} with probability at least $ 1-\delta/60 $ over $ \rS_{0},\rS_{i} $  
\vspace{-0.3cm}
\begin{align}\label{eq:laststepininduktionstep}
    80\max_{\rf\in\cA(\cS(\rS_{j},T))}\p_{\rx\sim\cD(\cdot|E')}\Big[ \sum_{s=1}^{t} \ind\{h_{\rf,s}(x)=c(x)\}/t< 3/4\Big] \cdot 6\cs (d+\ln{\left(60/\delta \right)})/m,  
  \end{align} 
  \vspace{-0.4cm} \newline as done similarly to \cite{hannekeoptimal}. Thus, if we can show that $\p_{\rx\sim\cD(\cdot|E')}[ \sum_{s=1}^{t} \ind\{h_{\rf,s}(x)=c(x)\}/t< 3/4]$ is uniformly bounded over the majority voters in  $ \cA(\cS(\rS_{j},T))$,  by $ 1/(80\cdot 6\cdot \ln{\left(60e \right)}) $ with probability at least $ 1-2\delta/60 $ over $ \rS_{j} $,  the upper bound of \cref{eq:ourrecursion} would follow. We notice the above since, $ \sum_{s=1}^{t} \ind\{h_{\rf,s}(x)=c(x)\}/t< 3/4 $ is equivalent to $ \sum_{s=1}^{t} h_{\rf,s}(x)c(x)/t< 1/2 $, the above uniform bound is equivalent to a uniform bound on  $ \p_{\cD_{c}(\cdot|E')}[\sum_{s=1}^{t} h_{\rf,s}(x)c(x)/t<1/2]$, the $ 1/2 $-margin loss. 
  
  We remark that the above-required uniform bound is different from and not implied by the technique used to get the above uniform bound that \cite{Optimalweaktostronglearning} want, which was a uniform bound on the classification loss of the $ \sign(\cdot) $ of a voting classifier, whereas the above is wanting something stronger, namely a uniform error bound on the $ 1/2 $-margin loss for the majority vote, which Theorem 4 of \cite{Optimalweaktostronglearning} does not imply. 

To the end of showing the above uniform bound, we now use the aforementioned property of $ \cA $ outputting majority voters having zero  $ 3/4$ margin loss on the training sequence it receives. 
To use this property, we show the following bound, \cref{lem1},
which, for $0< \gamma< 1 $ and $ \xi>1 $, states that with probability at least $ 1-\delta/60 $ over $ \rS_{j}\sqcap E' $, we have for all majority voters $ f\in\dlh=\{f:f=\sum_{s=1}^{t} h_{t}/t, \forall s\in\{1:t\}, h_{s}\in \cH\} $ that $ \p_{\rx\sim\cD(\cdot |E')}\left[f(\rx)c(\rx)\leq \gamma\right]$ is upper bounded by $ \p_{\rx\sim \rS_{j}\sqcap E'}\left[f(\rx)c(\rx)\leq \xi\gamma\right] +C(2d/(((\xi-1)\gamma)^{2}|\cS_{j}\sqcap E'|))^{0.5}+(2\ln{\left(120/\delta \right)}/|\cS_{j}\sqcap E'|)^{0.5} $, where $ C $ is some universal constant.\footnote{People familiar with ramp loss bounds combined with uniform convergence and Rademacher complexity arguments may recognize that this bound can be derived by letting the ramp start its descent at $ \gamma $ and end at $ \xi \gamma $, which gives the slope $ 1/((1-\xi)\gamma)$. We could not find the bound stated elsewhere, so we proved it in \cref{lem1}.}  

Now, using that any $ \rf \in\cA(\cS(\rS_{j},T))$, implies $ \rf=\cA(S') $, for some sample $ S'\in\cS(\rS_{j},T) $ which satisfies  $ (\rS_{j}\sqcap E')\sqsubset \rS_{j}\sqsubset S'$ being a subset of $ S' $  by \cref{alg:Subsample}  Line~\ref{alg:Subsample:recursion3}. Which, by the zero $ 3/4 $ margin loss guarantee of $ \cA $ on the training sequence it is given, implies that $ \rf $ has zero $ 3/4 $ margin loss on $ \rS_{j}\sqcap E' $. Furthermore, since  $ |\rS_{j}\sqcap E'|\geq \cs(d+\ln{\left(1/\delta \right)})/960 $ holds with probability at least $ 1-\delta/60 $ over $ \rS_{j} $,  and with probability at least $1-\delta/60 $ over $ \rS_{j}\sqcap E' $ the uniform bound over $ \dlh$, with $ \gamma=1/2 $, $ \xi=3/2 $, we conclude that with probability at least $ 1-2\delta/60 $ we have that 
\vspace{-0.4cm} 
\begin{align*}
  \max_{\rf\in\cA(\cS(\rS_{j},T))}\negmedspace\negmedspace\negmedspace\negmedspace\negmedspace\negmedspace\negmedspace\p_{\rx\sim\cD(\cdot|E')}\Big[ \sum_{s=1}^{t} \ind\{h_{\rf,s}(x)=c(x)\}/t< \frac{3}{4}\Big]\negmedspace\leq\negmedspace C\Big(\frac{30720d}{\cs(d+\ln{\left(1/\delta \right)})}\Big)^{\negmedspace\tfrac{1}{2}}+\Big(\frac{1920\ln{\left(120/\delta \right)}}{\cs(d+\ln{\left(1/\delta \right)})}\Big)^{\negmedspace\tfrac{1}{2}}
  \vspace{-0.4cm}
  \end{align*}        
which for $ \cs$ sufficiently large, is less than $ 1/(8\cdot6\cdot\ln{(60 e)}) $, which, as alluded to below \cref{eq:laststepininduktionstep} implies \cref{eq:ourrecursion}, which gives our optimal PAC error bound and inference complexity. Thus, if we can show that the training complexity is as claimed in \cref{introductionmaintheorem} with probability at least $ 1-\delta/2 $, we have shown the claim of \cref{introductionmaintheorem} with probability at least $ 1-\delta $.

However, before we do this, we make a small remark about the above argument allowing us to transition from a majority vote of majorities to a majority vote. Namely, that this argument does not lend itself to the weak to strong setup of \cite{Optimalweaktostronglearning}, as \cref{eq:deterministic error} to the best of our knowledge, then only hold with $\sum_{s=1}^{t}  \ind\{h_{\rf,s}(x)=c(x)\}/t\geq 1/2+\Theta(\gamma) $. Thus, the probability of sampling a correct voter becomes $ (3/4)(1/2+\Theta(\gamma))=3/8+(3/4)\Theta(\gamma) $, which is only greater than $ 1/2 $ if $ (3/4)\Theta(\gamma)>1/8 $, which is not always given in the weak to strong learning setup as $ \gamma $ can be arbitrarily small. Thus, another part of our observation is that we are in the realizable setup and not in the weak to strong setup, so we can get arbitrarily good margins, of the voters in the majority votes.  

Now, for the training complexity bound of $ \ah $, we need to show that the above sampling of voters can be done efficiently. The above sampling process can be described as follows: First, sample a training sequence/row $\rS' $  of $ \cS(\rS) $, then train $ \cA(\rS') $, and then sample uniformly at random a hypothesis of $ \cA(\rS') $ and repeat this $l=\Theta(\ln{\left(m/(\delta(d+\ln{\left(1/\delta \right)})) \right)})$ times. Sampling a row $ \rS' $  from $ \cS(\rS) $ can be accomplished using $ O(m) $ operations by 
drawing a $ \rw$ uniformly at random from $\left\{ 1:5 \right\}^{\log_{6}(m)}=\left\{ 1:5 \right\}^{k}$ and for each  $i=1,\ldots,k$, select the examples in $ \rS $ between $ [ 6^{k}/6^{i}\rw_i+1,(\rw_i+1)6^{k}/6^{i}] $. In this way, $ \rw_{i} $ determines which of the $ 5 $ training sequences to recurse on at the $ i $'th recursion step Line~\ref{alg:Subsample:recursion2} uniformly at random. Since it takes at most $ O(\ln{\left(m \right)}) $ operations to compute these indices,
we get that sampling a row of $ \cS(\rS) $ takes $ O(m) $ operations. 

Now since $ \cA $ had training complexity $ O(\ln{\left(m/\delta \right)} )(O(m+d\ln{\left(m \right)})+\Utrain(d)+3m\Uinf ) $ with probability at least $ 1-O(\delta/m) $, applying a union bound on the event that all the runs in $\cA(\cS(\rS,\emptyset))$ succeeds, which there is at most $ m^{\log_{6}(5)} $ of, one get that the above training complexity holds for all runs in $ \cA(\cS(\rS,\emptyset)) $, especially the one sampled,  with probability at least $ 1-\delta/2 $. Now given $ \cA(\rS') $  sampling a hypothesis from $ \cA(\rS')= \sum_{s=1}^{t} h_s/t $ takes at most $t= O(\ln{\left(m/(\delta(d+\ln{\left(1/\delta \right)})) \right)}) $ operations. Thus, the training complexity, of finding the $ l $ voters, becomes $ l $ times the computational complexity cost of running $ \cA$ which gives the training complexity      $O(\ln{m/(\delta(d+\ln{\left(1/\delta \right)}))} )\cdot\ln{(m/\delta )})\cdot (O(m+d\ln{\left(m \right)})+\Utrain(550d)+ 3m\Uinf)$
with probability at least over $ 1-\delta/2 $  over the randomness of $ \rS$ and $\cA$, which, as noted earlier, concludes the proof sketch.

\paragraph{Summary of our contribution:} \cref{maintheorem} provides an optimal PAC learner that only queries the $ \erm $  with training sequences of size $ 550d $ - positively answering \hyperref[question1]{Question~\ref*{question1}}. Moreover, the optimal PAC learner's of \cref{maintheorem} computational cost, when seen as a function of the number of training examples $ m $ runs up to a quadratic logarithmic factor in linear time - answering \hyperref[question2]{Question~\ref*{question2}} positively up to the quadratic logarithmic factor - however as mentioned in the perceptron example $ U_{T}(d) $ may depend on $ m $, and it may also be the case for $ U_{I} $. Furthermore, the inference complexity off \cref{maintheorem} is asymptotically the best among the known optimal PAC learners. 

The key observations that allowed for these positive results were seeing that $ \erm $ trained on $ \Theta(d) $ points allowed for creating majority voters with good margins, and it implied that the purely analytical structure $ \cA(\cS(\rS)) $ could be shown to have good margins on both the majority voter level and the majority of majorities level. This combined with the observation that $ \ah $ could be seen as sampled from the voters in the majorities of $ \cA(\cS(\rS)) $, led to the inference complexity of $\ah$ not suffering a blow-up from the boosting step, a close to linear training complexity, and an optimal PAC error bound. 

\paragraph{Use of randomness:}Before we move on to describe the structure of the paper we want to mention some limitations we see off our work. Our algorithm $ \ah $ uses randomness which current computational cost is accounted for by saying that all random variables has to be read when used and that this takes $ 1 $ operation - so we are not taking into account if there is some computational cost in creating the randomness. To make it clear how much and what randomness we use we state it here: we use $ \Theta(d\ln{(m/\delta)}) $ uniform random variables on the interval $ [0:1] $,  $ \Theta(\ln{(m )}\cdot\ln{(m/(\delta(d+\ln{(1/\delta )})) )})$ uniform random variables on the discrete values  $ \{ 1,2,3,4,5 \} $, and $ \Theta(\ln{(m/(\delta(d+\ln{(1/\delta )})) )})$  uniform random variables on the discrete values $ \{ 1,\ldots,\Theta(\ln{(m )}) \} $. Further the model we use to count computational cost is arguably simplistic, and could be refined further.     

\paragraph{Structure of the paper:}The next section introduces the notation and preliminaries used in this work. \cref{sec:optimalityah} gives the proof of our optimal PAC learner using properties of $ \cA(\cS(S,\emptyset)) $ proved in \cref{sec:optimalitya} and properties of $ \cA $ proved in \cref{sec:propAdaBoostSample}, where some of the proofs for the properties are deferred to the three appendix one for each of the above-mentioned sections. 

\section{Notation and Preliminaries}\label{sec:notation}

We work with a countable universe $\cX$, and countable hypothesis class $\cH$ to ensure sufficient measurability conditions.\footnote{See, e.g., \cite{measureone} and  \cite{measuretwo} for more sufficient measurability assumptions} We will write random variables $ \rx $ with boldface letters, and non-random $ x $ with non-bold face letters.  

We use $(\cX\times\{-1,1\})^{m}$ to denote training sequences of length $m$, with repetition. Further, we let $(\cX\times\{-1,1\})^{*}=\cup_{i=1}^{\infty}(\cX\times\{-1,1\})^{i}$ denote all possible finite training sequences. For sequences $S,T\in (\cX\times\{-1,1\})^{*}$, we write $S\sqcup T$ for the concatenation/union of the two sequences, i.e., $[S^T,T^T]^T$ where $^{T}$ denotes the transpose of the sequence, i.e. with repetitions. Furthermore, we call $ S' $  a sub training sequence of $ S $ if, for every $ (x,y)\in S' $, we have $ (x,y)\in S $. We sometimes write this as $S'\sqsubseteq S$. Note that $ S' $ may have a different multiplicity of a training example $ (x,y) $ than $ S $, possibly larger. For a set $A\subset \cX$ (without repetitions) we write $S\sqcap A$ for the training sequence $[S_{l}\mid S_{l,1}\in A]$, i.e., the sub training sequence of training examples $S_{l}$ of $S$ that has their feature $S_{l,1}$ in $A$.    

For $ a,b\in \mathbb{R} $ with $ a<b $, we write $ [a:b]=\{ x\in\mathbb{R} \mid a\leq x\leq b  \} $ and $ [a:b]^{*}=\cup_{i=1}^{\infty}[a:b]^{i} $ and $ ([a:b]^{*})^{*}=\cup_{i,j\in \mathbb{N}}([a:b]^{j})^{i} $. For $ a,b\in \mathbb{N} $ with $ a<b $, we write $ \{  a:b\}=\{ x\in\mathbb{N} \mid a\leq x \leq b \} $, with similar definitions for $ \{ a:b \}^{*} $ and $ (\{ a:b \}^{*} )^{*}$. When we for a random variable $ \rr $ write $ \rr\sim ([a:b]^{*})^{*} $, we mean that the $ \rr_{i,j} $'s are i.i.d. random variables uniformly drawn from $ [a:b] $. Similarly, for $ \rw \sim (\{ a:b \}^{*})^{*} $.

For a distribution $ \cD $ over $ (\cX\times \{-1,1\}) $, we define the error under $ \cD $ as  $ \ls_{\cD}(h)=\p_{(\rx,\ry)\sim \cD}[h(x)\not=y]$ for $ h\in\{-1,1\}^{\cX} $. Furthermore, for a distribution $ \cD $ over $ \cX $  and a target concept $ c \in \cH$, we define the distribution $ \cD_{c} $ over $ (\cX\times \{-1,1\})^{*} $  as having measure $ \p_{\rx\sim \cD}[(\rx,c(\rx))\in A] $ for $ A \subset (\cX\times \{-1,1\})^{*}$.  

For a training sequence $ S\in(\cX\times \{-1,1\})^{*} $ and a hypothesis class $ \cH $, we say that $ S $ is realizable by $ \cH $ if there exists $h\in\cH  $ such that for all examples $(x,y) \in S $,  we have $ h(x)=y $.
For a training sequence $ S\in(\cX\times \{-1,1\})^{*} $ and a target concept $ c\in \cH $, we say that $ S $ is realizable by $ c $ if for all examples $(x,y) \in S $, we have $ c(x)=y $. Furthermore, for a distribution $ \cD $ over $ (\cX\times\{-1,1\})^{*} $ and $ \cH $, we say that $ \cD $ is realizable by $ \cH $ if, for any $ m\in\mathbb{N} $ and any realization $ S $ of $ \rS\sim \cD^{m} $, $ S$ is realizable by $ \cH $.          

We say that a learning algorithm is a $ \erm $ learning algorithm for a hypothesis class $ \cH $  if, given a realizable training sequence $ S \in(\cX\times\{-1,1\})^{*}$ by $ \cH $, it outputs $h=\erm(S) $ such that $ S $ is realizable by $ h $ and $ h\in \cH $.   

We define the VC-dimension of a hypothesis class $ \cH\subset \{ -1,1 \}^{\cX} $ as the largest number $ d $ such that there exists a point set $ x_1,\ldots,x_d \in \cX$ where, for each $ y\in\{ -1,1 \}^{d} $, there exists $ h\in \cH $ such that $ (h(x_1),\ldots,h(x_{d}))=y$. For a hypothesis class $ \cH $ and  $t\in\mathbb{N}$, we write $\dlh=\{f:f=\sum_{i=1}^{t} h_{t}/t, \forall i\in\{1:t\}, h_{i}\in \cH\}$ for the class of linear combination of $ t $ classifiers in $ \cH $.       
We assume that the following operations cost one unit of computation: reading an entry, comparing two numbers, adding, multiplying, calculating $\exp(\cdot)$ and $\ln{\left(\cdot \right)}$, rounding to a natural number

For a $ \erm $ algorithm for a hypothesis class $ \cH $  we define the training cost  $ \text{U}_{\text{Train}}(\cdot):=\Utrain(\cdot):\mathbb{N}\rightarrow\mathbb{N}$ for $ m\in\mathbb{N} $    as the maximal number of operations needed to find the function $ \erm(S) \in \cH$ given any consistent training sequence $S\in (\cX\times\{-1,1\})^{m}$ of size $ m $, i.e.,    $$\Utrain(m):=\sup_{\stackrel{S\in (\cX\times\{-1,1\})^{m}}{S \text{ consistent with } \cH}} \# \{ \text{Operations to}\text{ find } \erm(S)  \}.$$  Further we define the inference cost $\text{U}_{\text{Inference}}:= \Uinf\in \mathbb{N}$ as the maximal number of operations needed to calculate the value of $ h(x) $ for any $ x $ given any $ h=\erm(S)$, where $ S\in (\cX\times \{-1,1\})^{*} $ is consistent with $ \cH $, i.e., 
$$\Uinf=\sup_{\stackrel{h=\erm(S), S\in (\cX\times\cY)^{*}}{S \text{ consistent with } \cH,\text{ }x\in\cX}} \# \{ \text{Operations to calculate } h(x)  \}.$$

For a learning algorithm $ \cA:(\cX\times \{ -1,1 \} )^{*}\rightarrow \{ -1,1 \}^{\cX} $, we define the training complexity of $ \cA $ for an integer $ m $ as the worst case number of operations made by the learning algorithm when given a realizable training sequence by $ \cH $ of length $ m $ , i.e.  $$\sup_{\stackrel{S\in(\cX\times \{-1,1\})^{m}}{ S \text{ realizable by } \cH  }} \# \{\text{Operations to find } \cA(S)\}.$$ The inference complexity of a learning algorithm $ \cA $ for an integer $ m $ we define as the worst case cost of predicting a new point $ x\in \cX $ for the learned mapping $f= \cA(S) $, given a realizable training sequence by $ \cH $ of length $ m $   i.e. $$ \sup_{\stackrel{f=\cA(S),S\in (\cX\times \{-1,1\})^{m}}{ S \text{ is realizable by } \cH, \text{ } x\in \cX}} \#\{\text{Operations to calculate } f(x)\}.$$
 \section{Efficient Optimal PAC Learner}\label{sec:optimalityah} 
To introduce our algorithm, we need the following subsampling algorithm, which is inspired by that of \cite{hannekeoptimal} \cref{alg:Subsamplehanneke}. We make $6$ sub-sequences and $ 5 $  recursive calls whereas \cite{hannekeoptimal} make $4$ sub-sequence and $ 3 $ recursive calls. Furthermore, we here use nonoverlapping subsequence except on the part that is recursed on. The reason for the $6$ sub-sequences, as mentioned earlier, is to ensure that $3/4$ of the majority voters have $3/4$ of their voters correct. The nonoverlapping subsequence, except on the part that is recursed on, was chosen to simplify notation. We will refer to \cref{alg:Subsample} as $\cS$.
\begin{algorithm}[H]
    \caption{Subsampling algorithm $\cS$}\label{alg:Subsample}
    \begin{algorithmic}[1]
        \State \textbf{Input:} Training sequences $S,T \in (\cX \times \cY)^{*}$.
        \State \textbf{Output:} A sequence of sub-training sequences of $S \sqcup T$.
        \State \textbf{if $|S| \geq 6$ then}
            \State\hspace{0.4cm}Split $S$ into $S_0, S_1, S_2, S_3, S_4, S_5$ where $S_i$ contains examples from  $i|S|/6+1$ to $(i+1)|S|/6$.\label{alg:Subsample:recursion}
            \State\hspace{0.4cm}\textbf{return} $[\cS(S_0, S_1 \sqcup T), \cS(S_0, S_2 \sqcup T), \cS(S_0, S_3 \sqcup T), \cS(S_0, S_4 \sqcup T), \cS(S_0, S_5 \sqcup T)].$\label{alg:Subsample:recursion2}
        \State \textbf{else}
            \State\hspace{0.4cm}\textbf{return} $S \sqcup T.$\label{alg:Subsample:recursion3}
    \end{algorithmic}
\end{algorithm}

In the following, we will for $S\in (\cX\times \cY)^{m}$ and $m=6^{k}$ for $k\in\mathbb{N}$ and $T\in (\cX\times \cY)^{*}$ let $\cS(S,T)$  be the matrix with rows corresponding to the sub training sequences that $\cS(S,T)$ produces. The matrix is given recursively by the following equation:  
\begin{align*}
    \cS(S,T)=
    \begin{bmatrix}
      \cS(S_0,S_{1}\sqcup T)\\
      \cS(S_0,S_{2}\sqcup T)\\
      \cS(S_{0},S_{3}\sqcup T)\\
      \cS(S_{0},S_{4}\sqcup T)\\
      \cS(S_{0},S_{5}\sqcup T)
    \end{bmatrix}.
\end{align*}
We will for short write $\cS(i,T)$ instead of $\cS(S_{0},S_{i}\sqcup T)$. We notice that since each of the recursive calls creates $5$ sub calls, and we assume that $m=6^{k}$, there will be $5^{k}=5^{\log_{6}(m)}=m^{\log_{6}(5)}\approx m^{0.898}$ many sub training sequences/rows in the above matrix. Furthermore, each sub training sequence has size $ m'=\sum_{i=1}^{k} m(1/6)^{i}+1$.

In the following we will use $\cA$ to refer to \cref{alg:AdaBoostSample} $\as$, which takes as input a training sequence $ S\in(\cX\times \cY)^{*}$ and a string $ r\in ([0:1]^{*})^{*}$. We also assume that $ \cA $ has query access to a $ \erm $ algorithm.  
Further for $\cA$ run on the family of sub training sequences $\cS(i,T)$ and string $r\in ([0:1]^{*})^{*}$ we will write  $\cA(i,T,r)$ for the sequence of hypotheses $[\cA(S',r)]_{S'\in \cS(i,T)}$. We also view $\cA(\cS(S,T),r)$ as the following recursively defined matrix:

\begin{align*}
    \cA(\cS(S,T),r)=
    \begin{bmatrix}
      \cA(\cS(S_0,S_{1}\sqcup T),r)\\
      \cA(\cS(S_0,S_{2}\sqcup T),r)\\
      \cA(\cS(S_{0},S_{3}\sqcup T),r)\\
      \cA(\cS(S_{0},S_{4}\sqcup T),r)\\
      \cA(\cS(S_{0},S_{5}\sqcup T),r)
    \end{bmatrix}.
\end{align*}

Our learner's behavior is closely related to the behavior of the hypotheses in the above matrix. Thus, we will now state some lemmas that describe the behavior of the hypotheses in the above matrix. 

The first lemma shows that the hypothesis returned by $\cA$ is always a majority vote, and it achieves a margin of at least $ 3/4$  on all examples in the training sequence $S$ it is trained on. The proof of the lemma can be found in \cref{sec:propAdaBoostSample}.

\begin{restatable}{lemma}{lemmaadaboostsampleeasyproperties}\label{lemma:adaboostsampleeasyproperties}
  For a hypothesis class $\cH$ of VC-dimension $d$, target concept $ c\in \cH $, training sequence size $m\in\mathbb{N}$, training sequence $S\in \left(\cX\times\cY\right)^{m}$ 
  realizable by $c$,  and string $r\in([0:1]^{*})^{*}$ the output $f=\cA(S,r)$ of \cref{alg:AdaBoostSample}, when run on 
  $S$ and $r$ is in $f\in\dlh$ for $t=\left\lceil20^{2}\ln{(m)}/2\right\rceil$ and satisfies: 
  $\sum_{x\in S} \ind\{\sum_{i=1}^{t} h_{f,i}(x)c(x)/t\leq \theta \}/m <1/m$ and $\theta=3/4$.
\end{restatable}
The next lemma shows that $\cA$, when run on a training sequence $S$ and a random string $\rr$, where $ \rr $  is used to sample from the sequential distributions $ \cA $ makes, with high probability, performs few operations. The proof of the lemma can be found in \cref{sec:propAdaBoostSample}

\begin{restatable}{lemma}{lemmaAdaBoostSampleefficent}\label{lemmaAdaBoostSampleefficent}
  For a hypothesis class $\cH$ of VC-dimension $d$, target concept $c\in \cH  $, training sequence size $ m\in\mathbb{N} $, training sequence   $S\in(\cX\times \cY)^{m}$ realizable by $c$, failure parameter $0<\delta <1$, random string length $n\in\mathbb{N}$, and random string $\rr\sim ([0:1]^{550d})^{n}$ we have for $t=\left\lceil20^{2}\ln{(m)}/2\right\rceil$ as in Line~\ref{alg:AdaBoostSampletset} and $n\geq 6\lceil20^{2}\ln{(8m/\delta)}/2\rceil$ that with probability at least $1-(\delta/(8m))^{20}$ over $ \rr $  \cref{alg:AdaBoostSample} run on $S,\rr$, $\cA(S,\rr)$, uses no more than
$n\cdot (O(m+d\ln{(m )})+\Utrain(550d)+3m\Uinf)$ operations.
\end{restatable}

Thus, from the above lemma, it follows that we can apply a union bound over all the hypotheses/rows in $\cA(\cS(S,\emptyset),\rr)$, which are at most $\leq m^{0.9}$, and conclude that, with probability at most $1-(\delta/(8m))^{19}$, any row of $\cA(\cS(S,\emptyset),\rr)$ uses few operations. 

The next lemma is inspired by the main theorem in \cite{hannekeoptimal} which proof is sketched in \cref{sec:proofoverview}. Specifically, it says that the majority vote of the hypotheses in $\cA(\cS(\rS,\emptyset),r)$ for a random sample $\rS$ and any string $r$, with probability $1-\delta$ over $\rS$, is such that fewer than $3/4$ of the majority voters in the majority vote of majorities has $3/4$ of its voters being correct on a new example $ (\rx,c(\rx)) $ has probability $O(d+\ln{\left(1/\delta \right)}  )/m$ over $ (\rx,c(\rx)).$

\begin{lemma}[\cref{induktionsteplemma} with $T=\emptyset$]\label{induktionsteplemmaeasy}
  There exists a universal constant $\cs$ such that for: Distribution $ \cD $ over $ \cX $, hypothesis class $ \cH $ of VC-dimension $ d $, target concept $ c\in \cH $,  failure probability $ 0<\delta<1 $, training sequence size $m=6^{k}$ for some $k\geq 1$, and string $r\in ([0,1]^{*})^{*}$, it holds with probability at least $1-\delta$ over $\rS\sim \cD_{c}^{m}$ that
     \begin{align*}
       \p_{\rx\sim \cD} \left[  
         \sum_{\rf\in \cA(\cS(\rS,\emptyset),r)} \frac{\ind\{\sum_{i=1}^{t} \ind\{h_{\rf,i}(\rx)=c(\rx)\}/t\geq 3/4\}}{|\cS(\rS,\emptyset)|}
          <3/4\right] \leq \cs \frac{d+\ln{\left(1/\delta \right)}}{m}.
     \end{align*}
\end{lemma}

The last lemma we introduce, states that we can efficiently sample with replacement from the rows of $\cA(\cS(S,\emptyset),r)$. The proof of the lemma can be found in \cref{appendixefficentoptimalpaclearner}
\begin{restatable}{lemma}{samplingfromrows}\label{samplingfromrows}
  Let $S\in (\cX\times \cY)^{*}$, with $|S|=m=6^{k}$ for $k\geq 1$.  Let $g'$ be the function from  $\{1:5\}^{k}$ into $\left([1:m]^{2}\right)^{k}$ defined by $g'(w)_{j}=[ 6^{k}w_j/6^{j}+1,6^{k}(w_j+1)/6^{j}]$ for $j\in \{1:k\}$. For $w\in\left\{ 1:5 \right\}^{k}$, we denote $ S[g'(w)] $ as a sub training sequence of $ S $, which can be found using $S$, $ w $, and $g'$ in $O(m)$ operations, and when viewed as a function of $ w\in\{ 1:5 \}^{k} $, $ S[g'(w)] $ is a bijection into $ \cS(\cS,\emptyset).$ \footnote{For readability, we here write it as we find the training sequence $ S[g'(w)],$ which would might imply reading the whole training sequence $ S,$ but what we actually find is the indexes of $ S[g'(w)],$ which only requires looking at the numbers $ [m].$ We will only read the training examples when training an $ \erm $ captured in $ \Utrain $ or doing inference captured in $ \Uinf. $     }
\end{restatable}
With the above lemmas in place, we now present the idea of our efficient PAC learner which is simply sampling majority voters/rows with replacement of the matrix $\cA(\cS(\rS,\rr))$ and for each sampled majority voter, sample a voter from the majority voter. By \cref{induktionsteplemmaeasy} one can deduce that one such sampled voter, with probability $(3/4)^{2}=9/16=1/2+1/16$, will correctly classify a new example. Since this probability is strictly greater than $1/2$, we can by repeating the sampling procedure $\Theta(\ln{\left(m/(\delta(d+\ln{\left(1/\delta \right)}) )\right)})$ times, ensure that the majority of the sampled hypotheses, with probability at least $1-\Theta(\delta)$ over the sampled hypotheses, will have error less than $O(m/\left(d+\ln{\left(1/\delta \right)}\right))$ (using Chernoff and Markovs inequality). In the above sampling a majority voter/row from $\cA(\cS(\rS,\rr))$ refers to sampling a sub training sequence/row from $\cS(\rS,\emptyset)$, which can be done efficiently by \cref{samplingfromrows}, and then run $\cA(\cdot,\rr)$ on the sampled sub training sequence, followed by sampling a random voter from the output of $ \cA $. That $ \cA $ can be run efficiently on the sampled sub training sequences from $ \cS(\rS,\emptyset) $ follows by the union bound argument over runs $ \cA(\cS(\rS,\emptyset)) $ described under \cref{lemmaAdaBoostSampleefficent}. We will make the above argument formal soon. With the intuitive explanation of our learner presented, we now state the algorithm explicitly. We will use $ \ah $ to denote the learner.
\begin{algorithm}[H]
    \caption{Random majority voter $\ah$}\label{alg:Randommajority}
    \begin{algorithmic}[1]
        \State \textbf{Input:} Training sequence $S \in (\cX \times \cY)^{*}$ such that $|S|=6^{k}$ for $k \in \mathbb{N}$, string $r \in ([0:1]^{*})^{*}$, failure parameter $0<\delta<1$, VC-dimension $d.$
        \State \textbf{Output:} $ \sign(\cdot) $ of a majority vote.
        
        \State $m \gets |S|$ \label{alg:Randommajoritysetm}
        \State $k \gets \log_{6}(m)$ \label{alg:Randommajoritysetk}
        \State $l \gets \lceil 16 \cdot 200 \ln{(m/(\delta(d + \ln{\left(1/\delta \right)})))}/9\rceil$ \label{alg:Randommajoritysetl}
        \State $f \gets [0]^{l}$ \label{alg:Randommajoritysetf}
        \State $m'=\sum_{i=1}^{k} m(1/6)^{i}+1$\label{alg:Randommajoritysetsubtraingsequencesamplesize}
        \State $t \gets \lceil 20^{2} \ln{(m')}/2\rceil$ \label{alg:Randommajoritysett}

        \State \textbf{for $i = 1$ \textbf{to} $l$} \label{alg:Randommajorityforloopstart} 
        \State\hspace{0.5cm}\vline\hspace{0.1cm}$\rw_{i} \sim \{1:5\}^{k}$ \label{alg:RandommajoritysamplingS}
        \State\hspace{0.5cm}\vline\hspace{0.1cm}$S_{i} \gets S[g'(\rw_{i})]$ \label{alg:Randommajoritysamplingefficient}
        \State\hspace{0.5cm}\vline\hspace{0.1cm}$f_{\rw_{i}}=\sum_{l=1}^{t}h_{f_{\rw_{i}},l}/t \gets \cA(S_{i}, r)$ \label{alg:Randommajorityrunsadaboostsample}
        \State\hspace{0.5cm}\vline\hspace{0.1cm}$\rz_{i} \sim \{1:t\}$ \label{alg:Randommajoritysamplingindex1}
        \State\hspace{0.5cm}\vline\hspace{0.1cm}$f \gets f + h_{f_{\rw_{i}}, \rz_{i}}$ \label{alg:Randommajorityaddshypothesis}
        \State \textbf{return }{$\sign(f)$.} \label{alg:Randommajorityoutputofhat}
    \end{algorithmic}
\end{algorithm}
Line~\ref{alg:Randommajoritysetm} reads the number of training examples $(x_i,y_i)\in\cX\times  \cY$ in $S.$
Line~\ref{alg:Randommajoritysetl} sets the number of hypotheses in the final majority vote.
Line~\ref{alg:Randommajoritysetf} initializes the majority voter as the constant $0$ function/array of size $ l.$
Line~\ref{alg:Randommajoritysetsubtraingsequencesamplesize} calculates the size of the sub training sequences in $ \cS(S,\emptyset) $, which will equal the training sequence size that $ \cA $ uses in its stopping criteria/number of hypotheses to include in its majority vote.
Line~\ref{alg:Randommajoritysett} equals the number of hypotheses in the output of $\cA$ for each call.
Line~\ref{alg:RandommajoritysamplingS} samples a random vector $\rw_{i}$ of length $k$ with i.i.d. entries uniformly chosen from $\{1:5\}$.
Line~\ref{alg:Randommajoritysamplingefficient} corresponds to sampling the row $S[g'(\rw_{i})]$ from $\cS(S,\emptyset)$, using the method in \cref{samplingfromrows}.
Line~\ref{alg:Randommajorityrunsadaboostsample} runs $\cA$ on the training sequence $S_{i}$ and the string $r$ to get a majority vote, using \cref{lemma:adaboostsampleeasyproperties} to express it as a normalized sum. 
Line~\ref{alg:Randommajoritysamplingindex1} samples a random index $\rz_{i}$ uniformly from $\{1:t\}$, which, combined with Line~\ref{alg:Randommajorityaddshypothesis}, corresponds to uniformly sampling a hypothesis $h_{f_{\rw_{i}},\rz_{i}}$ from the majority vote $f_{\rw_{i}}$ and adding it to $ f $.

For the analysis of \cref{alg:Randommajority}, we use $\rw=(rw_{1},\ldots,rw_{l})$ and $\rz=(\rz_{1},\ldots,\rz_{æ})$ to denote the collections of independent random variables $\rw_i$ and $\rz_i$ used during the $l$ rounds. Furthermore, we will use $\ah_{\delta}(S,r,\rw,\rz)$ to denote the output of \cref{alg:Randommajority}, suppressing the parameter $ d $ in the notation. With our algorithm introduced, we now give the formal statement of our main theorem and the proof of it. 
\begin{restatable}{theorem}{maintheorem}\label{maintheorem}
  There exists a universal constant $\cs$ such that: for failure parameter $0<\delta<1$, hypothesis class $\cH$ of VC-dimension $d$, distribution $\cD$,  target concept $c\in \cH$, training sequence size $ m=6^{k}$ for $k\in\mathbb{N},$ random string size $n=6\lceil20^{2}\ln{(8m/\delta)} /2\rceil$, random string $\rr\sim ([0:1]^{550d})^{n}$, and training sequence $\rS\sim\cD_{c}^{m}$, the learner $ \ah_{\delta}(\rS,\rr,\rw,\rz) $ will, with probability at least $ 1-\delta $ over $ \rS,\rr,\rw $, and $ \rz $, output $ \sign(\cdot) $ of a majority vote  consisting of $l=\lceil 16\cdot 200 \ln{(m/(\delta(d+\ln{\left(1/\delta \right)})) )}/9 \rceil$ hypotheses from $\cH$, with error $  \ls_{\cD_{c}}(\ah_{\delta}(\rS,\rr,\rw,\rz))\leq (4+\cs) (d+\ln{\left(4/\delta \right)})/m $, inference complexity $ O(\ln{[m /\delta(d+\ln{\left(1/\delta \right)}) ]}) \Uinf $ and computational complexity\footnote{The proof actually shows that with probability $ 1-\delta/2 $ the error is as stated and with probability at least $ 1-(6\delta/(8m))^{19} $ the training and inference complexity is as stated.  }
  \begin{align*}
    O\bigg(\negmedspace\ln{\bigg(\frac{m}{\delta(d+\ln{\left(1/\delta \right)})} \bigg)}\negmedspace\cdot\negmedspace\ln{\bigg(\frac{m}{\delta} \bigg)}\negmedspace\bigg)\negmedspace\cdot \negmedspace\Big(O\Big(m+d\ln{\left(m \right)}\Big)\negmedspace+\negmedspace\Utrain(550d)\negmedspace+\negmedspace 3m\Uinf\negmedspace\Big).
    \end{align*} 
\end{restatable}

\begin{proof}[Proof of \cref{maintheorem}.]
Let $\cs$ be the universal constant from \cref{induktionsteplemmaeasy}. We prove the above by showing that respectively, with probability at least $1-\delta/2$ over $\rS,\rr,\rw$ and $\rz$, the error is upper bound as stated above, and the inference and training complexity is upper bounds as stated above,  whereby a union bound completes the argument. We start with the error bound.

To this end consider any realization $r$ of $\rr$ and $S$ of $\rS\sim \cD_{c}^{m}$. We now define the event  $E_{S,r}=\{x:     
  \sum_{\tilde{f}\in \cA(\cS(S,\emptyset),r)} \ind\{\sum_{i=1}^{t} \ind\{h_{\tilde{f},i}(x)=c(x)\}/t\geq 3/4\}/|\cS(S,\emptyset)|
   <3/4 \}$  over $\cX$, i.e., the event where fewer than $3/4$ of the majority votes in $\cA(\cS(S,\emptyset),r)$ are correct on at least  $3/4$ of the votes. If $\p_{\rx\sim \cD}[\overline{E_{S,r}}] \leq (d+\ln{\left(1/\delta \right)})/m$, we get by using $\p(A)=\p(A\cap B)+\p(A\cap \bar{B})\leq\p(B)+\p(\bar{B})$ that
  \begin{align}\label{eq:maincorrectness1}
      \ls_{\cD_{c}}(\ah_{\delta}(S,r,\rw,\rz)) =\p_{\rx\sim \cD}\left[\ah_{\delta}(S,r,\rw,\rz)(\rx)\not=c(\rx)\right]
      \leq \p\left[E_{S,r}\right]+(d+\ln{(1/\delta)})/m
  \end{align}
Assume now that $ \p_{\rx\sim \cD}\left[\overline{E_{S,r}}\right] \geq (d+\ln{\left(1/\delta \right)})/m >0 $. By the law of total probability, we split the loss of $\ah$ into two parts:
   \begin{align}\label{eq:maincorrectness2}
    \ls_{\cD_{c}}(\ah_{\delta}(S,r,\rw,\rz)) 
    \leq \p_{\rx\sim \cD}\left[E_{S,r}\right]+\p_{\rx\sim \cD}\left[\ah_{\delta}(S,r,\rw,\rz)(\rx)\not=c(\rx)|\overline{E_{S,r}}\right]
   \end{align}
where $ \overline{E_{S,r}} $ denotes the complement of $E_{S,r}$, i.e., $\overline{E_{S,r}}=\{x:     
  \sum_{\tilde{f}\in \cA(\cS(S,\emptyset),r)} \ind\{\sum_{i=1}^{t} \ind\{h_{\tilde{f},i}(x)=c(x)\}/t\geq 3/4\}/|\cS(S,\emptyset)|
   \geq3/4 \}$ - the event where $3/4$ of the majority voters in $\cA(\cS(S,\emptyset))$ have $3/4$ of their votes correct. 

Now consider any $x\in \overline{E_{S,r}}$. We now notice that by Line~\ref{alg:RandommajoritysamplingS}, Line~\ref{alg:Randommajoritysamplingefficient} and \cref{samplingfromrows}, we have $S_i=S[g'(\rw_{i})]\in\cS(S,\emptyset)$. Furthermore, since $g'(\rw_{i})$ is a bijection into the training sequences/rows of $\cS(S,\emptyset)$ by \cref{samplingfromrows} and $ \rw_{i} \sim \{ 1:5 \}^{k}$, $ S_{i} $ can be seen as sampled uniformly at random from the training sequences of $ \cS(S,\emptyset) $. Hence, Line~\ref{alg:Randommajorityrunsadaboostsample} can be viewed as uniformly drawing a majority vote/row from $\cA(\cS(S,\emptyset),r)$. Furthermore since any $ S'\in \cS(S,\emptyset,r) $ has size $ |S'|=\sum_{i=1}^{k}m(1/6)^{i}+1$, which is $ m' $ in Line~\ref{alg:Randommajoritysetsubtraingsequencesamplesize}, and by \cref{lemma:adaboostsampleeasyproperties}, we have that  $ \cA(S_{i},r) $ is a majority voter over $ \lceil20^{2}\ln{(|S_{i}|)}/2 \rceil=\lceil20^{2}\ln{(m')}/2 \rceil $ voters, which is $ t $ in Line~\ref{alg:Randommajoritysett}.  
Thus, since we have by Line~\ref{alg:Randommajoritysamplingindex1} that $\rz_{i}$ is uniformly chosen between $\{ 1:t \}$, we conclude that $h_{f_{\rw_{i}},\rz_{i}}$ in Line~\ref{alg:Randommajorityaddshypothesis} can be seen as uniformly chosen from the voters in $ f_{\rw_{i}} =\sum_{l=1}^{t}h_{f_{\rw_{i}},l}/t$. Therefore, since $ f_{\rw_{i}} $ was uniformly chosen among all $ \cA(\cS(S,\emptyset),r) $ and $ h_{f_{\rw_{i}},\rz_{i}} $ was uniformly chosen among the voters of $ f_{\rw_{i}} $, we conclude that for $x\in \overline{E_{S,r}}$, i.e., such that $\sum_{\tilde{f}\in \cA(\cS(S,\emptyset),r)} \ind\{\sum_{i=1}^{t} \ind\{h_{\tilde{f},i}(x)=c(x)\}/t\geq 3/4\}/|\cS(S,\emptyset)|\geq3/4 $, we have:
   \begin{align*}
    &\p_{\rw_{i},\rz_{i}}\left[h_{f_{\rw_{i}},\rz_{i}}(x)=c(x)\right]\\
    \geq
    &\p_{\rw_{i},\rz_{i}}\Big[h_{f_{\rw_{i}},\rz_{i}}(x)=c(x)\Big| \sum_{j}^{t} \frac{\ind\{h_{f_{\rw_{i}},j}(x)=c(x)\}}{t}\geq \frac{3}{4}\Big] \p_{\rw_{i}}\Big[\frac{\ind\{h_{f_{\rw_{i}},j}(x)=c(x)\}}{t}\geq \frac{3}{4}\Big]\\
    \geq& 
    (3/4)^{2}=9/16>1/2.
   \end{align*}
Since $\ah_{\delta}(S,r,\rw,\rz)=\sign(\sum_{i=1}^{l}  h_{f_{\rw_{i}},\rz_{i}})$, we have that $\ah_{\delta}(S,r,\rw,\rz)(x)=c(x)$ if the number of $i$'s such that $h_{f_{\rw_{i}},\rz_{i}}(x)$ is equal to $c(x)$ is strictly greater than $l/2$. If we let $\rX= \sum_{i=1}^{l}\ind\{h_{f_{\rw_{i}},\rz_{i}}(x)=c(x)\}$, i.e., $\rX$ is the number of $ i $'s such that $h_{f_{\rw_{i}},\rz_{i}}(x)=c(x)$, we have that $\rX$ has an expectation of at least  $\e\left[\rX\right]\geq 9l/16$. Since $ \rX $ is a sum of i.i.d. $\left\{ 0,1 \right\}$-random variables, by Chernoff and $l=\lceil 16\cdot200 \ln{(m/(\delta(d+\ln{(1/\delta )})) )}/9 \rceil$ we have:
\begin{align*}
 \p\left[\rX\leq (1-1/10)9l/16\right]
\leq  \exp{\left(-9l/(16\cdot 200) \right)} \leq \delta (d+\ln{\left(1/\delta \right)})/m.
\end{align*} 
Since $(1-1/10)9l/16=81l/160>l/2$, we conclude that for $x\in \overline{E_{S,r}}$, with probability at least $1- \delta (d+\ln{\left(1/\delta \right)})/m$ over $ \rw,\rz $, we have that $\ah_{\delta}(S,r,\rw,\rz)(x)=c(x)$. Thus, by independence of $\rw,\rz,\rx$ we conclude that: 
\begin{align*}
 \e_{\rw,\rz}[\p_{\rx\sim\cD}[\ah_{\delta}(S,r,\rw,\rz)(\rx)\not=c(\rx)\mid \overline{E_{S,r}}]]
\leq \delta (d+\ln{\left(1/\delta \right)} )/m,
\end{align*}
and further by Markov's inequality that:
\begin{align}\label{eq:maincorrectness3}
  \p_{\rw,\rz}\left[\p_{\rx\sim\cD}\left[\ah_{\delta}(S,r,\rw,\rz)(\rx)\not=c(\rx)\mid \overline{E_{S,r}}\right]\geq 4(d+\ln{\left(1/\delta \right)})/m\right]\leq \delta /4.
\end{align}
Thus, by combining \cref{eq:maincorrectness1}, \cref{eq:maincorrectness2} and \cref{eq:maincorrectness3}, we conclude that for any realization $S$ and $r$ of $\rS$ and $\rr$, we have with probability at least $1-\delta/4$ over $\rw$ and $\rz$ that
$\ls_{\cD_{c}}\left(\ah_{\delta}(S,r,\rw,\rz)\right)\leq \p_{\rx\sim D}\left[E_{S,r}\right]+4(d+\ln{(1/\delta)})/m.$
Whereby independence of $\rS,\rr,\rw$ and $\rz$ implies:
\begin{align}\label{eq:unionboundinrandom}
  \p_{\rS,\rr,\rw,\rz}\left[ \ls_{\cD_{c}}\left(\ah_{\delta}(\rS,\rr,\rw,\rz)\right) \leq \p_{\rx\sim D}\left[E_{\rS,\rr}\right]+4(d+\ln{\left(1/\delta \right)})/m\right]
  \geq 1-\delta/4.
\end{align}
We now show that with probability at least $1-\delta/4$ over $\rS,\rr,\rw$, and $\rz$, we have $\p_{\rx\sim\cD}[E_{\rS,\rr}]\leq \cs(d+\ln{\left(4/\delta \right)})/m$. Thus combining this with \cref{eq:unionboundinrandom} and applying a union bound it holds with probability at least $1-\delta/2$ over $\rS,\rr,\rw$ and $\rz$ that $
  \ls_{\cD_{c}}(\ah_{\delta}(S,r,\rw,\rz))\leq(\cs+4) (d+\ln{\left(4/\delta \right)})/m,$
which would conclude that the error bound of $ \ah_{\delta}(\rS,\rr,\rw,\rz) $  holds with probability at least $1-\delta/2$ as claimed. 

Now to this end, let $r,w$, and $z$ be realizations of $\rr,\rw$, and $\rz$. By \cref{induktionsteplemmaeasy} with failure parameter $\delta/4$, we have with probability at least $1-\delta/4$ over $\rS\sim \cD_{c}^{m}$ that 
\begin{align*}
 \p_{\rx\sim \cD}\left[E_{\rS,r}\right]= \p_{\rx\sim \cD}\Bigg[   
  \sum_{\tilde{f}\in \cA(\cS(\rS,\emptyset),r)} \negmedspace \negmedspace \negmedspace  \negmedspace \negmedspace \negmedspace \frac{\ind\{\sum_{i=1}^{t} \ind\{h_{\tilde{f},i}(\rx)=c(\rx)\}/t\geq 3/4\}}{|\cS(\rS,\emptyset)|}
   <3/4 \Bigg]\leq \cs \frac{d+\ln{\left(4/\delta \right)}}{m}.
\end{align*}
Since the above holds for any realization $r,w$, and $z$ of $ \rr,\rw $, and $ \rz $, and that $\rS,\rr,\rw$, and $ \rz $ are independent, we obtain $
 \p_{\rS,\rr,\rw,\rz}[ \p_{\rx\sim \cD}[E_{\rS,\rr}]\leq \cs (d+\ln{\left(4/\delta \right)})/m] \geq 1-\delta/4$,
which concludes the claim regarding the error of $ \ah_{\delta}(\rS,\rr,\rw,\rz) $.

We now show that with probability at least $1-\delta/2$ over $\rS,\rr,\rw$, and $\rz$, we have that $\ah_{\delta}(\rS,\rr,\rw,\rz)$ has inference complexity $ O(\ln{[m /\delta(d+\ln{\left(1/\delta \right)}) ]}) \Uinf $ and computational complexity
\begin{align*}
 O\Big(\negmedspace\ln{\Big(\frac{m}{\delta\left(d+\ln{\left(1/\delta \right)}\right)} \Big)}\negmedspace\cdot\negmedspace\ln{\Big(\frac{m}{\delta} \Big)}\negmedspace\Big)\negmedspace\cdot \negmedspace\Big(O\Big(m+d\ln{\left(m \right)}\Big)\negmedspace+\negmedspace\Utrain(550d)\negmedspace+\negmedspace 3m\Uinf\negmedspace\Big).
 \end{align*} 

To this end, consider any realization of $S$, $ w $, and $ z $, of $\rS,\rw$, and $\rz$. For any $\tilde{f} \in \cA(\cS(S,\emptyset),\rr)$, we have a sub training sequence $S'\in \cS(S,\emptyset)$ such that $\tilde{f}=\cA(S',\rr)$, and $m'=\sum_{i=1}^{k} m(1/6)^{i} + 1 =|S'|\leq |S|=m$. Furthermore, since $n=|\rr|=6\left\lceil20^{2}\ln{\left(8m/\delta \right)}/2\right\rceil$, which is greater than $\lceil20^{2}\ln{(|S'|)}/2\rceil$, we have by \cref{lemmaAdaBoostSampleefficent}, that with probability at least $1-(\delta/(8|S'|))^{20}\geq 1-(6\delta/(8m))^{20}$ (since $S'\geq m/6$) over $\rr$, the number of operations needed to compute $\tilde{f}=\cA(S,\rr)$ is at most 

\begin{align*}
  n\negmedspace\cdot \negmedspace\Big(\negmedspace O\Big(\negmedspace |S'|\negmedspace +\negmedspace d\ln{\left(|S'| \right)}\negmedspace\Big)\negmedspace +\negmedspace\Utrain(550d)\negmedspace+\negmedspace |S'|\Uinf\negmedspace\Big)\negmedspace =\negmedspace O\Big(\negmedspace\ln{\Big(\frac{m}{\delta} \Big)}\negmedspace\Big)\negmedspace\cdot \negmedspace\Big(\negmedspace O\Big(m\negmedspace+\negmedspace d\ln{\left(m \right)}\negmedspace\Big)\negmedspace+\negmedspace\Utrain(550d)\negmedspace+\negmedspace 3m\Uinf\negmedspace\Big)\negmedspace.
\end{align*}
Since there are at most $5^{\log_{6}(m)}=m^{\log_{6}(5)}< m$ rows of $\cA(\cS(S,\emptyset),\rr)$, a union bound implies, that with probability at least $1-\left(6\delta/(8m)\right)^{19}$ over $\rr$,  any $h\in\cA((\cS(S,\emptyset),\rr))$ takes at most 

\begin{align*}
  O\Big(\ln{\Big(\frac{m}{\delta} \Big)}\Big)\cdot \Big(O\Big(m+d\ln{\left(m \right)}\Big)+\Utrain(550d)+3m\Uinf\Big),
  \end{align*}
operations to compute.
Thus, for any realization $r$ of $\rr$ such that the above holds, we notice by the for loop in Line~\ref{alg:Randommajorityforloopstart} consisting of $\lceil 16\cdot200 \ln{(m/(\delta(d+\ln{(1/\delta )})))} /9\rceil$ rounds, and in each round, the algorithm in Line~\ref{alg:RandommajoritysamplingS} finds $S_i\in \cS(S,\emptyset)$ and  
in Line~\ref{alg:Randommajorityrunsadaboostsample} runs $\cA(S_{i},r)$, which for this realization takes at most the above number of 
operations, since $\cA(S_{i},r)\in \cA(\cS(S,\emptyset),r)$, we get that Line~\ref{alg:Randommajorityrunsadaboostsample} over the for loop takes at most 
\begin{align}\label{eq:finalruntime}
  O\Big(\negmedspace\ln{\Big(\frac{m}{\delta\small(d+\ln{\left(1/\delta \right)}\small)} \Big)}\negmedspace\cdot\negmedspace  \ln{\Big(\frac{m}{\delta} \Big)} \negmedspace\Big)\negmedspace\cdot \negmedspace\Big(O\Big(m+d\ln{\left(m \right)}\Big)\negmedspace+\negmedspace\Utrain(550d)\negmedspace+\negmedspace 3m\Uinf\negmedspace\Big)
\end{align}
operations for such $r$, and since the probability of seeing such a realization $r$ of $ \rr $ is at least $1-\left(6\delta/(8m)\right)^{19}\geq 1-\delta/2$, we conclude that Line~\ref{alg:Randommajorityrunsadaboostsample} over the for loop takes at most the above many operations. This is the only part where $\rr$ affects the number of operations $\ah$ uses in total.
 
Now, setting parameters, Line~\ref{alg:Randommajoritysetm} to Line~\ref{alg:Randommajoritysett}, takes: $O(m)$ operations to read the length of $S$.  
Calculating $ k $ can be done in $ O(1) $ operations since we know $ m $ and $ k=\log_{6}(m) $. We assume that $\delta$ and $d$ are given as parameters so takes $ 2 $ operations to read,  and $m$ we have already calculated so calculating the number $m/(\delta(d+\ln{\left(1/\delta \right)}))$ takes at most $5$ additions/multiplications/divisions/operations, and taking $\ln{\left(\cdot \right)}$ of the number cost one operation thus calculating $ l  $ takes $ O(1) $ operations. Initializing $f$, an array of size $l$, takes $O(\ln{\left(m/(\delta(d+\ln{\left(1/\delta \right)})) \right)})$ operations. Calculating $ m' $ can be done by dividing $ m $ by $ 6 $ $ k $ times and adding the numbers as they are calculated and then adding $ 1  $ in the end, thus calculating $ m' $ takes at most $ O(k)=O(\ln{(m )}) $ operations. Since we know $ m' $ calculating $ t $ takes $ O(1) $. Thus, Line~\ref{alg:Randommajoritysetm} to Line~\ref{alg:Randommajoritysett} takes at most $O\left(m+\ln{\left(m/(\delta(d+\ln{\left(1/\delta \right)})) \right)}  \right)$ many operations, which is less than the number of operations used in Line~\ref{alg:Randommajorityrunsadaboostsample} over the for loop, as stated in \cref{eq:finalruntime}. 

Now, for the for loop in Line~\ref{alg:Randommajorityforloopstart} to Line~\ref{alg:Randommajorityoutputofhat}, we perform the following steps: Read $\rw_i$ which takes $\log_{6}(m)$ operations. Finding $S_i$ using $g',\rS$, and $\rw_i$ takes at most $O(m)$ operations by \cref{samplingfromrows}. The runtime in Line~\ref{alg:Randommajorityrunsadaboostsample} has been argued for in the above. Reading $\rz_i$ takes $t=O(\ln{(m)})$ operations. Extracting $h_{f_{\rw_{i}},\rz_{i}}$ out of $f_{\rw_{i}}$ is reading an entry in $f_{\rw_{i}}$ so takes $O(1)$ and reading $f_{\rw_{i}}$ takes at most $O(t)=O(\ln{\left(m\right)})$ operations. Adding/reading $h_{f_{\rw_{i}},\rz_{i}}$ into the $i$-th entry of $f$ takes $O(1)$ operations. Thus, all the operations in each iteration of the for loop, except Line~\ref{alg:Randommajorityrunsadaboostsample}, take at most $O(m)$ operations. Therefore, the total number of operations over the $l$ rounds is at most $O(lm)=O(\ln{\left(m/(\delta(d+\ln{\left(1/\delta \right)})m) \right)}\cdot m)$
operations, excluding the operations in Line~\ref{alg:Randommajorityrunsadaboostsample}, which are fewer than those made in Line~\ref{alg:Randommajorityrunsadaboostsample}, as argued above. 

Thus, the overall number of operations of $\ah$ can be bounded by the operations made in Line~\ref{alg:Randommajorityrunsadaboostsample} which was at most \cref{eq:finalruntime} with probability at least $1-\delta/2$ over $\rr$ for any realization $S,w$ and $ z $  of $\rS$, $ \rw $ and $ \rz $. Thus, by independence of $\rr,\rS,\rw$ and $\rz$ we conclude that the operations needed to calculate $ \ah_{\delta}(\rS,\rr,\rw,\rz) $ with probability at least $ 1-\delta/2$ over $ \rS,\rr,\rw,\rz $  is at most

\begin{align*}
  O\Big(\negmedspace\ln{\Big(\frac{m}{\delta\small(d+\ln{\left(1/\delta \right)}\small)} \Big)}\negmedspace\cdot\negmedspace  \ln{\Big(\frac{m}{\delta} \Big)} \negmedspace\Big)\negmedspace\cdot \negmedspace\Big(O\Big(m+d\ln{\left(m \right)}\Big)\negmedspace+\negmedspace\Utrain(550d)\negmedspace+\negmedspace 3m\Uinf\negmedspace\Big),
\end{align*}
which is the stated training complexity of $\ah_{\delta}(\rS,\rr,\rw,\rz)$ in \cref{maintheorem}. Furthermore, the inference complexity is $ l\cdot\Uinf $, which is $ O(\ln{(m/(\delta(d+\ln{(1/\delta )})) )})\Uinf $, as for a new example, all $ l $ voters in $ \ah_{\delta}(\rS,\rr,\rw,\rz) $ has to be queried to get $ \ah_{\delta}(\rS,\rr,\rw,\rz)(x) $, where each query takes $ \Uinf $ operations.     
\end{proof}

\section{Optimality of $\cA(\cS(\rS,\emptyset))$}\label{sec:optimalitya}
In this section, we prove \cref{induktionsteplemmaeasy} as a corollary of the main theorem of this section, \cref{induktionsteplemma}. For this, we need the following lemma, which gives a uniform error bound on majority voters. The proof of this result can be found in \cref{appendixoptimallityofdeterministicprocess}

\begin{restatable}{lemma}{lemone}\label{lem1}
  There exists a universal constant $C>1$ such that for: hypothesis class $\cH$ of VC-dimension $d$, number of voters $ t \in \mathbb{N} $, training sequence size $m\geq d$, distribution $\mathcal{D}$ over $\cX \times \{-1,1\}$, margin $0< \gamma < 1$, and $\xi> 1$, with probability at least $1-\delta$ over $\rS\sim \cD^{m}$, we have for all $f \in \dlh$.
  \begin{align*}
\Pr_{(\rx,\ry)\sim \mathcal{D}}[\ry f(\rx) \leq \gamma] \leq \Pr_{(\rx,\ry)\sim\rS}[\ry f(\rx) \leq \xi \gamma] +  C\sqrt{\frac{2d}{((\xi-1))\gamma)^2 m}} + \sqrt{\frac{2 \ln(2/\delta)}{m}}
\end{align*}

\end{restatable}

We also need the following lemma, which relates the error of $ \cA(\cS(S,T),r) $ to its recursive calls $ \cA(j,T,r) $. The proof of \cref{induktionlemma} can be found in \cref{appendixoptimallityofdeterministicprocess}.

\begin{restatable}{lemma}{induktionlemma}\label{induktionlemma}
  For any sequences $S,T\in (\cX\times \cY)^{*}$ such that $|S|= 6^{k}$ for some $k\geq 1$ and any string $r\in ([0,1]^{*})^{*}$, it holds that 
  \begin{align*}
    &\p_{\rx\sim \cD} \bigg[  
      \sum_{f\in \cA(\cS(S,T),r)} \frac{\ind\{\sum_{i=1}^{t} \ind\{h_{f,i}(\rx)=c(\rx)\}/t\geq 3/4\}}{|\cS(S,T)|}
       <3/4\bigg]
    \\
    \leq \negmedspace \negmedspace  \negmedspace   \negmedspace \negmedspace \max_{\stackrel{j\in\{ 1:5 \}}{\rf\in \cA(\cS(S,T),r) \backslash \cA(j,T,r)}} \negmedspace \negmedspace  \negmedspace \negmedspace \negmedspace \negmedspace \negmedspace  \negmedspace \negmedspace \negmedspace\negmedspace\negmedspace 80&\p_{\rx\sim \cD}
    \bigg[ \sum_{i=1}^{t}  \ind\{h_{\rf,i}(\rx)=c(\rx)\}/t < \frac{3}{4},
             \sum_{\text{\quad\quad}f\in \cA(j,T,r)}           \frac{\ind\{\sum_{i=1}^{t}  \ind\{h_{f,i}(\rx)=c(\rx)\}/t\geq \frac{3}{4}\}}{|\cS(j,T)|}
  <\frac{3}{4}\bigg]
  \end{align*}

\end{restatable}
The last lemma we need in the proof of \cref{induktionsteplemma} is \cref{lemma:adaboostsampleeasyproperties}, which gives the necessary properties of $ \cA $  for proving \cref{induktionsteplemma}. We restate it here for convenience.

\lemmaadaboostsampleeasyproperties*

With the above lemmas in place, we are now ready to give the proof of \cref{induktionsteplemma}.
\begin{theorem}\label{induktionsteplemma}
 There exists a universal constant $\cs$ such that for: Distribution $ \cD $ over $ \cX $, hypothesis class $ \cH $ of VC-dimension $ d $, target concept $ c\in \cH $,  failure probability $ 0<\delta<1 $, training sequence size $m=6^{k}$ for some $k\geq 1$, training sequence $T\in (\cX\times \cY)^{*}$ realizable by $ c $, and string $r\in ([0,1]^{*})^{*}$, it holds with probability at least $1-\delta$ over $\rS\sim \cD_{c}^{m}$ that 
  \begin{align*}
    \p_{\rx\sim \cD} \Bigg[  
      \sum_{\rf\in \cA(\cS(\rS,T),r)} \frac{\ind\{\sum_{i=1}^{t} \ind\{h_{\rf,i}(\rx)=c(\rx)\}/t\geq 3/4\}}{|\cS(\rS,T)|}
       <3/4\Bigg] \leq \cs \frac{d+\ln{\left(1/\delta \right)}}{m}.
  \end{align*}
\end{theorem}

\begin{proof}[Proof of \cref{induktionsteplemma}]
We show that for any $T\in(\cX\times \cY)^{*}$, $0<\delta<1$, $\rS\sim \cD^{m}$ where $m=6^{k}$ for integer $k\geq1$, and any string $r\in([0:1]^{*})^{*}$, with probability at least $1-\delta$ over $\rS$, we have: 
\begin{align}\label{highprobabilityeq3}
  \p_{\rx\sim \cD} \left[  
    \sum_{f\in \cA(\cS(S,T),r)} \frac{\ind\{\sum_{i=1}^{t} \ind\{h_{f,i}(\rx)=c(\rx)\}/t\geq 3/4\}}{|\cS(S,T)|}
     <3/4\right]
  \leq \cs(d+\ln{\left(1/\delta \right)})/m,
\end{align}
for $\cs= \max\{32C^{2}\cdot 960,3840\ln{\left(160 \right)}\}(2\cdot 80\cdot 6 \left(\ln{\left(40 \right)}+1\right))^2$ (where $C\geq 1$ is the universal constant from \cref{lem1}). We will prove this claim by induction. For any $k\leq \log_{6}(\cs)$, we are done as the above left-hand side is at most $1$, and the right side in this case is at least $1$, which concludes the induction base.

Now, for the induction step, let $T\in(\cX\times\cY)^{*}$, $0<\delta<1$, $m=6^{k}$, and $r\in([0:1]^{*})^{*}$. We first use \cref{induktionlemma}, followed by applying the union bound twice, respectively on $j\in\{1,2,3,4,5\}$ and $\cA(\cS(S,T),r) \backslash \cA(j,T,r)=\sqcup_{i=1,i\not=j}^{5}\cA(i,T)$, to obtain that for any $y>0$ we have:  
\begin{align}\label{highprobabilityeq7}
  &\p_{\rS}\left[   
    \p_{\rx\sim \cD} \left[  
      \sum_{f\in \cA(\cS(\rS,T))} \frac{\ind\{\sum_{i=1}^{t} \ind\{h_{f,i}(\rx)=c(\rx)\}/t\geq 3/4\}}{|\cS(S,T)|}
       <3/4\right]> y 
       \right]
    \\
    &\leq 
    \p_{S}\bigg[\negmedspace \negmedspace  \negmedspace   \negmedspace \negmedspace \negmedspace \negmedspace  \negmedspace   \negmedspace \negmedspace \negmedspace \negmedspace  \negmedspace     \max_{\stackrel{j\in\{ 1:5 \}}{\rf\in \cA(\cS(S,T),r) \backslash \cA(j,T,r)}} \negmedspace \negmedspace  \negmedspace \negmedspace \negmedspace \negmedspace \negmedspace  \negmedspace \negmedspace \negmedspace\negmedspace \negmedspace\negmedspace\negmedspace 80\p_{\rx\sim \cD}
    \bigg[\sum_{i=1}^{t}\negmedspace  \ind\{h_{\rf,i}(\rx)=c(\rx)\}/t\negmedspace <\negmedspace \frac{3}{4},\negmedspace
    \negmedspace \negmedspace  \negmedspace  \negmedspace \negmedspace  \negmedspace \negmedspace \sum_{\text{\quad\quad}f\in \cA(j,T,r)} \negmedspace \negmedspace \negmedspace \negmedspace\negmedspace \negmedspace  \negmedspace \negmedspace \negmedspace \negmedspace \frac{\ind\{\sum_{i=1}^{t}\negmedspace  \ind\{h_{f,i}(\rx)=c(\rx)\}/t\geq \frac{3}{4}\}}{|\cS(j,T)|}
    \negmedspace<\negmedspace\frac{3}{4}\bigg]> y
  \bigg]\nonumber
  \\
  &\leq\negmedspace
   \sum_{\stackrel{j,l=1}{l\not=j}}^{5}\p_{\rS_{0},\rS_{j},\rS_{l}}\bigg[ \negmedspace \negmedspace \negmedspace\negmedspace\negmedspace\negmedspace\negmedspace\max_{ \quad \quad\stackrel{}{\rf\in\cA(l,T,r)}} \negmedspace \negmedspace \negmedspace  \negmedspace\negmedspace\negmedspace\negmedspace  \negmedspace  80\p_{\rx\sim \cD}
   \bigg[ \sum_{i=1}^{t}\negmedspace  \ind\{h_{\rf,i}(\rx)=c(\rx)\}/t\negmedspace <\negmedspace \frac{3}{4},\negmedspace
   \negmedspace \negmedspace  \negmedspace  \negmedspace \negmedspace  \negmedspace \negmedspace \sum_{\text{\quad\quad}f\in \cA(j,T,r)} \negmedspace \negmedspace \negmedspace \negmedspace\negmedspace \negmedspace  \negmedspace \negmedspace \negmedspace \negmedspace \frac{\ind\{\sum_{i=1}^{t}\negmedspace  \ind\{h_{f,i}(\rx)=c(\rx)\}/t\geq \frac{3}{4}\}}{|\cS(j,T)|}
   \negmedspace<\negmedspace\frac{3}{4}\bigg]\negmedspace> \negmedspace y\bigg]\nonumber
  \end{align} 

We will now show that for $y=\cs(d+\ln{\left(1/\delta \right)})/m$, the last expression in the above is at most $\delta$, which would conclude the induction step and the proof. We will show this by bounding each term in the sum by $ \delta/20 $. Since the sum contains at most $ 20 $ terms, the claim follows. Thus, from now on, let $y=\cs(d+\ln{\left(1/\delta \right)})/m$, and $ j,l\in\{ 1:5 \} $ $ l\not=j $.

Let $S_{j}$ be any realization of $\rS_{j}$. We observe that since $|\rS_{0}|=6^{k-1}=m/6,$ and $\cA(j,T,r)=\{\cA(S,r)\}_{S\in\cS(j,T)}$, where we recall that $\cS(j,T)=\cS(\rS_{0},S_{j}\sqcup T),$ the induction hypothesis holds for this call with random training sequence $\rS_0$ (since $ |\rS_{0}|=6^{k-1}=m/6 $ ) in the first argument of $\cS(\cdot,\cdot)$,  training sequence $S_{j}\sqcup T$ in the second argument, and string $r$. We use the induction hypothesis with failure parameter $\delta/\cnn$, i.e., with probability at most $\delta/\cnn$ we have that  $\p_{\rx\sim \cD}[\sum_{f\in \cA(j,T,r)} \ind\{\sum_{i=1}^{t} \ind\{h_{f,i}(\rx)=c(\rx)\}/t\geq 3/4\}/|\cS(j,T)|
<3/4]> 6\cs(d+\ln{(\cnn/\delta )})/m$. We now consider the following disjoint events over $ \rS_{0} $ 
\begin{align*}
 &E_1
  = 
 \Bigg\{  \p_{\rx\sim \cD} \left[\sum_{f\in \cA(j,T,r)} \frac{\ind\{\sum_{i=1}^{t} \ind\{h_{f,i}(\rx)=c(\rx)\}/t\geq 3/4\}}{|\cS(j,T)|}
  < \frac{3}{4} \right]< y/80\Bigg\} 
\\
&E_{2}
 = 
\Bigg\{  y/80
\leq
\p_{\rx\sim \cD}  \left[ \sum_{f\in \cA(j,T,r)}      \frac{\ind\{\sum_{i=1}^{t} \ind\{h_{f,i}(\rx)=c(\rx)\}/t\geq 3/4\}}{|\cS(j,T)|}
 < \frac{3}{4} \right]  \leq  6\cs(d+\ln{\left(\cnn/\delta \right)})/m \Bigg\} 
\\
&E_{3} = \Bigg\{6\cs(d+\ln{\left(\cnn/\delta \right)})/m< \p_{\rx\sim \cD}  \left[\sum_{f\in \cA(j,T,r)} \frac{\ind\{\sum_{i=1}^{t} \ind\{h_{f,i}(\rx)=c(\rx)\}/t\geq 3/4\}}{|\cS(j,T)|}
 < \frac{3}{4} \right] \Bigg\} .
\end{align*} 
Now, using the law of total probability, that $E_{3}$ happens with probability at most $\delta/\cnn$, and that $ \rS_{0} $ and $ \rS_{l} $ are independent, we get:

\begin{align}\label{highprobabilityeq5}
  &\p_{\rS_{0},\rS_{l}}\bigg[  \max_{\stackrel{}{\rf\in\cA(l,T,r)}} \negmedspace 80\p_{\rx\sim \cD}
  \bigg[ \sum_{i=1}^{t}\negmedspace  \ind\{h_{\rf,i}(\rx)=c(\rx)\}/t\negmedspace <\negmedspace \frac{3}{4},
  \negmedspace \negmedspace  \negmedspace  \negmedspace \negmedspace  \negmedspace \negmedspace \negmedspace \negmedspace\sum_{\text{\quad\quad}f\in \cA(j,T,r)} \negmedspace \negmedspace \negmedspace \negmedspace\negmedspace \negmedspace  \negmedspace \negmedspace \negmedspace \negmedspace \frac{\ind\{\sum_{i=1}^{t}\negmedspace  \ind\{h_{f,i}(\rx)=c(\rx)\}/t\geq \frac{3}{4}\}}{|\cS(j,T)|}
  \negmedspace<\negmedspace\frac{3}{4}\bigg]\negmedspace> \negmedspace y\bigg]\nonumber
\\
&\leq
\e_{\rS_{0}}\bigg[ \p_{\rS_{l}}\bigg[    \max_{\stackrel{}{\rf\in\cA(l,T,r)}} \negmedspace    80\p_{\rx\sim \cD}
\bigg[
  \negmedspace \sum_{i=1}^{t}\negmedspace  \ind\{h_{\rf,i}(\rx)=c(\rx)\}/t\negmedspace <\negmedspace \frac{3}{4}\\&\quad\quad\quad\quad\quad\quad\quad\quad\quad\quad\quad\quad,
\negmedspace \negmedspace  \negmedspace  \negmedspace \negmedspace  \negmedspace \negmedspace \negmedspace \negmedspace\sum_{\text{\quad\quad}f\in \cA(j,T,r)} \negmedspace \negmedspace \negmedspace \negmedspace\negmedspace \negmedspace  \negmedspace \negmedspace \negmedspace \negmedspace \frac{\ind\{\sum_{i=1}^{t}\negmedspace  \ind\{h_{f,i}(\rx)=c(\rx)\}/t\geq \frac{3}{4}\}}{|\cS(j,T)|}
\negmedspace<\negmedspace\frac{3}{4}\bigg]\negmedspace>\negmedspace y\bigg] \bigg |  E_1 \bigg]\p_{S_{0}}[E_1]\nonumber
\\
&+
\e_{\rS_{0}}\bigg[ \p_{\rS_{l}}\bigg[    \max_{\stackrel{}{\rf\in\cA(l,T,r)}}\negmedspace    80\p_{\rx\sim \cD}
\bigg[
  \negmedspace \sum_{i=1}^{t}\negmedspace  \ind\{h_{\rf,i}(\rx)=c(\rx)\}/t\negmedspace <\negmedspace \frac{3}{4}\nonumber
  \\
  &\quad \quad\quad\quad\quad\quad\quad\quad\quad\quad\quad\quad,
\negmedspace \negmedspace  \negmedspace  \negmedspace \negmedspace  \negmedspace \negmedspace \negmedspace \negmedspace\sum_{\text{\quad\quad}f\in \cA(j,T,r)} \negmedspace \negmedspace \negmedspace \negmedspace\negmedspace \negmedspace  \negmedspace \negmedspace \negmedspace \negmedspace \frac{\ind\{\sum_{i=1}^{t}\negmedspace  \ind\{h_{f,i}(\rx)=c(\rx)\}/t\geq \frac{3}{4}\}}{|\cS(j,T)|}
\negmedspace<\negmedspace\frac{3}{4}\bigg]\negmedspace>\negmedspace y\bigg] \bigg |  E_2 \bigg]\p_{S_{0}}[E_2]+
\delta/\cnn\nonumber
\end{align}  

Now, if $S_0$ is a realization of $\rS_{0}$ where $ \p_{\rx\sim \cD}[\sum_{f\in \cA(j,T,r)} \ind\{\sum_{i=1}^{t} \ind\{h_{f,i}(\rx)=c(\rx)\}/t\geq 3/4\}/|\cS(j,T)|<\frac{3}{4}
 ]< \cs(d+\ln{\left(1/\delta \right)})/(80m)=y/80$ it follows by monotonicity of measures, and since we had $y=\cs(d+\ln{\left(1/\delta \right)})/m$, that
\begin{align*}
  \p_{\rS_{l}}\bigg[   \max_{\stackrel{}{\rf\in\cA(l,T,r)}} \negmedspace 80\p_{\rx\sim \cD}
  \bigg[ \sum_{i=1}^{t}\negmedspace  \ind\{h_{\rf,i}(\rx)=c(\rx)\}/t\negmedspace <\negmedspace \frac{3}{4},
  \negmedspace \negmedspace  \negmedspace  \negmedspace \negmedspace  \negmedspace \negmedspace \negmedspace \negmedspace\sum_{\text{\quad\quad}f\in \cA(j,T,r)} \negmedspace \negmedspace \negmedspace \negmedspace\negmedspace \negmedspace  \negmedspace \negmedspace \negmedspace \negmedspace \frac{\ind\{\sum_{i=1}^{t}\negmedspace  \ind\{h_{f,i}(\rx)=c(\rx)\}/t\geq \frac{3}{4}\}}{|\cS(j,T)|}
  \negmedspace<\negmedspace\frac{3}{4}\bigg]\negmedspace>\negmedspace y\bigg]\negmedspace=\negmedspace0.
\end{align*} 

Since this holds on $E_1$, we get that the first term in \cref{highprobabilityeq5} is zero.  

Now, for realizations $S_0$ of $\rS_0$, which are in $ E_{2}$, i.e., is such that $y/80=\cs(d+\ln{\left(1/\delta \right)})/(80m)\leq \p_{\rx\sim \cD}[\sum_{f\in \cA(j,T,r)} \ind\{\sum_{i=1}^{t} \ind\{h_{f,i}(\rx)=c(\rx)\}/t\geq 3/4\}/|\cS(j,T)|]\leq 6\cs(d+\ln{\left(\cnn/\delta \right)})/m$ (the middle term of \cref{highprobabilityeq5}) using the law of conditional probability (which is well defined as the above probability is now non zero), we get that

  \begin{align}\label{highprobabilityeq6}
    &\max_{\stackrel{}{\rf\in\cA(l,T,r)}} \negmedspace 80\p_{\rx\sim \cD}
  \bigg[ \sum_{i=1}^{t}\negmedspace  \ind\{h_{\rf,i}(\rx)=c(\rx)\}/t\negmedspace <\negmedspace \frac{3}{4},
  \negmedspace \negmedspace  \negmedspace  \negmedspace \negmedspace  \negmedspace \negmedspace \negmedspace \negmedspace\sum_{\text{\quad\quad}f\in \cA(j,T,r)} \negmedspace \negmedspace \negmedspace \negmedspace\negmedspace \negmedspace  \negmedspace \negmedspace \negmedspace \negmedspace \frac{\ind\{\sum_{i=1}^{t}\negmedspace  \ind\{h_{f,i}(\rx)=c(\rx)\}/t\geq \frac{3}{4}\}}{|\cS(j,T)|}
  \negmedspace<\negmedspace\frac{3}{4}\bigg]
      \\  
      &=\negmedspace\negmedspace\negmedspace\negmedspace\max_{\stackrel{}{\rf\in\cA(l,T,r)}} \negmedspace  80\p_{\rx\sim \cD}
      \bigg[
        \negmedspace \sum_{i=1}^{t}\negmedspace  \ind\{h_{\rf,i}(\rx)=c(\rx)\}/t\negmedspace <\negmedspace \frac{3}{4}\bigg |
       \negmedspace  \negmedspace \negmedspace  \negmedspace \negmedspace \negmedspace \negmedspace\sum_{\text{\quad\quad}f\in \cA(j,T,r)} \negmedspace \negmedspace \negmedspace \negmedspace\negmedspace \negmedspace  \negmedspace \negmedspace \negmedspace \negmedspace \frac{\ind\{\sum_{i=1}^{t}\negmedspace  \ind\{h_{f,i}(\rx)=c(\rx)\}/t\geq \frac{3}{4}\}}{|\cS(j,T)|}
       \negmedspace<\negmedspace\frac{3}{4}\bigg]
      \nonumber\\ 
      &\quad\quad\quad\quad\quad\quad\quad\quad\quad\quad\quad\qquad\qquad\qquad\cdot \p_{\rx\sim \cD}\bigg[ \negmedspace \negmedspace\negmedspace \negmedspace \negmedspace\negmedspace \negmedspace \negmedspace\sum_{\text{\quad\quad}f\in \cA(j,T,r)} \negmedspace \negmedspace \negmedspace \negmedspace\negmedspace \negmedspace  \negmedspace \negmedspace \negmedspace \negmedspace \frac{\ind\{\sum_{i=1}^{t}\negmedspace  \ind\{h_{f,i}(\rx)=c(\rx)\}/t\geq \frac{3}{4}\}}{|\cS(j,T)|}\negmedspace<\negmedspace\frac{3}{4}\bigg]
      \nonumber\\
      &\leq \negmedspace\negmedspace\negmedspace\negmedspace
      \max_{\stackrel{}{\rf\in\cA(l,T,r)}} \negmedspace  80\p_{\rx\sim \cD}
      \bigg[
        \negmedspace \sum_{i=1}^{t}\negmedspace  \ind\{h_{\rf,i}(\rx)=c(\rx)\}/t\negmedspace <\negmedspace \frac{3}{4}\bigg |
       \negmedspace  \negmedspace \negmedspace  \negmedspace \negmedspace \negmedspace \negmedspace\sum_{\text{\quad\quad}f\in \cA(j,T,r)} \negmedspace \negmedspace \negmedspace \negmedspace\negmedspace \negmedspace  \negmedspace \negmedspace \negmedspace \negmedspace \frac{\ind\{\sum_{i=1}^{t}\negmedspace  \ind\{h_{f,i}(\rx)=c(\rx)\}/t\geq \frac{3}{4}\}}{|\cS(j,T)|}
       \negmedspace<\negmedspace\frac{3}{4}\bigg]\cdot\frac{6\cs(d+\ln{(\frac{\cnn}{\delta} )})}{m}.\nonumber
    \end{align}
Thus, if we can show that with probability at least $1-\delta/\cff$ over $\rS_{l}$ that 
\begin{align*}
  \max_{\stackrel{}{\rf\in\cA(l,T,r)}} \negmedspace \p_{\rx\sim \cD}
  \bigg[
  \negmedspace \sum_{i=1}^{t}\negmedspace  \ind\{h_{\rf,i}(\rx)=c(\rx)\}/t\negmedspace <\negmedspace \frac{3}{4}\bigg |
   \negmedspace  \negmedspace \negmedspace  \negmedspace \negmedspace \negmedspace \negmedspace\sum_{\text{\quad\quad}f\in \cA(j,T,r)} \negmedspace \negmedspace \negmedspace \negmedspace\negmedspace \negmedspace  \negmedspace \negmedspace \negmedspace \negmedspace \frac{\ind\{\sum_{i=1}^{t}\negmedspace  \ind\{h_{f,i}(\rx)=c(\rx)\}/t\geq \frac{3}{4}\}}{|\cS(j,T)|}
   \negmedspace<\negmedspace\frac{3}{4}\bigg] 
\end{align*}
is at most $(80\cdot 6 \left(\ln{\left(\cnn \right)}+1\right))^{-1}$ for realizations $ S_{0} $ of  $\rS_{0}\in E_{2}$, we get by \cref{highprobabilityeq6} that 
\begin{align*}
  \max_{\stackrel{}{\rf\in\cA(l,T,r)}} \negmedspace   80\p_{\rx\sim \cD}
  \bigg[ \sum_{i=1}^{t}\negmedspace  \ind\{h_{\rf,i}(\rx)=c(\rx)\}/t\negmedspace <\negmedspace \frac{3}{4},
  \negmedspace \negmedspace  \negmedspace  \negmedspace \negmedspace  \negmedspace \negmedspace \negmedspace \negmedspace\sum_{\text{\quad\quad}f\in \cA(j,T,r)} \negmedspace \negmedspace \negmedspace \negmedspace\negmedspace \negmedspace  \negmedspace \negmedspace \negmedspace \negmedspace \frac{\ind\{\sum_{i=1}^{t}\negmedspace  \ind\{h_{f,i}(\rx)=c(\rx)\}/t\geq \frac{3}{4}\}}{|\cS(j,T)|}\negmedspace <\negmedspace \frac{3}{4}=y
  \bigg]
  \\ 
  \leq \cs(d+\ln{\left(1/\delta \right)})/m,
\end{align*}
implying the term in \cref{highprobabilityeq5} conditioned on $ E_{2} $  is at most $\delta/\cff$. Consequently, we have shown that \cref{highprobabilityeq5} happens with probability at most $\delta/\cff+\delta/\cnn$ over $\rS_0$ and $ \rS_{l}$ for any realization $S_{j}$ of $\rS_{j}$. Therefore, by taking expectation with respect to $\rS_{j}$ in \cref{highprobabilityeq5} and $ \rS_{0},\rS_{j},\rS_{l} $ being independent,  we get the same upper bound, which upper bounds each term in \cref{highprobabilityeq7} with probability $ \delta/\cff+\delta/\cnn $. This implies that \cref{highprobabilityeq7} is at most $20(\delta/\cff+\delta/\cnn)=\delta$, which concludes the induction step and the proof. 

Thus, we now show that with probability at least $1-\delta/\cff$ over $ \rS_{l} $ 
\begin{align}\label{highprobabilityeq11}
  \max_{\stackrel{}{\rf\in\cA(l,T,r)}}  \negmedspace    &\p_{\rx\sim \cD}
  \bigg[\sum_{i=1}^{t}\negmedspace  \ind\{h_{\rf,i}(\rx)=c(\rx)\}/t\negmedspace <\negmedspace \frac{3}{4}\bigg |
   \negmedspace  \negmedspace \negmedspace  \negmedspace \negmedspace \negmedspace \negmedspace\sum_{\text{\quad\quad}f\in \cA(j,T,r)} \negmedspace \negmedspace \negmedspace \negmedspace\negmedspace \negmedspace  \negmedspace \negmedspace \negmedspace \negmedspace \frac{\ind\{\sum_{i=1}^{t}\negmedspace  \ind\{h_{f,i}(\rx)=c(\rx)\}/t\geq \frac{3}{4}\}}{|\cS(j,T)|}<\frac{3}{4}
  \bigg] \leq \frac{1}{80\cdot6(\ln{(\cnn )}+1)}
\end{align}
for realizations $ S_{0} $ of $\rS_{0}\in E_{2}$.

To this end, let $ S_{0} $ be a realization of $ \rS_{0} \in E_{2}$. We consider the set $$A=\left\{x\in\cX:\sum_{f\in \cA(j,T,r)} \ind\{\sum_{i=1}^{t} \ind\{h_{f,i}(\rx)=c(\rx)\}/t\geq 3/4\}/|\cS(j,T)|<3/4\right\}.$$ Now, let $\cD(\cdot | A)$ be the conditional probability of $\cD$ restricted to $A$. We can then rewrite the first expression of \cref{highprobabilityeq11} as:
\begin{align}\label{highprobabilityeq13}
  \negmedspace\negmedspace\negmedspace\negmedspace
  &\max_{\stackrel{}{\rf\in\cA(l,T,r)}} 80\p_{\rx\sim \cD}
  \bigg[ \sum_{i=1}^{t}\negmedspace  \ind\{h_{\rf,i}(\rx)=c(\rx)\}/t\negmedspace <\negmedspace \frac{3}{4}\bigg |
   \negmedspace  \negmedspace \negmedspace  \negmedspace \negmedspace \negmedspace \negmedspace\sum_{\text{\quad\quad}f\in \cA(j,T,r)} \negmedspace \negmedspace \negmedspace \negmedspace\negmedspace \negmedspace  \negmedspace \negmedspace \negmedspace \negmedspace \frac{\ind\{\sum_{i=1}^{t}\negmedspace  \ind\{h_{f,i}(\rx)=c(\rx)\}/t\geq \frac{3}{4}\}}{|\cS(j,T)|}
   \negmedspace<\negmedspace\frac{3}{4}\bigg]\nonumber
  \\
  =
  &\max_{\rf\in \cA(l,T,r)} 80\p_{\rx\sim \cD(\cdot| A)}
  \left[
  \sum_{i=1}^{t} \ind\{h_{\rf,i}(\rx)=c(\rx)\}/t< 3/4
 \right]. 
\end{align}
We now consider the training sequence $\rS_{l}\sqcap A=[\rS_{l,i}| \rS_{l,i,1}\in A]$, i.e., training examples in $(x,y)\in\rS_{l}$ where $x\in A$. We first notice that $|\rS_{l}\sqcap A|=\sum_{i=1}^{m/6} \ind\{{\rS_{l,i,1}}\in A\}$ has an expected size of at least $(m/6)\cdot \cs(d+\ln{\left(1/\delta \right)})/(80m)=\cs(d+\ln{\left(1/\delta \right)})/480$, as we have assumed a  $ S_0 $  of $\rS_0\in E_{2}$ such that $\cs(d+\ln{\left(1/\delta \right)})/(80m)\leq \p_{\rx\sim \cD}[\sum_{f\in \cA(j,T,r)} \ind\{\sum_{i=1}^{t} \ind\{h_{f,i}(\rx)=c(\rx)\}/t\geq 3/4\}/|\cS(j,T)|<3/4]$, and we have that $ \rS_{l,i,1}\sim \cD $. Furthermore, since $ |\rS_{l}\sqcap A| $  is a sum of i.i.d. Bernoulli random variables, we have by Chernoffs inequality, and for  $\cs\geq\ln{\left(2\cdot\cff \right)}8\cdot480=\ln{(2\cdot\cff )}3840$, (which holds since $\cs= \max\{32C^{2}\cdot 960,3840\ln{\left(160 \right)}\}(2\cdot 80\cdot 6 \left(\ln{\left(40 \right)}+1\right))^2$) and $d+\ln{\left(1/\delta \right)}\geq1+ \ln{\left(1/\delta \right)}$,  that
\begin{align*}
  \p_{\rS_{l}}\left[|\rS_{l}\cap T|> \cs(d+\ln{\left(1/\delta \right)})/960\right]\leq 1-\exp\left(-\left(\cs(d+\ln{\left(1/\delta \right)})/480\right)/8\right)\leq1- \delta/(2\cdot\cff).
\end{align*}
Thus, by the law of total probability, and using the above, we have
\begin{align}\label{highprobabilityeq10}
  &\p_{\rS_{l}}\bigg[ \max_{\rf\in \cA(l,T,r)} 80\p_{\rx\sim \cD(\cdot| A)}
  \left[
  \sum_{i=1}^{t} \ind\{h_{\rf,i}(\rx)=c(\rx)\}/t< 3/4
 \right]\leq \frac{1}{80\cdot6(\ln{(\cnn )}+1)}\bigg]\nonumber\\ 
  &\leq 
  \p_{\rS_{l}}\bigg[ \max_{\rf\in \cA(l,T,r)} 80\p_{\rx\sim \cD(\cdot| A)}
  \left[
  \sum_{i=1}^{t} \ind\{h_{\rf,i}(\rx)=c(\rx)\}/t< 3/4
 \right]\leq \frac{1}{80\cdot6(\ln{(\cnn )}+1)}\nonumber\\&
 \quad\quad\quad\quad\quad\quad\quad\quad\quad\quad\quad\quad\quad\quad\quad\quad\quad\quad\bigg| |\rS_{l}\sqcap A|> \cs(d+\ln{\left(1/\delta \right)})/960 \Bigg]+ \delta/(2\cdot\cff)
\end{align}
We now use that $1/2=(\sum_{i=1}^{t} \ind\{h_{\rf,i}(\rx)=c(\rx)\}/t+\sum_{i=1}^{t} \ind\{h_{\rf,i}(\rx)\not=c(\rx)\}/t)/2$ and the definition of $ \cD_{c} $ as the distribution over $ (\cX\times \cY) $, where the examples have distribution given by $ \rx,c(\rx) $, with  $ \rx\sim \cD $ to obtain:
\begin{align}\label{highprobabilityeq12}
  &\max_{\rf\in \cA(l,T,r)} \p_{\rx\sim \cD(\cdot | A)}
  \left[
  \sum_{i=1}^{t} \ind\{h_{\rf,i}(\rx)=c(\rx)\}/t< 3/4 \right] 
   \\
  =  &\max_{\rf\in \cA(l,T,r)} \p_{\rx\sim \cD(\cdot | A)}
  \left[
  \left(\sum_{i=1}^{t} \ind\{h_{\rf,i}(\rx)=c(\rx)\}/t-\sum_{i=1}^{t} \ind\{h_{\rf,i}(\rx)\not=c(\rx)\}/t\right)/2< 3/4-1/2 \right] \nonumber
  \\
  =  &\max_{\rf\in \cA(l,T,r)} \p_{\rx\sim \cD(\cdot | A)}
  \left[
  \sum_{i=1}^{t} h_{\rf,i}(\rx)c(\rx)/(2t)< 1/4 \right]\nonumber 
  \\
  =  &\max_{\rf\in \cA(l,T,r)} \negmedspace\p_{\rx\sim \cD(\cdot | A)}
  \left[
  \sum_{i=1}^{t} h_{\rf,i}(\rx)c(\rx)/t< 1/2 \right] =\negmedspace\negmedspace\max_{\rf\in \cA(l,T,r)} \negmedspace\negmedspace\p_{(\rx,\ry)\sim \cD_{c}(\cdot | A)}
  \left[
  \sum_{i=1}^{t} h_{\rf,i}(\rx)\ry/t< 1/2 \right], \nonumber
\end{align}
where $ \cD_{c}(\cdot \mid A) $ is the measure such that $ \p_{\rx\sim \cD(\cdot\mid A)}\left((\rx,\c(\rx))\in B\right) $ for $ B\subset (\cX\times \{  -1,1\} ).$ 
Further, since $\cA\left(l,T,r\right)$ are the hypotheses that \cref{alg:AdaBoostSample} outputs on the training sequences in $\cS(S_{0},\rS_{l}\sqcup T)$ and random string $ r $,  and by \cref{lemma:adaboostsampleeasyproperties} we know that an outputted majority vote $\rf\in \cA\left(l,T\right)$, where $ \rf=\cA(\tilde{\rS},r) $ for $ \tilde{\rS} \in \cS(S_{0},\rS_{l}\sqcup T)$, is such that its in sample $ 3/4 $-margin loss   $\sum_{(x,y)\in \tilde{S}}  
\ind\{\sum_{i=1}^{t} h_{\rf,i}(\rx)c(\rx)/t\leq 3/4\} < 1/|\tilde{S}|$, is zero. Thus as $\rS_{l}\sqcap A$ is a sub training sequence of any $ \tilde{S}\in \cS(S_{0},\rS_{l}\sqcup T) $, we conclude that $ \rf\in  \cA\left(l,T\right)$  also has $ 3/4 $-margin loss equal to zero on $\rS_{l}\sqcap A$. 

We further notice that each example in $\rS_{l}\sqcap A$ has distribution $\cD_{c}(\cdot| A)$. Since we have just concluded that any $\rf\in \cA\left(l,T\right)$ has $ 3/4 $-margin loss equal to zero on $\rS_{l}\sqcap A$, and $\rf\in \dlh$, with $ t=\lceil20^{2}\ln{(|\tilde{S}|)}/2 \rceil $  by \cref{lemma:adaboostsampleeasyproperties} (note that $ |\tilde{S}| $ is the same for any $ \tilde{S} \in \cS(S_{0},\rS_{l}\cup T)$, so they are all in $ \dlh $ ),  we get by invoking \cref{lem1} with $\rS_{l}\sqcap A$, $ \dlh $, $\gamma=1/2$, $\xi=3/2$, and failure probability $\delta/(2\cdot \cff)$, combined with \cref{highprobabilityeq12}, that with probability at least $1-\delta/(2\cdot \cff)$ over $\rS_{l}\sqcap A$,conditioned on $ |\rS_{l}\sqcap A|>\cs(d+\ln{\left(1/\delta \right)})/960 $,  we have that
\begin{align}\label{highprobabilityeq9}
  &\max_{\rf\in \cA(l,T,r)} \p_{\rx\sim \cD(\cdot | A)}
  \left[
  \sum_{i=1}^{t} \ind\{h_{\rf,i}(\rx)=c(\rx)\}/t< 3/4 
  \right]\nonumber\\
  =
  &\max_{\rf\in \cA(l,T,r)} \p_{(\rx,\ry)\sim \cD_{c}(\cdot | A)}
  \left[
  \sum_{i=1}^{t} h_{\rf,i}(\rx)\ry/t< 1/2 \right]\nonumber 
  \\
  \leq 
  &\max_{\rf \in \cA(l,T,r)}\left\{ 
    \p_{\rx\sim \rS_{l}\sqcap A}
  \left[
  \sum_{i=1}^{t} h_{\rf,i}(\rx)c(\rx)/t\leq 3/4 \right] 
  +C\sqrt{\frac{32d}{|\rS_{l}\sqcap A|}}
  +\sqrt{\frac{2\ln{\left(4\cdot \cff/\delta \right)}}{|\rS_{l}\sqcap A|}}
  \right\}\nonumber
  \\
  \leq
  &C\sqrt{\frac{32d}{|\rS_{l}\sqcap A|}}
  +\sqrt{\frac{2\ln{\left(2\cdot \cff/\delta \right)}}{|\rS_{l}\sqcap A|}}.
\end{align}
For $|\rS_{l}\sqcap A|> \cs(d+\ln{\left(1/\delta \right)})/960$, we have
\begin{align*}
  \frac{32d}{|\rS_{l}\sqcap A|}\leq \frac{32\cdot960}{\cs}
 \quad 
 \text{ and }
\quad 
\frac{2\ln{\left(4\cdot \cff/\delta\right)}}{|\rS_{l}\sqcap A|}\leq \frac{3840\ln{\left(4\cdot \cff \right)}}{\cs}.
\end{align*}
Thus, since $\cs= \max\{32C^{2}\cdot 960,3840\ln{\left(160 \right)}\}(2\cdot 80\cdot 6 \left(\ln{\left(40 \right)}+1\right))^2$, that is, $\cs= \max\{32C^{2}\cdot 960,3840\ln{\left(4\cdot \cff \right)}\}(2\cdot 80\cdot 6 \left(\ln{\left(\cnn \right)}+1\right))^2$,  we get that 
\begin{align*}
  C\sqrt{\frac{32d}{|\rS_{l}\sqcap A|}}
  +\sqrt{\frac{2\ln{\left(4\cdot \cff/\delta \right)}}{|\rS_{l}\sqcap A|}}\leq
  (80\cdot 6 \left(\ln{\left(\cnn \right)}+1\right))^{-1}
\end{align*}
Thus, we conclude by \cref{highprobabilityeq10} and \cref{highprobabilityeq9}, combined with the above bound on the last expression in \cref{highprobabilityeq9}, that
\begin{align*}
  \p_{\rS_{l}}\negmedspace\bigg[ \max_{\rf\in \cA(l,T,r)} 80\p_{\rx\sim \cD(\cdot| A)}
  \left[
  \sum_{i=1}^{t} \ind\{h_{\rf,i}(\rx)=c(\rx)\}/t< 3/4
 \right]\leq \frac{1}{80\cdot6(\ln{(\cnn )}+1)}\bigg]\geq 1- \delta/\cff.
\end{align*}
This concludes \cref{highprobabilityeq11} (using \cref{highprobabilityeq13}) and, as argued, \cref{highprobabilityeq6} and  \cref{highprobabilityeq11},  combined upper bounds the term in \cref{highprobabilityeq5} with the condition on $ \rS_{0}\in E_{2} $, by at most $\delta/\cff$. Since the term with $ \rS_{0}\in E_{1} $ and $ \rS_{0}\in E_{3} $ in \cref{highprobabilityeq5}  was upper bounded by respectively $ 0 $ and $ \delta/\cnn $, which implies that \cref{highprobabilityeq5} is bounded by $ \delta/\cff +\delta/\cnn=\delta/20$, we obtain that \cref{highprobabilityeq7} is upper bounded by $ \delta $ as explained, and concludes the induction step and the proof.
\end{proof}

\section{Properties of $\as$ }\label{sec:propAdaBoostSample}
In this section, we present \as \cref{alg:AdaBoostSample} and prove the properties of $ \cA $  stated in \cref{lemma:adaboostsampleeasyproperties} and \cref{lemmaAdaBoostSampleefficent}. In the presentation of the algorithm, we use the following notation: For $r \in [0:1]$, and $C$ a cumulative distribution function for a distribution $D$ over $\{ 1:m \}$, we define $C^{-1}(r)$ to be the index $l$ in $\{ 1:m \}$ such that $C(l-1)\leq r$  and $C(l)> r$, we take $ C(0)=0 $. For a vector $ r\in [0:1]^{s} $, we let $ C^{-1}(r) $ denote  $ C^{-1} $ applied entrywise on $ r $, i.e., $C^{-1}(r)_{i}=C^{-1}(\rr_{i})$. With this definition, we have for $ \rr\sim[0:1]^{s} $  that $ C^{-1}(\rr)\sim \cD^{s} $ by the $ \rr_{i} $'s being i.i.d. $ \sim[0,1] $. We sometimes write $ (x,y)\gets(a,b) $ which means $ x=a $ and $ y=b $. We now present  \cref{alg:AdaBoostSample}.
\vspace{-0.3cm}
\begin{algorithm}[H]
\caption{$ \as $ $\cA$}\label{alg:AdaBoostSample}
\begin{algorithmic}[1]
\State\textbf{Input:} Training sequence $S\negmedspace\in\negmedspace(\cX\negmedspace\times\negmedspace \cY)^{*}\negmedspace,$ string $r\negmedspace \in\negmedspace ([0,1]^{*})^{*}\negmedspace\negmedspace.$ 
\State\textbf{Output:} Majority vote with margin $ 3/4.$
\State $(m,n,s) \gets (|S|,|r|,|r_{1}|)$ \label{alg:AdaBoostSamplemset} \label{alg:AdaBoostsamplesetn}\label{alg:AdaBoostsamplesets}
\State $(D_1,C_{1})\gets ((\frac{1}{m},\ldots,\frac{1}{m}),(\frac{1}{m},\ldots,1))$ \label{alg:AdaBoostSamplesetDone} \label{alg:AdaBoostSamplesetCone}
\State $t \gets \left\lceil 20^{2}\ln(m)/2\right\rceil$\label{alg:AdaBoostSampletset} 
\State $f \gets [0]^{t}$\label{alg:AdaBoostSamplesetf} 
\State $(\theta,\gamma)\gets (3/4,9/20)$ \label{alg:AdaBoostSamplethetaset}  \label{alg:AdaBoostSamplegammaset}
\State $l=\sum_{j=1}^{0}\ind_{\alpha_{i}>0}\gets 0$ \label{alg:AdaBoostSamplesettingl}
\State\textbf{for $ i \in \{ 1:n \} $ } \label{alg:AdaBoostSampleforloop}
\State\hspace{0.5cm}\vline\quad $S_i \gets \C_i^{-1}(r_{i})$\label{alg:AdaBoostSample:sample} 
\State\hspace{0.5cm}\vline\quad$h_i\leftarrow \erm (S_i )$\label{alg:AdaBoostSampletraining} 
\State\hspace{0.5cm}\vline\quad$1/2-\gamma_{i}=\loss_i=\ls_{D_{i}}(h_i)$\label{alg:AdaBoostSamplelossoftrained} 
\State\hspace{0.5cm}\vline\quad\textbf{if $1/2-\gamma_{i}=\loss_{i}\leq 1/2-\gamma$ \label{alg:AdaBoostsamplenonzeroweight} then}
\State\hspace{1cm}\vline\quad$(\alpha_i,\sum_{j=1}^{i}\ind_{\alpha_j>0})\gets(\frac{1}{2}\ln{\left(\frac{1+2\gamma}{1-2\gamma}\right)- \frac{1}{2}\ln{\left(\frac{1+\theta}{1-\theta} \right)}},\sum_{j=1}^{i-1}\ind_{\alpha_j>0}+1)$\label{alg:AdaBoostSamplealphaset}
\State\hspace{1cm}\vline\quad\textbf{if $\sum_{j=1}^{i}\ind_{\alpha_j>0}\leq t$ \label{alg:AdaBoostSampleearlystop} then}
\State\hspace{1.5cm}\vline\quad$(\gamma_{f,l},\epsilon_{f,l},h_{f,l},\alpha_{f,l},D_{f,l})\gets (\gamma_{i},\epsilon_{i},h_{i},\alpha_{i},D_{i})$\label{alg:AdaBoostSampleupdatel}
\State\hspace{1.5cm}\vline\quad$f\gets f+\alpha_{f,l} h_{f,l}$\label{alg:AdaBoostSampleaddsf}
\State\hspace{0.5cm}\vline\quad\textbf{else }
\State\hspace{1cm}\vline\quad$(\alpha_i,\sum_{j=1}^{i}\ind_{\alpha_j>0})\gets(0,\sum_{j=1}^{i-1}\ind_{\alpha_j>0})$ \label{alg:AdaBoostSamplezeroweight}
\State\hspace{0.5cm}\vline\quad\textbf{for $ i \in \{ 1:m \}$ }\label{alg:AdaBoostSampleupdatestart}
\State\hspace{1cm}\vline\quad$D_{i+1}(j)\gets D_{i}(j)\exp(-\alpha_i h_i(S_{i,1}) S_{i,2})$\label{alg:AdaBoostSamplegetweights}
\State\hspace{0.5cm}\vline\quad$Z_{i}\gets \sum_{j=1}^{m} D_{i+1}(j)$\label{AdaBoostSamplesetnormalization}
\State\hspace{0.5cm}\vline\quad\textbf{if $1/2-\gamma_{i}=\loss_{i}\leq 1/2-\gamma$ and $l=\sum_{j=1}^{n}\ind_{\alpha_j>0}\leq t$ then}
\State\hspace{1cm}\vline\quad $ Z_{f,l}\gets Z_{i} $\label{alg:AdaBoostSamplegetsZ} 
\State\hspace{0.5cm}\vline\quad$D_{i+1}\gets D_{i+1}/Z_{i}$\label{alg:AdaBoostSamplenormalizing}
\State\hspace{0.5cm}\vline\quad$C_{i+1}(1)\gets D_{i+1}(1)$\label{alg:AdaBoostSamplecummlativeupdate}
\State\hspace{0.5cm}\vline\quad\textbf{for $ j\in\{ 2:m \} $  }
\State\hspace{1cm}\vline\quad$C_{i+1}(j)\gets C_{i+1}(j-1)+D_{i+1}(j)$  \label{alg:AdaBoostsampleendforloop}
\State\textbf{if $\sum_{j=1}^{n}\ind_{\alpha_j>0}\geq t$ then}\label{alg:AdaBoostSamplecheckt}
\State\hspace{0.5cm}\vline\quad$f \gets \frac{f}{t\alpha_{f,1}}=\sum_{j=1}^{t} \frac{\alpha_{f,j}h_{f,j}}{\sum_{l=1}^{t}{\alpha_{f,l}}}=\sum_{j=1}^{t} \frac{h_{f,j}}{t}$ \label{alg:AdaBoostsample:output1} 
\State\hspace{0.5cm}\vline\quad\textbf{if $\sum_{x\in S} \ind\{\sum_{i=1}^{t} h_{f,i}(x)c(x)/t\leq \theta \}/m < \frac{1}{m}$ then}\label{alg:AdaBoostSamplecheckingmargin}
\State\hspace{1cm}\vline\quad\textbf{return $f$}\label{alg:AdaBoostSampleoutputgood}
\State\textbf{return $ \erm(S) $ }\label{alg:AdaBoostsample:output2}\label{alg:AdaBoostsample:output3}
\end{algorithmic}
\end{algorithm}
Line~\ref{alg:AdaBoostSamplemset} reads, the number of training examples $(x_i,y_i)\in\cX\times \cY$ in $S$, the number of entries in $r$, and number of entries in any $r_i$. 
Line~\ref{alg:AdaBoostSamplesetDone} initializes both, the distribution and cumulative distribution over $[m]$. 
Line~\ref{alg:AdaBoostSampletset} sets the number of hypotheses in the final majority vote for the early stopping criteria. 
Line~\ref{alg:AdaBoostSamplesetf} initializes $f$ as the $0$ function on $\mathcal{X}$/ an empty array of size $t.$ 
Line~\ref{alg:AdaBoostSamplethetaset} sets the target margin for the final majority vote and  sets target error for voters.
Line~\ref{alg:AdaBoostSamplesettingl} sets the counter for early stopping to $0$, represented both as $ l $ and the more descriptive $ \sum_{j=1}^{0}\ind_{\alpha_j>0}.$ 
Line~\ref{alg:AdaBoostSample:sample} decides $S_{i}\sqsubset S$ based on  $C^{-1}_{i}$ and $r_i\in [0,1]^{s}$, where $|S_{i}|=|r_{i}|=s$. 
Line~\ref{alg:AdaBoostSampletraining} runs $ \erm $ on $S_i$ to obtain hypothesis $ h_{i}.$ 
Line~\ref{alg:AdaBoostSamplelossoftrained} calculates the loss of the trained hypothesis $ h_{i}.$ 
Line~\ref{alg:AdaBoostsamplenonzeroweight} includes hypotheses with a loss smaller than $1/2-\gamma=1/2-9/20=1/20$. 
Line~\ref{alg:AdaBoostSamplealphaset}  sets $\alpha_{i} $ and updates the counter $ \sum_{j=1}^{i}\ind_{\alpha_j>0} $  with a plus $ 1 $ as this was a successful boosting step. 
Line~\ref{alg:AdaBoostSampleearlystop} ensures that the majority vote will consist of at most $t$ hypotheses by employing early stopping. 
Line~\ref{alg:AdaBoostSampleupdatel} renames $ \gamma_{i},\epsilon_{i},h_{i},\alpha_{i},D_{i}$. 
Line~\ref{alg:AdaBoostSampleaddsf} adds $\alpha_{f,l} h_{f,l}$ to $f$ / the array of size $t$, where we recall that $ l= \sum_{j=1}^{i}\ind_{\alpha_j>0} $. 
Line~\ref{alg:AdaBoostSamplezeroweight} ignores hypotheses with a loss larger than $1/2-\gamma$ by setting $ \alpha_{i}=0 $, and the counter for successful boosting steps is not increased in this case. 
Line~\ref{alg:AdaBoostSamplegetweights} updates $ D_i $, increasing the weight of misclassified points, where $ S_{i,1}\in \cX $ is the feature of $ i $'th point and  $S_{i,2}\in \cY $ is the label of the $ i $'th point. 
Line~\ref{AdaBoostSamplesetnormalization} calculates the normalization.
Line~\ref{alg:AdaBoostSamplegetsZ} rename the normalization term for successful boosting rounds. 
Line~\ref{alg:AdaBoostSamplenormalizing} normalizes $D_{i+1}$ to a probability distribution. 
Line~\ref{alg:AdaBoostSamplecummlativeupdate} to \ref{alg:AdaBoostsampleendforloop} calculates $C_{i+1}$. 
Line~\ref{alg:AdaBoostsample:output1} normalizes $f$. 
Line~\ref{alg:AdaBoostSamplecheckingmargin} checks for sufficient margins. 
Line~\ref{alg:AdaBoostSampleoutputgood} outputs $ f $ with $ \theta=3/4 $ margins for all $ x \in S$. 
Line~\ref{alg:AdaBoostsample:output2} boosting failed and $ \erm(S) $ is returned.

As commented on in the footnote of \cref{samplingfromrows} all other than the Line~\ref{alg:AdaBoostSampletraining} calling the $ \erm $, and Line~\ref{alg:AdaBoostSamplegetweights} making inference do only require working over the indexes $ m $  of the given training sequence, so these steps do not require reading the training examples explicitly. Reading the training examples in the call to the $ \erm $ and doing inference is captured in $ \Utrain $ and $ \Uinf.$

We begin with the proof of \cref{lemma:adaboostsampleeasyproperties}, which makes some observations about $ \cA$. For convenience, we restate \cref{lemma:adaboostsampleeasyproperties} here before providing its proof. 

\lemmaadaboostsampleeasyproperties*

\begin{proof}[Proof of \cref{lemma:adaboostsampleeasyproperties}]
The claim that $\cA(S,r)\in\dlh$ for $t=\left\lceil20^{2}\ln{(m)}/2\right\rceil$ follows from Line~\ref{alg:AdaBoostSamplealphaset} and Line~\ref{alg:AdaBoostSampleearlystop}, which implies that $f$ will never be updated more than $t$ times. Line~\ref{alg:AdaBoostSamplecheckt} ensures there are always at least $t$ hypotheses if Line~\ref{alg:AdaBoostsample:output1} is the output. If Line~\ref{alg:AdaBoostsample:output3} is the output, we can write it as $t$ copies of the same hypothesis and then normalize by $t$. 

That the output $f$ is such that $\sum_{x\in S} \ind\{\sum_{i=1}^{t} h_{f,i}(x)c(x)/t\leq \theta \}/m <1/m$, with $ \theta=3/4 $,  follows from Line~\ref{alg:AdaBoostSamplecheckingmargin}, ensuring this condition holds if $f$ in Line~\ref{alg:AdaBoostsample:output1} is the output. If $f$ is outputted in Line~\ref{alg:AdaBoostsample:output3}, it holds since $f$ is the output of a $\erm$-learner on the entire training sequence $S$, and therefore has a margin of $ 1 $ on all points.  
\end{proof}

We now move on to show \cref{lemmaAdaBoostSampleefficent}, i.e, that $ \cA $, with high probability, never runs $ \erm(S) $ and, in this case, is efficient to run. For convenience, we restate it here.  

\lemmaAdaBoostSampleefficent*

For the proof of this lemma, we need the following two lemmas.
The first lemma says that the output of $ \cA $, with high probability, is $ f $ from Line~\ref{alg:AdaBoostsample:output1} in Line~\ref{alg:AdaBoostSampleoutputgood}.

\begin{lemma}\label{adaboostsamplefewroundslemma}
  For hypothesis class $\cH$ of VC-dimension $d$, target concept $ c\in \cH $, training sequence 
   $S\in(\cX\times \cY)^{m}$ realizable by $c$ and of size $m\in \mathbb{N}$, failure parameter $ 0<\delta <1$,  random string $\rr \sim ([0:1]^{550d})^{n}$ of length $n\in\mathbb{N}$, boosting rounds $t=\left\lceil20^{2}\ln(m)/2\right\rceil$, as in Line~\ref{alg:AdaBoostSampletset}, then for  $n\geq 6\left\lceil20^{2}\ln{\left(8m/\delta \right)}/2\right\rceil$ we have, with probability at least $1-(\delta/(8m))^{20}$ over $ \rr $, that \cref{alg:AdaBoostSample} run on $S,\rr$, $\cA(S,\rr)$ outputs $f$ from Line~\ref{alg:AdaBoostsample:output1} in Line~\ref{alg:AdaBoostSampleoutputgood}.
  \end{lemma}

The next lemma shows that if $ f $ from Line~\ref{alg:AdaBoostsample:output1} in Line~\ref{alg:AdaBoostSampleoutputgood} is the output of $ \cA $, then $ \cA $ uses few operations.   

\begin{lemma}\label{adaboostsampleruntime}
  For hypothesis class $\cH$ of VC-dimension $d$, target concept $ c\in \cH $, training sequence 
  $S\in(\cX\times \cY)^{m}$ realizable by $c$, and of size $m\in\mathbf{N}$, and string $r\in([0:1]^{s})^{n}$ with $s,n\in\mathbb{N}$, then if the output  $f=\cA(S,r)$ of \cref{alg:AdaBoostSample} is $f$ from Line~\ref{alg:AdaBoostsample:output1} outputted in Line~\ref{alg:AdaBoostSampleoutputgood}, then $\cA(S,r)$ uses fewer operations than: 
\begin{align*}
  n\left(O\left(m+s\ln{\left(m \right)}\right)+\Utrain(s)+3m\Uinf\right).
\end{align*}  
\end{lemma}

With the above two lemmas, we now give the proof of \cref{lemmaAdaBoostSampleefficent}.

\begin{proof}[Proof of \cref{lemmaAdaBoostSampleefficent}]
     \cref{adaboostsamplefewroundslemma} give that $\cA(S,\rr)$  outputs $f$ from Line~\ref{alg:AdaBoostsample:output1} in Line~\ref{alg:AdaBoostSampleoutputgood} with probability at least $1-(\delta/(8m))^{20}$, since $ n\geq 6\lceil20^{2}\ln{(8m/\delta )}/2\rceil $ and $ t=\lceil 20^{2}\ln{(m )}/2\rceil $. For the above parameters and $ s=550d $, \cref{adaboostsampleruntime} implies that
    \begin{align*}
      n\left(O\left(m+s\ln{\left(m \right)}\right)+\Utrain(s)+3m\Uinf\right)= n\left(O\left(m+d\ln{\left(m \right)}\right)+\Utrain(550d)+3m\Uinf\right)
       \end{align*}
which concludes the proof
\end{proof}
We now, proceed to prove \cref{adaboostsamplefewroundslemma} and \cref{adaboostsampleruntime}, starting with the former. In the proof of \cref{adaboostsamplefewroundslemma}, we need the following lemma, whose proof can be found in \cref{appendixpropertiesadaboossample}. The lemma gives an in-sample margin guarantee of \as after $ t $ rounds of successful boosting. Since \as is a slightly modified version of AdaBoost, using early stopping and a fixed learning rate, we thought it was good practice to write out the margin error analysis for $ \cA $ (the steps follow \cite{boostingbookSchapireF12} closely).  

\begin{restatable}{lemma}{AdaBoostSampleMarginLemmathree}\label{AdaBoostSampleMarginLemma3}
Consider \cref{alg:AdaBoostSample} run on a training sequence $S\in\left(\cX\times\cY\right)^{m}$, string $r\in([0:1]^{*})^{*}$, Line~\ref{alg:AdaBoostSampletset} set to a $t\in\mathbb{N}$, Line~\ref{alg:AdaBoostSamplegammaset} set with a $0< \gamma< 1/2$, and Line~\ref{alg:AdaBoostSamplethetaset} set with a $0<\theta <1$ such that $\theta<2\gamma$. If $\sum_{j=1}^{n}\ind_{\alpha_j>0}\geq t$, we have that $f=\sum_{j=1}^{t} \frac{\alpha_{f,j}h_{f,j}}{\sum_{l=1}^{t}{\alpha_{f,l}}}=\sum_{j=1}^{t} h_{f,j}/t$, the function in \cref{alg:AdaBoostSample} Line~\ref{alg:AdaBoostsample:output1} for $ \alpha=(1/2)\ln{\left((1+2\gamma)/(1-2\gamma)\right)- (1/2)\ln{\left((1+\theta)/(1-\theta) \right)}}>0 $  satisfies,
\begin{align*}
  \sum_{x\in S} \ind\{\sum_{i=1}^{t} h_{f,i}(x)c(x)/t\leq \theta \}/m  =\p_{(\rx,\ry)\sim S}\left[\ry f(\rx)\leq \theta \right]\\\leq \big(\exp{((\theta-1)\alpha)}(1/2+\gamma)+\exp{((\theta+1)\alpha)}(1/2-\gamma)\big)^{t},
\end{align*}
which for $ \theta=3/4 $ and $ \gamma=9/20 $, (that satisfies $ \theta/2 =3/8 <\gamma=9/20 $) is at most $ (24/25)^{t} $.  
\end{restatable}

We also need the following lemma in the proof of \cref{adaboostsamplefewroundslemma}. The following lemma is the uniform convergence lemma for the realizable case of binary classification, which gives a bound on the difference between the in-sample and out-sample error of all consistent $h\in \cH$ simultaneously.

\begin{lemma}\label{uniformconvergence}[\cite{vapnik74theory}, \cite{Blumeruniformconvergence} from \cite{Simons}[Theorem 2]]
 For $0<\delta,\eps<1$, hypothesis class $\cH$ of VC-dimension $d$, target concept $c\in\cH$, distribution $\cD$ over $\cX$, and sample $\rS\sim\cD_{c}^{m}$,  we have with probability at least $1-\delta$ over $\rS$, that for all $h\in \cH$ with $\ls_{\rS}(h)=0$, it holds that
    \begin{align*}
    \ls_{\cD}(h)\leq 2\frac{d\log_{2}{\left(2em/d \right)}+\log_{2}{\left(2/\delta \right)}}{m}.
    \end{align*}
\end{lemma}

With the above two lemmas, we are now ready to prove \cref{adaboostsamplefewroundslemma}.

\begin{proof}[Proof of \cref{adaboostsamplefewroundslemma}]
    
Let $S\in (\cX\times \cY)^{m}$. We observe that $\as$ outputs $f$ from Line~\ref{alg:AdaBoostsample:output1} in Line~\ref{alg:AdaBoostSampleoutputgood} only if 
$\sum_{i=1}^{n} \ind_{\alphar_i>0} \geq t= \left\lceil 20^{2}\ln(m)/2 \right\rceil$ and 

\begin{align*}
    \sum_{x\in S} \ind\{\sum_{i=1}^{t} h_{f,i}(x)c(x)/t\leq \theta \}/m = \sum_{x\in S} \ind\{\sum_{i=1}^{t} h_{f,i}(x)c(x)/t\leq 3/4 \}/m < \frac{1}{m}.
\end{align*}

We now show that $\sum_{i=1}^{n} \ind_{\alphar_i>0} \geq \left\lceil20^{2}\ln(m)/2\right\rceil$ happens with probability at least $1-(\delta/(8m))^{20}$ over $\rr$, and that if this happens, it also holds that $\sum_{x\in S}  
                \ind\{\sum_{i=1}^{t} h_{f,i}(x)c(x)/t\leq 3/4\} < \frac{1}{m}$.

To see that $\sum_{x\in S} \ind\{\sum_{i=1}^{t} h_{f,i}(x)c(x)/t\leq 3/4 \}/m < 1/m$ holds when $\sum_{i=1}^{n} \ind_{\alphar_i>0} \geq \left\lceil 20^{2}\ln(m)/2 \right\rceil$, we invoke \cref{AdaBoostSampleMarginLemma3} with $t=\left\lceil 20^{2}\ln(m)/2 \right\rceil$,  $\gamma=9/20$ and $\theta=3/4$, and get by numerical calculations $ \ln{(24/25 )} 20^{2}/2\leq-8$, which implies that:
\begin{align*}
    \sum_{x\in S} \ind\{\sum_{i=1}^{t} h_{f,i}(x)c(x)/t\leq 3/4 \}/m =\p_{(\rx,\ry)\sim S}\left[\ry f(\rx)\leq \theta \right]\negmedspace\leq \negmedspace(24/25)^{t}\negmedspace=\exp(-8\ln{(m )})\negmedspace<\negmedspace1/m
\end{align*}  
as claim.

Thus, we now proceed to show that $\sum_{i=1}^{n} \ind_{\alphar_i>0} \geq t= \left\lceil 20^{2}\ln(m)/2 \right\rceil$ with probability at least $1-(\delta/(8m))^{20}$ over $\rr$. We first notice that this is a sum of $\{0,1\}$-random variables that are not independent or identically distributed, however $\alphar_i$ is a function of $\rr_i$ and $\rD_i$ through Line~\ref{alg:AdaBoostSample:sample}, where $\rD_i$ is only a function of $\rr_1,\ldots,\rr_{i-1}$. Now let $r_1,\ldots,r_{i-1}$ be a realization of $\rr_1,\ldots,\rr_{i-1}$, and let $D_i$ denote the realization of $\rD_i$. For this realization, $\rS_i$ is a training sequence of size $550d$ drawn according to $D_i$ (by the comment before \cref{alg:AdaBoostSample}), and since $\rhh_i$ is the output of $\erm(\rS_i)$, we have that $\ls_{\rS_i}(\rh_i)=0$, and that $\rhh_{i}\in \cH$. Thus, it follows from \cref{uniformconvergence} with $\delta= 2^{-d} $ and $m=550d=\cu d$ (i.e. $\cu=550$) that with probability at least $1-2^{-d}$ over $\rr_{i}$ we have that 
\begin{align*}
    \ls_{D_i}(\rS_i)\leq2\frac{d\log_{2}{\left(2e(\cu d)/d \right)}+\log_{2}{\left(2^{d+1}  \right)}}{\cu d}\leq 2\frac{\log_{2}{\left(2e \cu \right)}+2}{\cu }\leq 1/20,  
\end{align*}
which by Line~\ref{alg:AdaBoostsamplenonzeroweight} implies $\alphar_{i}>0$ i.e. we have shown that 
\begin{align*}
\e_{\rr_{i}}[\ind_{\alphar_i>0}]\geq 1-2^{-d}:=p.                          \end{align*} 
Furthermore, for any $t\leq0 $ and using that $\ln{\left(1+x \right)}\leq x$ for any $x>-1$ we have   
\begin{align*}
&\e_{\rr_{i}}[\exp{\left(t\ind_{\alphar_i>0} \right)}]\\
&\leq (1-\e_{\rr_{i}}[{\ind_{\alphar_i>0}}])+\e_{\rr_{i}}[\ind_{\alphar_i>0}]\exp(t)\\
&=1+\e_{\rr_{i}}[\ind_{\alphar_i>0}](\exp{\left(t \right)}-1)
\\ 
&\leq \exp(\e_{\rr_{i}}[\ind_{\alphar_i>0}](\exp{\left(t \right)}-1))
\leq 
\exp\left(p(\exp{\left(t \right)}-1)\right).  
\end{align*} 
Since we showed the above for any realization $r_1,\ldots,r_{i-1}$ of $\rr_1,\ldots,\rr_{i-1}$ and for any $i$, we get, by independence of the $\rr_{i}$'s, applying the above recursively, and using that $\alphar_{i-1}$ is only a function of $\rr_{1},\ldots,\rr_{i-1}$,  that 
 
\begin{align*}
 &\e_{\rr_{1},\ldots,\rr_{n}}\Big[{\exp{\Big(t\sum_{i=1}^{n}\ind_{\alphar_i>0} \Big)}}\Big]
 =\e_{\rr_{1},\ldots,\rr_{n-1}}\Big[{\exp{\Big(t\sum_{i=1}^{n-1}\ind_{\alphar_i>0} \Big)}\e_{\rr_{n}}\Big[{\exp{\Big(t\ind_{\alphar_n>0} \Big)}}\Big]}\Big]\\
\leq&\e_{\rr_{1},\ldots,\rr_{n-1}}\Big[{\exp{\Big(t\sum_{i=1}^{n-1}\ind_{\alphar_i>0} \Big)}\exp{\left(p\left(\exp{\left(t \right)}-1\right) \right)}}\Big]
 \leq \hdots
\leq \exp{\left(np\left(\exp{\left(t \right)}-1\right) \right)}.
\end{align*}
Setting $t=\ln{\left(1-\rho \right)}\leq 0$ for $0<\rho<1$, we conclude that 
\begin{align*}
\p_{\rr_{1},\ldots,\rr_{n}}\left[\sum_{i=1}^{n}\ind_{\alphar_i>0}< (1-\rho)np\right]\leq \exp{\left(np\left(\exp{\left(t \right)}-1\right)-t(1-\rho)np \right)}=\left(\frac{\exp{\left(-\rho \right)}}{(1-\rho)^{1-\rho}}\right)^{np}.
\end{align*}

Thus, setting $\rho=2/3$ so that $ (\exp(-3/4)/(1/3)^{1/3})^{1/2}\leq 83/100 $, and using that $d\geq 1$ so that $p=1-2^{-d}\geq 1/2$, we get that
\begin{align*}
\p_{\rr_{1},\ldots,\rr_{n}}\left[\sum_{i=1}^{n}\ind_{\alphar_i>0}< n/6\right]\leq \p_{\rr_{1},\ldots,\rr_{n}}\left[\sum_{i=1}^{n}\ind_{\alphar_i>0}< (1-\rho)np\right]\leq (\exp(-3/4)/(1/3)^{1/3})^{n/2}\\\leq (83/100)^{n}.
    \end{align*}
Using that $n\geq 6\left\lceil20^{2}\ln{(8m/\delta)}/2\right\rceil$ so that $\lceil20^{2}\ln{(m )/2}\rceil=t\leq n/6$ and that $(83/100)^{6\cdot20^{2}/2}\leq (1/e)^{20}$, we get that 
\begin{align*}
\p_{\rr_{1},\ldots,\rr_{n}}\left[\sum_{i=1}^{n}\ind_{\alphar_i>0}< \left\lceil 20^{2}\ln(m)/2 \right\rceil\right]\leq \p_{\rr_{1},\ldots,\rr_{n}}\left[\sum_{i=1}^{n}\ind_{\alphar_i>0}< n/6\right]
\leq \left(\frac{\delta}{8m}\right)^{ 20},
    \end{align*} 
which concludes the proof. 

\end{proof}

We now give the proof of \cref{adaboostsampleruntime}.

\begin{proof}[Proof of \cref{adaboostsampleruntime}]

We first set parameters (Line~\ref{alg:AdaBoostSamplemset} to Line~\ref{alg:AdaBoostSamplesettingl}). We first read the length of $S, r$, and $r_{1}$, which takes $O(m+ns)$ operations. Making $D_{1}$ and $C_{1}$ takes $O(m)$ operations. Calculating/setting $t,\theta,\gamma,\sum_{j=1}^{0}\ind_{\alpha_{i}>0}$ takes $O(1)$ operations. Initializing $f$ as an array of size $t$ takes $O(t)=O(\ln{(m)})$ operations. Thus, the lines from Line~\ref{alg:AdaBoostSamplemset} to Line~\ref{alg:AdaBoostSamplesettingl} take $O(m+ns)$ operations.

We next analyze the for loop over $n$ starting in Line~\ref{alg:AdaBoostSampleforloop} and ending in Line~\ref{alg:AdaBoostsampleendforloop}. In each iteration, we do the following: We find $S_i=C_{i}^{-1}(r_{i})$ using $r_{i}\in[0:1]^{s}$ and $C_{i}^{-1}$. We recall that for $r_{i,j}\in[0:1]$, $C_{i}^{-1}(r_{i,j})$ was defined as the index $l$ in $\{ 1:m \}$ such that $C_{i}(l-1)\leq r_{i,j}$  and $r_{i,j}<C_{i}(l) $. Since $C_{i}$ is sorted in increasing order in $ i $, allowing for binary search, $ l $ can be found in $ O(\ln{(m )}) $ reads from $ C_{i} $. That is given an interval $ I\subset \{ 1:m \} $ chose the middle point of $ I $, call it $ l' $, if $r_{i,j}$ is between $ C_{i}(l'-1) $ and $ C_{i}(l') $ set $ l=l'$, if $r_{i,j} <C_{i}(l'-1)$ then recurse on the interval $ I\cap\{ 1:l'-2 \} $ else recurse on $ I\cap\{ l': \} $. Since the interval is halved each time the index is not found and the interval starts which $ \{ 1:m \} $, this is at most $O( \ln{(m )}) $ reads/operations from $ C_{i} $. Since $C_{i}^{-1}(r_{i})$ is applying $ C_{i}^{-1} $ entry-wise to $ r_{i} $  this takes $O(s\ln{\left(m \right)})$ operations. Calculating $h_{i}$ takes $\Utrain(s)$. Calculating $\ls_{D_{i}}$ takes $m\cdot (\Uinf+O(1))$ operations. Checking $\eps_{i}\leq 1/2-\gamma$, setting $\alpha_{i}$ and updating the counter $\sum_{j=1}^{i}\ind_{\alpha_{j}>0}$ takes $O(1)$ operations. Checking the counter $\sum_{j=1}^{i}\ind_{\alpha_{j}>0}<t$ takes $O(1)$ operations. Renaming $\gamma_{i},\epsilon_{i},h_{i},\alpha_{i},D_{i}$ takes $O(1)$ operation, and adding $h_{f,l}$ to $f$ takes $O(1)$ operation ($ f $ is normalized in the last line with $ \alpha_{f,l} $, which is the same for all $ h_{f,l} $, so we only save $ h_{f,l} $ unscaled). Setting $\alpha_{i}=0$ and updating the counter $\sum_{j=1}^{i}\ind_{\alpha_{j}>0}$ in the else statement takes $O(1)$ operation. To compute $D_{i+1}$ in Line~\ref{alg:AdaBoostSamplegetweights} we preform inference using $h_{i}$ on all the $m$ points, which takes $m\cdot(\Uinf+O(1))$ operations. Calculating $Z_{i}$, $D_{i+1}$, and $C_{i+1}$ afterwards requires $O(m)$ operations. Thus, each iteration of the for loop, starting from  Line~\ref{alg:AdaBoostSampleforloop} and ending at Line~\ref{alg:AdaBoostsampleendforloop} uses 
\begin{align*}
    O(s\ln{\left(m \right)})+\Utrain(s)+m\cdot (2\Uinf+O(1))+O(m)
\end{align*}
operations. Therefore, the total cost over the entire for loop is
\begin{align*}
    n(O(s\ln{\left(m \right)}+m)+\Utrain(s)+2m\cdot \Uinf).
\end{align*}

Next, we analyze Line~\ref{alg:AdaBoostSamplecheckt} to Line~\ref{alg:AdaBoostSampleoutputgood}. First, Line~\ref{alg:AdaBoostSamplecheckt}, checking/reading $\sum_{j=1}^{n}\ind_{\alpha_{i}>0}$ takes $O(1)$. Renaming $f$ takes $O(t)=O(\ln{(m)})$ operations, as we have $t$ functions. Checking that all points has $ \theta $-margin takes $ tm (\Uinf+2)\leq nm(\Uinf+2) $, as $ t\leq n $ in the case that we reach this condition (else $\sum_{j=1}^{n}\ind_{\alpha_{i}>0}\geq t$ could not have happened). Thus, Line~\ref{alg:AdaBoostSamplecheckt} to Line~\ref{alg:AdaBoostSampleoutputgood}, takes at most $nm(\Uinf+2)+O(\ln{\left(m \right)})$ operations are required.

Consequently, we conclude that the above steps take at most:
\begin{align*}
  O(m+ns)+n(O(s\ln{\left(m \right)}+m)+\Utrain(s)+3m\cdot \Uinf)+O(\ln{(m)})\\
  =n(O\left(m+s\ln{\left(m \right)}\right)+\Utrain(s)+3m\Uinf),
\end{align*}
operations, which concludes the proof. 
\end{proof}
\bibliographystyle{plainnat}
\bibliography{ref}

\begin{thebibliography}{31}
\providecommand{\natexlab}[1]{#1}
\providecommand{\url}[1]{\texttt{#1}}
\expandafter\ifx\csname urlstyle\endcsname\relax
  \providecommand{\doi}[1]{doi: #1}\else
  \providecommand{\doi}{doi: \begingroup \urlstyle{rm}\Url}\fi

\bibitem[Aden-Ali et~al.(2023)Aden-Ali, Cherapanamjeri, Shetty, and Zhivotovskiy]{optimalwithoutuniformconvergence}
Ishaq Aden-Ali, Yeshwanth Cherapanamjeri, Abhishek Shetty, and Nikita Zhivotovskiy.
\newblock Optimal pac bounds without uniform convergence.
\newblock pages 1203--1223, 11 2023.
\newblock \doi{10.1109/FOCS57990.2023.00071}.

\bibitem[Aden-Ali et~al.(2024)Aden-Ali, H{\o}andgsgaard, Larsen, and Zhivotovskiy]{majorityofthree}
Ishaq Aden-Ali, Mikael~M{\o}ller H{\o}andgsgaard, Kasper~Green Larsen, and Nikita Zhivotovskiy.
\newblock Majority-of-three: The simplest optimal learner?
\newblock In \emph{The Thirty Seventh Annual Conference on Learning Theory}, pages 22--45. PMLR, 2024.

\bibitem[Auer and Ortner(2007)]{auer2007new}
Peter Auer and Ronald Ortner.
\newblock A new {PAC} bound for intersection-closed concept classes.
\newblock \emph{Machine Learning}, 66\penalty0 (2):\penalty0 151--163, 2007.

\bibitem[Bartlett and Mendelson(2003)]{Bartlett2003RademacherAG}
Peter~L. Bartlett and Shahar Mendelson.
\newblock Rademacher and gaussian complexities: Risk bounds and structural results.
\newblock In \emph{Journal of machine learning research}, 2003.
\newblock URL \url{https://api.semanticscholar.org/CorpusID:463216}.

\bibitem[Blumer et~al.(1989{\natexlab{a}})Blumer, Ehrenfeucht, Haussler, and Warmuth]{Blumeruniformconvergence}
Anselm Blumer, A.~Ehrenfeucht, David Haussler, and Manfred~K. Warmuth.
\newblock Learnability and the vapnik-chervonenkis dimension.
\newblock \emph{J. ACM}, 36\penalty0 (4):\penalty0 929–965, oct 1989{\natexlab{a}}.
\newblock ISSN 0004-5411.
\newblock \doi{10.1145/76359.76371}.
\newblock URL \url{https://doi.org/10.1145/76359.76371}.

\bibitem[Blumer et~al.(1989{\natexlab{b}})Blumer, Ehrenfeucht, Haussler, and Warmuth]{measureone}
Anselm Blumer, A.~Ehrenfeucht, David Haussler, and Manfred~K. Warmuth.
\newblock Learnability and the vapnik-chervonenkis dimension.
\newblock \emph{J. ACM}, 36\penalty0 (4):\penalty0 929–965, oct 1989{\natexlab{b}}.
\newblock ISSN 0004-5411.
\newblock \doi{10.1145/76359.76371}.
\newblock URL \url{https://doi.org/10.1145/76359.76371}.

\bibitem[Blumer et~al.(1989{\natexlab{c}})Blumer, Ehrenfeucht, Haussler, and Warmuth]{blumer1989learnability}
Anselm Blumer, Andrzej Ehrenfeucht, David Haussler, and Manfred~K Warmuth.
\newblock Learnability and the {V}apnik-{C}hervonenkis dimension.
\newblock \emph{Journal of the ACM}, 36\penalty0 (4):\penalty0 929--965, 1989{\natexlab{c}}.

\bibitem[Bousquet et~al.(2020)Bousquet, Hanneke, Moran, and Zhivotovskiy]{bousquet2020proper}
Olivier Bousquet, Steve Hanneke, Shay Moran, and Nikita Zhivotovskiy.
\newblock Proper learning, {H}elly number, and an optimal {SVM} bound.
\newblock In \emph{Conference on Learning Theory}, pages 582--609. PMLR, 2020.

\bibitem[Breiman(1996)]{Breiman1996BaggingP}
L.~Breiman.
\newblock Bagging predictors.
\newblock \emph{Machine Learning}, 24:\penalty0 123--140, 1996.
\newblock URL \url{https://api.semanticscholar.org/CorpusID:47328136}.

\bibitem[David(2012)]{measuretwo}
Pollard David.
\newblock \emph{Convergence of stochastic processes}.
\newblock Springer Science \& Business Media, 2012.

\bibitem[Dudley(1978)]{rademacherDudley1978}
R.~M. Dudley.
\newblock {Central Limit Theorems for Empirical Measures}.
\newblock \emph{The Annals of Probability}, 6\penalty0 (6):\penalty0 899 -- 929, 1978.
\newblock \doi{10.1214/aop/1176995384}.
\newblock URL \url{https://doi.org/10.1214/aop/1176995384}.

\bibitem[Ehrenfeucht et~al.(1989)Ehrenfeucht, Haussler, Kearns, and Valiant]{EHRENFEUCHT1989247}
Andrzej Ehrenfeucht, David Haussler, Michael Kearns, and Leslie Valiant.
\newblock A general lower bound on the number of examples needed for learning.
\newblock \emph{Information and Computation}, 82\penalty0 (3):\penalty0 247--261, 1989.
\newblock ISSN 0890-5401.
\newblock \doi{https://doi.org/10.1016/0890-5401(89)90002-3}.
\newblock URL \url{https://www.sciencedirect.com/science/article/pii/0890540189900023}.

\bibitem[Freund and Schapire(1997)]{Adaboost}
Yoav Freund and Robert~E Schapire.
\newblock A decision-theoretic generalization of on-line learning and an application to boosting.
\newblock \emph{Journal of Computer and System Sciences}, 55\penalty0 (1):\penalty0 119--139, 1997.
\newblock ISSN 0022-0000.
\newblock \doi{https://doi.org/10.1006/jcss.1997.1504}.
\newblock URL \url{https://www.sciencedirect.com/science/article/pii/S002200009791504X}.

\bibitem[Hajek and Raginsky(2021)]{rademacherboundlecturenotes}
Bruce Hajek and Maxim Raginsky.
\newblock {ECE 543: Statistical Learning Theory}.
\newblock \textit{Department of Electrical and Computer Engineering and the Coordinated Science Laboratory, University of Illinois at Urbana-Champaign}, 2021.
\newblock URL \url{http://maxim.ece.illinois.edu/teaching/SLT}.
\newblock Last updated March 18, 2021.

\bibitem[Hanneke(2016{\natexlab{a}})]{hanneke2016refined}
Steve Hanneke.
\newblock Refined error bounds for several learning algorithms.
\newblock \emph{The Journal of Machine Learning Research}, 17\penalty0 (1):\penalty0 4667--4721, 2016{\natexlab{a}}.

\bibitem[Hanneke(2016{\natexlab{b}})]{hannekeoptimal}
Steve Hanneke.
\newblock The optimal sample complexity of pac learning.
\newblock \emph{Journal of Machine Learning Research}, 17\penalty0 (38):\penalty0 1--15, 2016{\natexlab{b}}.
\newblock URL \url{http://jmlr.org/papers/v17/15-389.html}.

\bibitem[Haussler et~al.(1994)Haussler, Littlestone, and Warmuth]{haussler1994predicting}
David Haussler, Nick Littlestone, and Manfred~K Warmuth.
\newblock Predicting $\{$0, 1$\}$-functions on randomly drawn points.
\newblock \emph{Information and Computation}, 115\penalty0 (2):\penalty0 248--292, 1994.

\bibitem[Larsen(2023)]{baggingoptimalPAClearner}
Kasper~Green Larsen.
\newblock Bagging is an optimal {PAC} learner.
\newblock In Gergely Neu and Lorenzo Rosasco, editors, \emph{The Thirty Sixth Annual Conference on Learning Theory, {COLT} 2023, 12-15 July 2023, Bangalore, India}, volume 195 of \emph{Proceedings of Machine Learning Research}, pages 450--468. {PMLR}, 2023.
\newblock URL \url{https://proceedings.mlr.press/v195/larsen23a.html}.

\bibitem[Larsen and Ritzert(2022)]{Optimalweaktostronglearning}
Kasper~Green Larsen and Martin Ritzert.
\newblock Optimal weak to strong learning.
\newblock In Sanmi Koyejo, S.~Mohamed, A.~Agarwal, Danielle Belgrave, K.~Cho, and A.~Oh, editors, \emph{Advances in Neural Information Processing Systems 35: Annual Conference on Neural Information Processing Systems 2022, NeurIPS 2022, New Orleans, LA, USA, November 28 - December 9, 2022}, 2022.
\newblock URL \url{http://papers.nips.cc/paper\_files/paper/2022/hash/d38653cdaa8e992549e1e9e1621610d7-Abstract-Conference.html}.

\bibitem[Moran and Yehudayoff(2016)]{samplecompressionschemesamirmoran}
Shay Moran and Amir Yehudayoff.
\newblock Sample compression schemes for {VC} classes.
\newblock \emph{J. {ACM}}, 63\penalty0 (3):\penalty0 21:1--21:10, 2016.
\newblock \doi{10.1145/2890490}.
\newblock URL \url{https://doi.org/10.1145/2890490}.

\bibitem[Novikoff(1963)]{Novikoff1963ONCP}
Albert B.~J. Novikoff.
\newblock On convergence proofs for perceptrons.
\newblock 1963.
\newblock URL \url{https://api.semanticscholar.org/CorpusID:122810543}.

\bibitem[Rosenblatt(1958)]{Rosenblatt1958ThePA}
Frank Rosenblatt.
\newblock The perceptron: a probabilistic model for information storage and organization in the brain.
\newblock \emph{Psychological review}, 65 6:\penalty0 386--408, 1958.
\newblock URL \url{https://api.semanticscholar.org/CorpusID:12781225}.

\bibitem[Schapire and Freund(2012)]{boostingbookSchapireF12}
Robert~E. Schapire and Yoav Freund.
\newblock \emph{{Boosting: Foundations and Algorithms}}.
\newblock The MIT Press, 05 2012.
\newblock ISBN 9780262301183.
\newblock \doi{10.7551/mitpress/8291.001.0001}.
\newblock URL \url{https://doi.org/10.7551/mitpress/8291.001.0001}.

\bibitem[Shalev-Shwartz and Ben-David(2014)]{understandingmachinelearning}
Shai Shalev-Shwartz and Shai Ben-David.
\newblock \emph{Understanding Machine Learning: From Theory to Algorithms}.
\newblock Cambridge University Press, USA, 2014.
\newblock ISBN 1107057132.

\bibitem[Simon(2015{\natexlab{a}})]{Simons}
Hans~U. Simon.
\newblock An almost optimal pac algorithm.
\newblock In Peter Grünwald, Elad Hazan, and Satyen Kale, editors, \emph{Proceedings of The 28th Conference on Learning Theory}, volume~40 of \emph{Proceedings of Machine Learning Research}, pages 1552--1563, Paris, France, 03--06 Jul 2015{\natexlab{a}}. PMLR.
\newblock URL \url{https://proceedings.mlr.press/v40/Simon15a.html}.

\bibitem[Simon(2015{\natexlab{b}})]{simon2015almost}
Hans~U Simon.
\newblock An almost optimal {PAC} algorithm.
\newblock In \emph{Conference on Learning Theory}, pages 1552--1563. PMLR, 2015{\natexlab{b}}.

\bibitem[Valiant(1984)]{valiant1984theory}
Leslie~G Valiant.
\newblock A theory of the learnable.
\newblock \emph{Communications of the ACM}, 27\penalty0 (11):\penalty0 1134--1142, 1984.

\bibitem[Vapnik and Chervonenkis(1974)]{vapnik74theory}
V.~Vapnik and A.~Chervonenkis.
\newblock \emph{Theory of Pattern Recognition [in Russian]}.
\newblock Nauka, Moscow, 1974.
\newblock (German Translation: W.~Wapnik \& A.~Tscherwonenkis, {\em Theorie der Zeichenerkennung}, Akademie--Verlag, Berlin, 1979).

\bibitem[Vapnik and Chervonenkis(1964)]{vapnik1964class}
Vladimir Vapnik and Alexey Chervonenkis.
\newblock A class of algorithms for pattern recognition learning.
\newblock \emph{Avtomatika i Telemekhanika}, 25\penalty0 (6):\penalty0 937--945, 1964.

\bibitem[Vapnik and Chervonenkis(1968)]{vapnik1968algorithms}
Vladimir Vapnik and Alexey Chervonenkis.
\newblock Algorithms with complete memory and recurrent algorithms in the problem of learning pattern recognition.
\newblock \emph{Avtomatika i Telemekhanika}, pages 95--106, 1968.

\bibitem[Vapnik and Chervonenkis(1971)]{vapnik1971uniform}
Vladimir Vapnik and Alexey Chervonenkis.
\newblock On uniform convergence of the frequencies of events to their probabilities.
\newblock \emph{Teoriya Veroyatnostei i ee Primeneniya}, 16\penalty0 (2):\penalty0 264--279, 1971.

\end{thebibliography}
\appendix

\section{Perceptron example}\label{appendix:perceptron}
In the following, we consider $m \in \mathbb{N}, m \geq 2200$, and the perceptron algorithm (due to \cite{Rosenblatt1958ThePA}, see e.g. \cite{understandingmachinelearning} chapter 9.1.2). We assume that input points $ x\in \mathbf{R}^{d} $  have a hard-coded a $ 1 $ in the d'th coordinate $ x_{d}=1 $ such that $ xw=\sum_{i=1}^{d-1}x_{i}w_{i}+w_{d},$ i.e. $ w_{d} $ is the bias of the hyperplane $ w.$      
\begin{algorithm}[H]
\caption{Perceptron Algorithm}\label{alg:perceptron}
\begin{algorithmic}[1]
\State \textbf{Input:} Training set $\{(x_i, y_i)\}_{i=1}^n$ with $x_i \in \mathbb{R}^d,$ such that $ x_{i,d}=1 $ for all $ i=1\ldots,n $  and $y_i \in \{-1, 1\}$
\State \textbf{Output:} Weight vector $w$
\State Initialize $w_0 \leftarrow 0$
\State \textbf{Repeat:} until no classification mistake for $i = 1, \ldots, n$
\State\hspace{0.5cm}\vline\quad \textbf{for: $i \in\{ 1, \ldots, n\}$}
\State\hspace{1cm}\vline\quad\textbf{if $y_i \cdot \left\langle w_{j-1}, x_i \right\rangle \leq 0$ then}
\State\hspace{1.5cm}\vline\quad $w_j \leftarrow w_{j-1} + y_i \cdot x_i$
\State \Return last iterate $w_j$
\end{algorithmic}
\end{algorithm}

Let $x_i = \left(0, 1 - \frac{i}{m^4},1\right)$ and $y_i = -1$ for $i = 1, \ldots, m-1$, and let $x_m = \left(\sqrt{\frac{1}{m}}, 1,1\right)$ and $y_m = 1$. We observe that 
$
2 - \frac{2}{m^3} < \left\langle x_i, x_j \right\rangle < 2 - \frac{i}{m^4}, \quad \text{for } i, j \in \{1, \ldots, m-1\},
$
$
\left\langle x_i, x_m \right\rangle = 2 - \frac{i}{m^4}, \quad \text{for } i \in \{1, \ldots, m-1\},
$
and 
$
\left\langle x_m, x_m \right\rangle = 2 + \frac{1}{m}.
$

Furthermore, we notice that if $r_1, \ldots, r_t \in \{1, \ldots, m-1\}$ and 
$
w' = \sum_{q=1}^t y_{r_q} x_{r_q} + t y_m x_m = -\sum_{q=1}^t x_{r_q} + t x_m,
$
then for any $i \in \{1, \ldots, m-1\}$, we have 
$
\left\langle w', x_i \right\rangle > -t\left(2 - \frac{i}{m^4}\right) + t\left(2 - \frac{i}{m^4}\right) \geq 0,
$
so $x_i$ is misclassified. Moreover, 
$
\left\langle w', x_m \right\rangle \geq -t\left(2 - \frac{1}{m^4}\right) + t\left(2 + \frac{1}{m}\right) > 0.
$
Thus, if $w'$ consists of equally many point from $(x_1, y_1), \ldots, (x_{m-1}, y_{m-1})$ and from $(x_m, y_m)$, then $(x_m, y_m)$ will be classified correctly, while any $(x_1, y_1), \ldots, (x_{m-1}, y_{m-1})$ will be misclassified. 

Furthermore, if $r_1, \ldots, r_{t+1} \in \{1, \ldots, m-1\}$ and 
$
w' = \sum_{q=1}^{t+1} y_{r_q} x_{r_q} + t y_m x_m = -\sum_{q=1}^{t+1} x_{r_q} + t x_m,
$
and $t \leq 2m-2$, then for any $i \in \{1, \ldots, m-1\}$, 
$
\left\langle w', x_i \right\rangle < -(t+1)(2 - \frac{2}{m^3}) + t(2 - \frac{i}{m^4}) < 0,
$
so $x_i$ is correctly classified. Moreover, 
$
\left\langle w', x_m \right\rangle \leq -(t+1)(2 - \frac{1}{m^3}) + t(2 + \frac{1}{m}) < 0,
$
since $t\leq 2m-2$. Thus, if $w'$ consists of one more point from $(x_1, y_1), \ldots, (x_{m-1}, y_{m-1})$ than the number of points from $(x_m, y_m)$, then any point in $(x_1, y_1), \ldots, (x_{m-1}, y_{m-1})$ is correctly classified and $(x_m, y_m)$ is misclassified.

Let $((x_{s_{1}},y_{s_{1}}),\ldots,(x_{s_{n}},y_{s_{n}}))$ be a training sequence of size $n$, i.e., $s_{1},\ldots,s_{n}\in \{1,\ldots,m\}$ (possibly with repetitions), and such that there exists $l\in \{1,\ldots,n\}$ with $y_{s_{l-1}}\neq y_{s_{l}}$, where $s_{0}=s_{n}$. If $y_{s_{1}} = -1$, then $w_{1} = -x_{s_{1}}$, and by the above for $ j\geq 2 $ , $w_{j-1}$ is updated each time $y_{s_{l-1}}\neq y_{s_{l}}$, provided $j-1$ is odd and $(j-2)/2+1 \leq 2m-2$, or $j-1$ is even and $(j-1)/2 \leq 2m-2$.

In the odd case, $w_{j-1} = \sum_{q=1}^{(j-2)/2+1} x_{r_{q}} + (j-2)/2 x_{m}$ is updated to $w_{j} = \sum_{q=1}^{(j-2)/2+1} x_{r_{q}} + ((j-2)/2+1) x_{m}$. In the even case, $w_{j-1} = \sum_{q=1}^{(j-1)/2} x_{r_{q}} + (j-1)/2 x_{m}$ is updated to $w_{j} = \sum_{q=1}^{(j-1)/2+1} x_{r_{q}} + (j-1)/2 x_{m}$.

Thus, we conclude that for a training sequence $((x_{s_{1}},y_{s_{1}}),\ldots,(x_{s_{n}},y_{s_{n}}))$ where there exists $l \in \{1,\ldots,n\}$ with $y_{s_{l-1}}\neq y_{s_{l}}$, $y_{s_{1}} = -1$ leads to update of $ w_{j-1} $  if $j-1$ is odd and $(j-2)/2+1 \leq 2m-2$, i.e., $j \leq 2(2m-3)+2$, or if $j-1$ is even and $(j-1)/2 \leq 2m-2$, i.e., $j \leq 2(2m-2)+1$. This results in at least $4m-4$ updates. Moreover, except for the first update, updates occur only when $y_{s_{l-1}}\neq y_{s_{l}}$.

We now consider a distribution $\cD$ over $(x_{1},y_{1}),\ldots,(x_{m},y_{m})$, which assigns $\cD(x_{m})=p$ and $\cD(x_{i})=(1-p)/(m-1)$ for $i\neq m$. For a training sequence $\rS = ((\rS_{1,1},\rS_{1,2}),\ldots,(\rS_{m,1},\rS_{m,2})) \sim \cD^{m}$, let $\mathbf{M}_{1} = |{i : \rS_{i,2}=1}| = \sum_{(x,y)\in \rS} \ind\{(x,y)=(x_{m},y_{m})\}$. Then, $\e_{\rS\sim\cD^{m}}[\mathbf{M}_{1}] = pm$ and $\text{Var}(\mathbf{M}_{1}) = (1-p)pm \leq pm$. Thus, by Chebyshev's inequality, we have that, with probability at least $1-1/4$ over $\rS$, it holds that $pm-\sqrt{4pm} < \mathbf{M}_{1} < pm+\sqrt{4pm}$. Let now $\rS_{B} = (({\rS_{B}}_{1,1}, {\rS_{B}}_{1,2}), \ldots, ({\rS_{B}}_{n,1}, {\rS_{B}}_{n,2}))$ be a sub-training sequence of size $n=0.02m$, drawn with replacement from $\rS$ (i.e., a bagging sample). Thus, by Chebyshev's inequality, we have that, with probability at least $1-1/4$ over $\rS_{B}$, it holds that $\frac{\mathbf{M}_{1}}{m}n-\sqrt{4\frac{\mathbf{M}_{1}}{m}n} < \sum_{(x,y) \in \rS_{B}} \ind\{(x,y) = (x_{m},y_{m})\} < \frac{\mathbf{M}_{1}}{m}n+\sqrt{4\frac{\mathbf{M}_{1}}{m}n}$.

Now, for $p=250/m$, we have that $pm-\sqrt{4pm}=250-\sqrt{1000} > 200$ and $pm+\sqrt{4pm}<300$, thus with probability at least $1-1/4$ over $\rS$, $200 < \mathbf{M}_{1} < 300$, which implies that, with probability at least $1-1/4$ over $\rS_{B}$, it holds that $0=200\cdot 0.02-\sqrt{4\cdot 200 \cdot 0.02} < \sum_{(x,y) \in \rS_{B}} \ind\{(x,y) = (x_{m},y_{m})\} \leq 300\cdot 0.02+\sqrt{4\cdot 300 \cdot 0.02} < 11$, where the first inequality follow by $ \frac{\mathbf{M}_{1}}{m}n-\sqrt{4\frac{\mathbf{M}_{1}}{m}n} $ being increasing in $ \frac{\mathbf{M}_{1}}{m}n $ for $ \frac{\mathbf{M}_{1}}{m}n \geq  4$ . Thus, with probability at least $(1-1/4)^2$ over both $\rS$ and one bootstrap sample of size $0.02m$, it holds that $1 \leq \sum_{(x,y) \in \rS_{B}} \ind\{(x,y) = (x_{m},y_{m})\} \leq 10$. Furthermore, with probability at least $1-\frac{10}{m} \geq 1-\frac{10}{2000}$ over $\rS_{B} = (({\rS_{B}}_{1,1}, {\rS_{B}}_{1,2}), \ldots, ({\rS_{B}}_{n,1}, {\rS_{B}}_{n,2}))$, the first example $({\rS_{B}}_{1,1}, {\rS_{B}}_{1,2})$ of $\rS_{B}$ is not equal to $({\rS_{B}}_{1,1}, {\rS_{B}}_{1,2}) \not= (x_{m}, y_{m})$. Thus, with probability at least $(1-1/4)^2(1-1/10) \geq 1/2$ over $\rS$ and $\rS_{B}$, it holds that $1 \leq \sum_{(x,y) \in \rS_{B}} \ind\{(x,y) = (x_{m},y_{m})\} \leq 10$ and that $({\rS_{B}}_{1,1}, {\rS_{B}}_{1,2}) \not= (x_{m}, y_{m})$. Consider such a realization $S_{B} = ((x_{s_{1}}, y_{s_{1}}), \ldots, (x_{s_{n}}, y_{s_{n}}))$. 

We recall that we concluded above that, for any training sequence of $(x_{1}, y_{1}), \ldots, (x_{m}, y_{m})$, especially $S_{B} = ((x_{s_{1}}, y_{s_{1}}), \ldots, (x_{s_{n}}, y_{s_{n}}))$, with $y_{s_{1}} = -1$ and such that there exists $y_{s_{l-1}} \not= y_{s_{l}}$, the perceptron run on $S_{B}$ has at least $4m-4$ updates, and except for the first update, the updates only happen when $y_{s_{l-1}} \not= y_{s_{l}}$. We conclude that each time the perceptron passes over $((x_{s_{2}}, y_{s_{2}}), \ldots, (x_{s_{n}}, y_{s_{n}}), (x_{s_{1}}, y_{s_{1}}))$, it makes at most $2\sum_{(x,y) \in S_{B}} \ind\{(x,y) = (x_{m}, y_{m})\} \leq 20$ updates and thus has to pass $\Omega(m)$ times over $((x_{s_{2}}, y_{s_{2}}), \ldots, (x_{s_{n}}, y_{s_{n}}), (x_{s_{1}}, y_{s_{1}}))$, where each pass takes $\Omega(m)$ operations, leading to $U_{T}(n)=U_{T}(0.02m) \geq \Omega(m^2)$ with probability at least $1-1/2$ over $\rS$ and $\rS_{B}$, implying that \cite{baggingoptimalPAClearner} takes at least $\Omega(m^{2})$ operations with probability at least $1/2$ over the random sample and the bagging step.

Now, for the training complexity of \cref{introductionmaintheorem}, we notice that since the $ \sign $ of hyperplanes in $ \mathbf{R}^{3} $  has VC-dimension $ 4,$ it suffices to consider sub training sequence of size $ 2200.$  Furthermore,  we observe that the vector $(1, -\sqrt{\frac{1}{2m}},0)$ is such that $\left\langle (1, -\sqrt{\frac{1}{2m}},0), x_{i} \right\rangle = \left\langle (1, -\sqrt{\frac{1}{2m}},0), (0,1-\frac{i}{m^{4}},1) \right\rangle = -\sqrt{\frac{1}{2m}} + \frac{i}{\sqrt{2}m^{4.5}} < -\sqrt{\frac{1}{4m}}$ for $i \in \{1, \ldots, m-1\}$, and $\left\langle (1, -\sqrt{\frac{1}{2m}},0), x_{m}\right\rangle = \left\langle (1, -\sqrt{\frac{1}{2m}},0), (\sqrt{\frac{1}{m}}, 1,1) \right\rangle \geq \sqrt{\frac{1}{16m}}$ by $m \geq 10.$ Furthermore, since $\left|\left|(1, -\sqrt{\frac{1}{2m}},0)\right|\right|{2} \leq 2$, we conclude from the above that the margin of $(x_{1}, y_{1}), \ldots, (x_{m}, y_{m})$ is at least $\gamma = \max_{w \in \mathbb{R}^{3}, \left|\left|w\right|\right|_{2} \leq 1} \min_{j \in \{1, \ldots, m\}} y_{j} \left\langle w, x_{j} \right\rangle \geq \sqrt{\frac{1}{64m}}$. Furthermore, the norm of $ x_{i} $ is at most  $\left|\left|x_{i}\right|\right|_{2} \leq 2$ for any $ i\in\{  1,\ldots,m\}  $. It is known by \cite{Novikoff1963ONCP} that the number of mistakes/updates the perceptron makes when passing over a data stream of any length, where the points have norm at most $M$, and the data stream can be separated by a hyperplane with margin $\gamma^{2}$, is at most $\frac{M^{2}}{\gamma^{2}}$. Thus, for any sub-training sequence of $x_{s_{1}}, \ldots, x_{s_{n}}$ of size $n = 2200$, the perceptron makes at most $O(m)$ passes over $x_{s_{1}}, \ldots, x_{s_{n}}$, before it make no mistakes when passing over $x_{s_{1}}, \ldots, x_{s_{n}}$, and since each pass takes $O(1)$ operations, $U_{T}(2200) = O(m)$. Thus, except for the small constant probability (considered in these examples), \cref{introductionmaintheorem} has a training complexity of $O\left(\ln{\left(\frac{m}{d}\right)} \ln{(m)} \left(m + U_{T}(2200) + mU_{I}\right)\right) = O\left(\ln{\left(\frac{m}{d}\right)} \ln{(m)} m\right)$, where we have used that performing inference over a hyperplane in $\mathbb{R}^{3}$ takes at most $U_{I} = O(1)$ operations.
\section{Proof of \cref{samplingfromrows}}\label{appendixefficentoptimalpaclearner}
In this appendix, we provide the proof of \cref{samplingfromrows}, which we restate here for convenience before giving the proof.
\samplingfromrows*

\begin{proof}[Proof of \cref{samplingfromrows}]
    Let $S\in (\cX\times \cY)^{*}$ and $|S|=m=6^{k}$ for $k\geq 1$. Let $\cS(S,T)$ for $T\in (\cX\times \cY)^{*}$ be the matrix whose rows are the sub training sequences produced by $\cS(S,T)$ \cref{alg:Subsample}. The matrix is given recursively by the following equation:  
    \begin{align*}
      \cS(S,T)=
      \begin{bmatrix}
        \cS(S_0,S_{1}\sqcup T)\\
        \cS(S_0,S_{2}\sqcup T)\\
        \cS(S_{0},S_{3}\sqcup T)\\
        \cS(S_{0},S_{4}\sqcup T)\\
        \cS(S_{0},S_{5}\sqcup T)
      \end{bmatrix}.
    \end{align*}
  We first notice that since $|S|=6^{k}$, we get by \cref{alg:Subsample}, Line~\ref{alg:Subsample:recursion} and Line~\ref{alg:Subsample:recursion2}, that   $\cS(S,\emptyset)$ will make 5 recursive calls, each of which will create $5$ new recursive calls, and so on for $k-2$ rounds. Counting from the top call, this totals $k$ recursive calls. At each recursion depth $k$, the only element left in the first argument of $\cS(\cdot,\cdot)$ is the first element of $S$, which, by Line~\ref{alg:Subsample:recursion3}, will, produce only one output per call. Thus the matrix $\cS(S,\emptyset)$ of sub training sequences will contain $5^{k}$ rows. 
  
Let for $ i\in\{  0,\ldots,5\}  $  $S_{i}^{1}$ be the $S_i$ created in \cref{alg:Subsample}, Line~\ref{alg:Subsample:recursion}, for $\cS(S,\emptyset)$. For $j\geq 2$ and $ i\in \{  0,\ldots,5\}  $  let $S_{i}^{j}$ be the $S_i$ created in \cref{alg:Subsample}, Line~\ref{alg:Subsample:recursion}, for $\cS(S_{0}^{j-1},\cdot)$, where $\cdot$ represents any sub training sequence, since Line~\ref{alg:Subsample:recursion} does not depend on the training sequences in the second argument, $S_{i}^{j}$ is well defined.

 Using this notation, we notice by Line~\ref{alg:Subsample:recursion2} that each $w\in \{1:5\}^{k}$ uniquely maps to a training sequence $S_{0}^{k+1}\sqcup(\sqcup_{j=1}^{k} S_{w_{j}}^{j} )$ of $\cS(S,\emptyset)$. Furthermore, since any training sequence in $ \cS(S,\emptyset) $ can be written like this by Line~\ref{alg:Subsample:recursion2} we also get that each $ S'\in \cS(S,\emptyset) $ has a corresponding $ w'\in \{ 1:5 \}^{k} $, such that $ S'= S_{0}^{k+1}\sqcup(\sqcup_{j=1}^{k} S_{w'_{j}}^{j} )$.   
  
Since Line~\ref{alg:Subsample:recursion} ensures that $S_0^{1}=S_0$ is the first $1/6$ of the training examples in $S$ (and similarly in subsequent rounds with $ S_{0}^{j} $ and $ S_{0}^{j-1} $), we conclude that $S_0^{j}=S[\{1:6^{k}/6^{j}\}]$. This implies, by Line~\ref{alg:Subsample:recursion2} (\cref{alg:Subsample} always  recurse on on $S_0^{j-1}$) and  Line~\ref{alg:Subsample:recursion} (\cref{alg:Subsample} always partitions $S_{0}^{j-1}$ in acceding order), that for $w_j\in\{1:5\}$

  \begin{align*}
    S_{w_{j}}^{j}=S_0^{j-1}[\{|(S_0^{j-1}|/6)w_{j}+1:(S_0^{j-1}|/6)(w_{j}+1)\}]\\
    =S_0^{j-1}[(\{6^{k}/6^{j})w_{j}+1:(6^{k}/6^{j})(w_{j}+1)\}]\\
    =\left(S[1:6^{k}/6^{j-1}]\right)[(\{6^{k}/6^{j})w_{j}+1:(6^{k}/6^{j})(w_{j}+1)\}]\\
    =S[\{(6^{k}/6^{j})w_{j}+1:(6^{k}/6^{j})(w_{j}+1)\}].
  \end{align*}

  Let $g':\{1:5\}^{k}\rightarrow {(\mathbb{R}^{2})}^{k}$ be the mapping which, for $w\in\{1:5\}^{k}$, defined by $g'(w)_{j}=[  w_j6^{k}/6^{j}+1,(w_j+1) 6^{k}/6^{j}]$ for $j=1,\ldots,k$.
  Thus, $ g' $ maps to the endpoint indices of the sub training sequence $S_{0}^{k+1}\sqcup(\sqcup_{j=1}^{k} S_{w_{j}}^{j} )$, which can be realized by using the relation $  S_{w_{j}}^{j}
  =S[\{w_{j}6^{k}/6^{j}+1:(w_{j}+1)6^{k}/6^{j}\}],$ and $g'(w)_{j,1}=w_{j}6^{k}/6^{j}+1$, and $g'(w)_{j,2}=(w_{j}+1)6^{k}/6^{j}$. Therefore, we can determine $S_{0}^{k+1}\sqcup(\sqcup_{j=1}^{k} S_{w_{j}}^{j} )$ using only $g'$, $w$, and $S$. 

  We now argue how this can be done efficiently: First, we read the length of $ S $ to obtain $ m=6^{k}, $ which takes $ O(m) $ operations. We then calculate $g'(w)$ as follows:  since we know $ m =6^{k}$, calculating the first entry of $g'(w)$ involves two divisions by $6$, two multiplications by $w_1$, and two additions of $1$. We also save $6^{k}/6$ so that when calculating the next entry of $g'(w)$, we only need to perform two additional divisions by $6$. Overall, this requires $O(k)=O(\ln{\left(m \right)})$ operations to calculate $g'(w)$. 
  Thus, finding the unique row in $ \cS(S,T) $ that corresponds to a given  $w\in \{ 1:5 \}^{k}$ takes at most $O(m)$ operations using $g'$, $w$, and $S$.

  We will denote $S[g'(w)]$ as the above efficient method for finding the unique row of $\cS(S,\emptyset)$ for a given $ w $  using $g'$, $w$, and $S$. 
  
  Finally, since a given $w\in \{ 1:5 \}^{k}$ corresponds to a unique row of $\cS(S,T)$, and since each $ S'\in\cS(S,T)  $ also has a unique $ w'\in \{ 1:5 \} $, we conclude that $ S[g'(w)] $ is a bijection into the rows of $ \cS(S,\emptyset) $, as claimed, which concludes the proof.    
  \end{proof}
 
\section{Proof of lemmas used in \cref{sec:optimalitya}}\label{appendixoptimallityofdeterministicprocess}
In this section, we prove \cref{lem1} and \cref{induktionlemma}.
We first prove \cref{lem1}, which we restate here for convenience.

\lemone*  
  The proof follows by the following two inequalities. The first inequality is due to \cite{rademacherDudley1978}, see e.g. Theorem 7.2 \cite{rademacherboundlecturenotes}.
  \begin{lemma}[\cite{rademacherboundlecturenotes} Theorem 7.2]\label{rademacherbound}
  There exists a universal constant $C>1$ such that for: $\cX$, an arbitrary set, and $\cH$, a class of functions $h: \cX \rightarrow\{-1,1\}$ of VC dimension $d$, and $ \cD $, a distribution over $ \cX $, the following holds with probability one over $ \rS=(\rx_1,\ldots,\rx_m)\sim \cD^{m}$ 
    \begin{align*}
      R_{\rS}(\cH):=\e_{\bm{\sigma}}\left[\sup_{h \in \cH} \frac{1}{m} \sum_{i=1}^{m} \bm{\sigma}_i  h(\rx_i)\right]\leq \e_{\bm{\sigma}}\left[\sup_{h \in \cH} \frac{1}{m} |\sum_{i=1}^{m} \bm{\sigma}_i  h(\rx_i)|\right]\leq \C \sqrt{\frac{2d}{m}}.
  \end{align*}
  \end{lemma}
  
  The second inequality uniformly bounds the error in terms of Rademacher complexity and empirical error. The inequality is due to \cite{Bartlett2003RademacherAG}.
  \begin{lemma}[\cite{boostingbookSchapireF12} Theorem 5.7]\label{lem:boostingbook}
  
  Let $\mathcal{F}$ be any family of functions $f: \cX \rightarrow[-1,+1]$. Let $\rS$ be a random sequence of $m$ points chosen independently according to some distribution $\cD$ over $\cX$. Then, with probability at least $1-\delta$ over $\rS$: $$
  \e_{\rx \sim \mathcal{D}}[f(\rx)] \leq \e_{\rx \sim S}[f(\rx)]+2 R_{\rS}(\mathcal{F})+\sqrt{\frac{2 \ln (2 / \delta)}{m}}
  $$
  for all $f \in \mathcal{F}$.
  
  \end{lemma}

  \begin{proof}[Proof of \cref{lem1}]
  We define $\phi(x)$ as the continuous interpolation between the $\gamma$-margin loss and the $\xi\gamma$-margin loss:
  \begin{align*}
    \phi(x)=\begin{cases}
    1 & \text{if } x \leq \gamma\\
    \frac{\xi}{\xi-1}-\frac{1}{(\xi-1)\gamma}x & \text{if } \gamma< x \leq \xi\gamma\\
    0 & \text{if } x > \xi\gamma.
    \end{cases}
  \end{align*}
  We notice that $\phi$ has a Lipschitz constant of $1/((\xi-1)\gamma)$. We now consider the function class  $\phi(\dlh)=\{(x,y)\rightarrow \phi(yf(x)):f\in\dlh \}$.
  Using \cref{lem:boostingbook}, we have with probability at least $1-\delta$ over the sample $\rS$, that 
  \begin{align*}\e_{(\rx,\ry) \sim \mathcal{D}_{c}}\left[\phi(\ry f(\rx))\right]\leq \e_{(\rx,\ry) \sim \rS}\left[\phi(\ry  f(\rx))\right]+2 R_{\rS}(\phi(\dlh))+\sqrt{\frac{2 \ln (2 / \delta)}{m}}\end{align*}
  for all $f\in \dlh$.
  
  Now, using the inequalities $\ind\{\ry f(\rx) \leq \gamma\} \leq \phi(\ry f(\rx))$ and $\phi(\ry f(\rx)) \leq \ind\{\ry f(\rx) \leq \xi\gamma\}$, we get $\Pr_{(\rx,\ry)\sim \mathcal{D}}[\ry f(\rx) \leq \gamma] \leq \e_{(\rx,\ry) \sim \mathcal{D}}\left[\phi(\ry f(\rx))\right] $ and $\e_{(\rx,\ry) \sim \rS }\left[\phi(\ry f(\rx))\right] \leq \Pr_{(\rx,\ry)\sim\rS}[\ry f(\rx) \leq \xi\gamma].$ Thus, using these together with the previous inequality, we obtain that with probability at least $1-\delta$ over the training sequence $\rS$:
  
  \begin{align}\label{eq:dlhproperty}
    \Pr_{(\rx,\ry)\sim \mathcal{D}}[\ry f(\rx) \leq \gamma] \leq \Pr_{(\rx,\ry)\sim\rS}[\ry f(\rx) \leq \xi \gamma] +2 R_{\rS}(\phi(\dlh))+\sqrt{\frac{2 \ln (2 / \delta)}{m}}.
  \end{align}
Thus if we can show that with probability $1$ over $\rS,$ we have $R_{\rS}(\phi(\dlh))\leq C\sqrt{\frac{2d}{((\xi-1)\gamma)^2m}}$ the claim follows.
  
  We now proceed to show this. Since $\phi$ has a Lipschitz constant of $1/((\xi-1)\gamma)$ and is continuous, and since the Rademacher complexity of the composition of a continuous Lipschitz function $ g $  and a function class $\cF$ (the function class $ g(\cF) $ ) is bounded by the Rademacher complexity of the function class $ \cF $,  multiplied by the Lipschitz constant (See e.g. \cite{rademacherboundlecturenotes} Proposition 6.2), we get:
  \begin{align*}R_{\rS}(\phi(\dlh))\leq \e_{\bm{\sigma}}\Big[\sup_{f \in \dlh} \frac{1}{m} \sum_{i=1}^{m} \bm{\sigma}_i  \sum_{l=1}^{t}\ry_{i}h_{f,l}(\rx_i)/t\Big]/((\xi-1)\gamma)=\frac{1}{(\xi-1)\gamma }R_{\rS}(\cH) ,\end{align*}
  where the last inequality follows from the $ \sup $, always being attainable by just considering one $ h\in \cH $ such that $ h_{f,l}=h $ for all $ l=1,\ldots,t $, as the sum of the $ t $ functions is normalized. Now, by \cref{rademacherbound} we can bound the Rademacher complexity of $\cH$ by $C\sqrt{2d/m}$, which, combined with \cref{eq:dlhproperty} concludes the proof of \cref{lem1}.
  
  \end{proof}
  
We now restate and prove \cref{induktionlemma}.

\induktionlemma*
  
\begin{proof}[Proof of \cref{induktionlemma}.]
    We first notice that for any $x\in \cX$ such that 
    \begin{align}\label{eq:highprobabilityeq1}
      \sum_{f\in \cA(\cS(S,T),r)} \frac{\ind\{\sum_{i=1}^{t} \ind\{h_{f,i}(x)=c(x)\}/t\geq 3/4\}}{|\cS(S,T)|}< 3/4,
    \end{align}
    and by $|S|\geq 6$ such that \cref{alg:Subsample} Line~\ref{alg:Subsample:recursion} implies $\sum_{f\in \cA(\cS(S,T),r)}=\sum_{j=1}^{5}\sum_{f\in \cA(j,T,r)}$, we have that 
    \begin{align}\label{eq:highprobabilityeq2}
     &\sum_{j=1}^{5} \frac{1}{5}\sum_{f\in \cA(j,T,r)} \frac{\ind\{\sum_{i=1}^{t} \ind\{h_{f,i}(x)=c(x)\}/t\geq 3/4\}}{|\cS(j,T)|}\nonumber
     \\
     =&
     \sum_{j=1}^{5} \sum_{f\in \cA(j,T,r)} \frac{\ind\{\sum_{i=1}^{t} \ind\{h_{f,i}(x)=c(x)\}/t\geq 3/4\}}{|\cS(S,T)|}
     < 3/4,
    \end{align}
and furthermore, for any $j\in\{1,\ldots,5\}$, we have by the above inequality that  
    
    \begin{align}\label{eq:highprobabilityeq3}
      &\sum_{f\in \cA(\cS(\rS,T))\backslash\cA(j,T,r)} \frac{\ind\{\sum_{i=1}^{t} \ind\{h_{f,i}(x)=c(x)\}/t\geq 3/4\}}{|\cA(\cS(\rS,T))\backslash\cA(j,T,r)|}\nonumber\\
      =
      &\sum_{\substack{l=1\\ l\not=j}}^{5} \frac{1}{4}\sum_{f\in \cA(l,T)} \frac{\ind\{\sum_{i=1}^{t} \ind\{h_{f,i}(x)=c(x)\}/t\geq 3/4\}}{|\cS(l,T)|}\nonumber
      \\
       \leq \frac{5}{4}&\sum_{l=1}^{5} \frac{1}{5}\sum_{f\in \cA(l,T)} \frac{\ind\{\sum_{i=1}^{t} \ind\{h_{f,i}(x)=c(x)\}/t\geq 3/4\}}{|\cS(l,T)|}
      <\frac{15}{16}.
     \end{align}
That is, for any $x$ satisfying \cref{eq:highprobabilityeq1}, by an average argument, \cref{eq:highprobabilityeq2} implies that one of the recursion calls $j$ fails to have $3/4$ of its majority voters $f\in \cA(j,T,r)$, having more than $3/4$ correct votes $h_{f,i}(x)=c(x)$, i.e., $\exists j \in\{1,2,3,4,5\}$ such that 
    \begin{align}\label{eq:highprobabilityeq4}
      \sum_{f\in \cA(j,T,r)} \frac{\ind\{\sum_{i=1}^{t} \ind\{h_{f,i}(x)=c(x)\}/t\geq 3/4\}}{|\cS(j,T)|}<3/4. 
    \end{align}
furthermore, for any $j$, \cref{eq:highprobabilityeq3} shows that at least a $1/16$ fraction of the majority voters $f$ in the other recursion calls $\cA(\cS(\rS,T))\backslash\cA(j,T,r)$ fail to have more than $3/4$ correct votes $h_{f,i}(x)=c(x)$, i.e., $\forall j\in\{1,2,3,4,5\}$ we have that
    \begin{align}\label{eq:highprobabilityeq5}
      &\sum_{f\in \cA(\cS(\rS,T))\backslash\cA(j,T,r)} \frac{\ind\{\sum_{i=1}^{t} \ind\{h_{f,i}(x)=c(x)\}/t< 3/4\}}{|\cA(\cS(\rS,T))\backslash\cA(j,T,r)|}\geq 1/16.
    \end{align} 
    
Thus, if we now let $\rJ$ be uniformly chosen over $\{1,2,3,4,5\}$, and $\rf$, be a uniformly chosen majority voter from $\cA(\cS(S,T),r) \backslash \cA(\rJ,T,r)$, the above \cref{eq:highprobabilityeq4}(which is $\exists j\in\{1,2,3,4,5\}$) and \cref{eq:highprobabilityeq5}(which is $\forall j\in\{1,2,3,4,5\}$), combined with the law of conditional probability, implies that for any $x\in\cX$ satisfying \cref{eq:highprobabilityeq1}, we have that 
    \begin{align*}
    &\p_{\rJ,\rf}\bigg[
    \sum_{i=1}^{t} \ind\{h_{\rf,i}(x)=c(x)\}/t< 3/4,
    \sum_{f\in \cA(\rJ,T,r)}\negmedspace \negmedspace \negmedspace \negmedspace  \frac{\ind\{\sum_{i=1}^{t} \ind\{h_{f,i}(x)=c(x)\}/t\geq 3/4\}}{|\cS(\rJ,T)|}<3/4\bigg]
    \\
    \geq&
    \p_{\rJ,\rf}\bigg[
    \sum_{i=1}^{t} \ind\{h_{\rf,i}(x)=c(x)\}/t< 3/4 \Big|
    \sum_{f\in \cA(\rJ,T,r)} \negmedspace \negmedspace \negmedspace \negmedspace \frac{\ind\{\sum_{i=1}^{t} \negmedspace  \ind\{h_{f,i}(x)=c(x)\}/t\geq 3/4\}}{|\cS(\rJ,T)|}<3/4\bigg]\frac{1}{5}
    \\\geq&
    \frac{1}{16}\frac{1}{5}=\frac{1}{80},
    \end{align*}
and since we showed the above for any $x$ satisfying \cref{eq:highprobabilityeq1}, we conclude that 
    \begin{align*}
    80\p_{\rJ,\rf}\left[
    \sum_{i=1}^{t} \ind\{h_{\rf,i}(x)=c(x)\}/t< 3/4,
    \sum_{f\in \cA(\rJ,T,r)} \frac{\ind\{\sum_{i=1}^{t} \ind\{h_{f,i}(x)=c(x)\}/t\geq 3/4\}}{|\cS(\rJ,T)|}<3/4\right]
    \\
    \geq 
    \ind \left\{  
      \sum_{f\in \cA(\cS(S,T),r)} \frac{\ind\{\sum_{i=1}^{t} \ind\{h_{f,i}(x)=c(x)\}/t\geq 3/4\}}{|\cS(S,T)|}
      <3/4\right\}.
    \end{align*}
Thus, we conclude by switching the order of expectation that
    \begin{align}\label{eq:highprobabilityeq6}
      80\e_{\rJ,\rf}
      \left[
        \p_{x\sim \cD}
        \left[
    \sum_{i=1}^{t} \ind\{h_{\rf,i}(x)=c(x)\}/t< \frac{3}{4},\negmedspace \negmedspace \negmedspace \negmedspace 
    \sum_{f\in \cA(\rJ,T,r)} \negmedspace \negmedspace \frac{\ind\{\sum_{i=1}^{t} \ind\{h_{f,i}(x)=c(x)\}/t\geq \frac{3}{4}\}}{|\cS(\rJ,T)|}
    <\frac{3}{4}\right]
    \right]
    \nonumber
    \\
    \geq 
    \p_{x\sim \cD} \left[  
      \sum_{f\in \cA(\cS(S,T),r)} \frac{\ind\{\sum_{i=1}^{t} \ind\{h_{f,i}(x)=c(x)\}/t\geq 3/4\}}{|\cS(S,T)|}
       <3/4\right]
    \end{align}
Now, using that $\rf$ is a uniformly chosen majority voter from $\cA(\cS(S,T),r) \backslash \cA(\rJ,T,r)$ and $\rJ\in\{1,2,3,4,5\}$, we conclude that 
    
    \begin{align*}
      \p_{x\sim \cD}
      \bigg[
    \sum_{i=1}^{t} \ind\{h_{\rf,i}(x)=c(x)\}/t< 3/4,
    \sum_{f\in \cA(\rJ,T,r)} \frac{\ind\{\sum_{i=1}^{t} \ind\{h_{f,i}(x)=c(x)\}/t\geq 3/4\}}{|\cS(\rJ,T)|}
    <\frac{3}{4}\bigg]
    \\
    \leq   \negmedspace \negmedspace  \negmedspace   \negmedspace \negmedspace \max_{\stackrel{}{\tilde{f}\in \cA(\cS(S,T),r) \backslash \cA(\rJ,T,r)}} \negmedspace \negmedspace  \negmedspace \negmedspace \negmedspace \negmedspace \negmedspace  \negmedspace \negmedspace \negmedspace\negmedspace \p_{x\sim \cD}
    \bigg[
      \negmedspace \sum_{i=1}^{t}\negmedspace  \ind\{h_{\tilde{f},i}(x)=c(x)\}/t\negmedspace <\negmedspace \frac{3}{4},
  \negmedspace \negmedspace  \negmedspace  \negmedspace \negmedspace  \negmedspace \negmedspace \negmedspace \negmedspace\sum_{\text{\quad\quad}f\in \cA(\rJ,T,r)} \negmedspace \negmedspace \negmedspace \negmedspace\negmedspace \negmedspace  \negmedspace \negmedspace \negmedspace \negmedspace \frac{\ind\{\sum_{i=1}^{t}\negmedspace  \ind\{h_{f,i}(x)=c(x)\}/t\geq \frac{3}{4}\}}{|\cS(\rJ,T)|}
  <\frac{3}{4}\bigg]
    \\
    \leq   \negmedspace \negmedspace  \negmedspace   \negmedspace \negmedspace \max_{\stackrel{j\in\{ 1:5 \}}{\tilde{f}\in \cA(\cS(S,T),r) \backslash \cA(j,T,r)}} \negmedspace \negmedspace  \negmedspace \negmedspace \negmedspace \negmedspace \negmedspace  \negmedspace \negmedspace \negmedspace\negmedspace \p_{x\sim \cD}
    \bigg[
      \negmedspace \sum_{i=1}^{t}\negmedspace  \ind\{h_{\tilde{f},i}(x)=c(x)\}/t\negmedspace <\negmedspace \frac{3}{4},
  \negmedspace \negmedspace  \negmedspace  \negmedspace \negmedspace  \negmedspace \negmedspace \negmedspace \negmedspace\sum_{\text{\quad\quad}f\in \cA(j,T,r)} \negmedspace \negmedspace \negmedspace \negmedspace\negmedspace \negmedspace  \negmedspace \negmedspace \negmedspace \negmedspace \frac{\ind\{\sum_{i=1}^{t}\negmedspace  \ind\{h_{f,i}(x)=c(x)\}/t\geq \frac{3}{4}\}}{|\cS(j,T)|}
  <\frac{3}{4}\bigg].
    \end{align*}
Finally, combining the above with \cref{eq:highprobabilityeq6}, we conclude that
    \begin{align*}
        &\p_{x\sim \cD} \bigg[  
          \sum_{f\in \cA(\cS(S,T),r)} \frac{\ind\{\sum_{i=1}^{t} \ind\{h_{f,i}(x)=c(x)\}/t\geq 3/4\}}{|\cS(S,T)|}
           <3/4\bigg]
        \\
        \leq \negmedspace \negmedspace  \negmedspace   \negmedspace \negmedspace \max_{\stackrel{j\in\{ 1:5 \}}{\tilde{f}\in \cA(\cS(S,T),r) \backslash \cA(j,T,r)}} \negmedspace \negmedspace  \negmedspace \negmedspace \negmedspace \negmedspace \negmedspace  \negmedspace \negmedspace \negmedspace\negmedspace \negmedspace80&\p_{x\sim \cD}
        \bigg[
          \negmedspace \sum_{i=1}^{t}\negmedspace  \ind\{h_{\tilde{f},i}(x)=c(x)\}/t\negmedspace <\negmedspace \frac{3}{4},
      \negmedspace \negmedspace  \negmedspace  \negmedspace \negmedspace  \negmedspace \negmedspace \negmedspace \negmedspace\sum_{\text{\quad\quad}f\in \cA(j,T,r)} \negmedspace \negmedspace \negmedspace \negmedspace\negmedspace \negmedspace  \negmedspace \negmedspace \negmedspace \negmedspace \frac{\ind\{\sum_{i=1}^{t}\negmedspace  \ind\{h_{f,i}(x)=c(x)\}/t\geq \frac{3}{4}\}}{|\cS(j,T)|}
      \negmedspace<\negmedspace\frac{3}{4}\bigg],
      \end{align*}
      which finishes the proof. 
    \end{proof}

\section{Proof of \cref{AdaBoostSampleMarginLemma3}}\label{appendixpropertiesadaboossample}
We now provide the proof of \cref{AdaBoostSampleMarginLemma3}. For convenience, we restate it here before proceeding with its proof.

\AdaBoostSampleMarginLemmathree*

\begin{proof}[Proof of \cref{AdaBoostSampleMarginLemma3}]
    To show the above bound on $\sum_{x\in S} \ind\{\sum_{j=1}^{t} h_{f,j}(x)c(x)/t\leq \theta \} $, assuming $\sum \ind_{\alphar_i>0}\geq t$, we first notice that if $\sum \ind_{\alphar_i>0}\geq t$, then $f$ in line Line~\ref{alg:AdaBoostsample:output1} is exactly the voting classifiers of the first $t$ hypotheses returned by the $ \erm $  learner, with margin $\gamma_{f,j}\geq \gamma$. This is ensured by Line~\ref{alg:AdaBoostsamplenonzeroweight}, Line~\ref{alg:AdaBoostSamplealphaset} and Line~\ref{alg:AdaBoostSampleearlystop}, which guarantee that $f$ will never be a combination of more than the first $t$ hypotheses with margins $\gamma_{f,j}\geq \gamma$. 
    
    Moreover, since Line~\ref{alg:AdaBoostSamplezeroweight} sets $\alpha_{i}=0$ for any hypothesis $h_i$ with margin $\gamma_{i}<\gamma$ and sets $\sum_{j=1}^{i}\ind_{\alpha_j>0}=\sum_{j=1}^{i-1}\ind_{\alpha_j>0}  $, the updates in Line~\ref{alg:AdaBoostSampleupdatestart} to Line~\ref{alg:AdaBoostSamplegetweights}, correspond to setting $D_{i+1}=D_{i}$, i.e. restarting a step in AdaBoost until a hypothesis with margin $\gamma_{i}\geq \gamma$ under $D_{i}$ is found or the for loop ends. Further Line~\ref{alg:AdaBoostSamplezeroweight} also in this case sets the counter $\sum_{j=1}^{i}\ind_{\alpha_j>0}=\sum_{j=1}^{i-1}\ind_{\alpha_j>0}$ to what it was in the previous round, i.e. skips this fail round, such that $ f $ will be a combination of $ t $ hypotheses if it is outputted.  
    
    Thus in the case $\sum \ind_{\alphar_i>0}\geq t$ happens $f$ is the outcome of AdaBoost with a fixed learning rate $\alpha_{f,j}=\frac{1}{2}\ln{\left(\frac{1/2+\gamma} {1/2-\gamma}\right)}-\frac{1}{2}\ln{(\frac{1+\theta}{1-\theta} )}$, stopped after $t$ rounds of boosting, within each round having received a hypothesis $ h_{f,j} $ with margin $ \gamma_{f,j}\geq \gamma$. Thus we now analysis it like that, with its $ D_{f,1},\ldots,D_{f,t+1} $ corresponding distributions see Line~\ref{alg:AdaBoostSampleupdatel}, we never explicitly in \cref{alg:AdaBoostSample} define $  D_{f,t+1}$ but it refers to the distribution that $ h_{f,t}$ updates $ D_{f,t} $ to in \cref{alg:AdaBoostSamplegetweights}. The analysis follows steps from \cite{boostingbookSchapireF12}[page 55, page 114]. By Line~\ref{alg:AdaBoostSamplegetweights}, Line~\ref{alg:AdaBoostSamplenormalizing} and Line~\ref{alg:AdaBoostSamplesetDone}, which implies $ D_{f,1}(i)=D_{1}=1/m $ we have that 
    \begin{align*}
      D_{f,t+1}(i)= D_{f,t}(i) \exp{(-\alpha_{f,j}h_{j,t}(S_{i,1})S_{i,2} )}/Z_{f,t} 
      \\
      =\frac{D_{1}(i)\exp{(- \sum_{j=1}^{t}\alpha_{f,j}h_{f,j}(S_{i,1})S_{i,2})}}{\prod_{l=1}^{t}Z_{f,l} }
      \\
      =\frac{\exp{( -\sum_{j=1}^{t}\alpha_{f,j}h_{f,j}(S_{i,1})S_{i,2})}}{(m\cdot\prod_{l=1}^{t}Z_{f,l})} ,
    \end{align*}
 where $ S_{i,1} \in \cX$  and $S_{i,2}  \in\cY$ denote respectively the point and label of the $ i $-th training example. Since $ D_{f,t+1}(i) $ is a probability distribution so sums to $ 1 $, we have that 
 
 \begin{align*}
  \sum_{(x,y)\in S} \frac{\exp{\left( -y\sum_{j=1}^{t} \alpha_{f,j}h_{f,j}(x)\right)}}{m} =\sum_{i=1}^{m}\frac{\exp{\left( -\sum_{j=1}^{t}\alpha_{f,j}h_{f,j}(S_{i,1})S_{i,2}\right)}}{m}= \prod_{j=1}^{t}Z_{f,j}.
 \end{align*} 
 Now, using that by Line~\ref{alg:AdaBoostsamplenonzeroweight} we have for $ j\in\{ 1:t \} $,  that $ h_{f,j} $ has error  $ \epsilon_{f,j}=1/2 -\gamma_{f,j}$ under $ D_{f,j} $, and combining this with Line~\ref{alg:AdaBoostSamplegetweights} and Line~\ref{AdaBoostSamplesetnormalization} we get:
    \begin{align}\label{eq:adaboosttypeargument1}
      Z_{f,j}&=\sum_{i=1}^{m} D_{f,j+1}(i)=\sum_{i=1}^{m}D_{f,j}(i) \exp{(-\alpha_{f,j}h_{f,j}(S_{i,1})S_{i,2} )} \\
      &=\sum_{h_{f,j}(S_{i,1})=S_{i,2}}\negmedspace\negmedspace \negmedspace\negmedspace D_{f,j}(i) \exp{(-\alpha_{f,j}h_{f,j}(S_{i,1})S_{i,2} )} +\negmedspace\negmedspace\negmedspace\negmedspace\sum_{h_{f,j}(S_{i,1})\not=S_{i,2}}\negmedspace\negmedspace \negmedspace\negmedspace D_{f,j}(i) \exp{(-\alpha_{f,j}h_{f,j}(S_{i,1})S_{i,2} )}\nonumber\\
      &=(1/2+\gamma_{f,j})\exp(-\alpha_{f,j})+(1/2-\gamma_{f,j})\exp(\alpha_{f,j})\nonumber
    \end{align}          
    Now using that $ f=\sum_{j=1}^{t} \frac{\alpha_{f,j}h_{f,j}}{\sum_{l=1}^{t}{\alpha_{f,l}}} $ and that $ \alpha_{f,l}= \frac{1}{2}\ln{\left(\frac{1+2\gamma}{1-2\gamma}\right)- \frac{1}{2}\ln{\left(\frac{1+\theta}{1-\theta} \right)}}>0$ since $ \theta< 2\gamma $ and the above relations $\sum_{(x,y)\in S}\exp{\left( -y\sum_{j=1}^{t} \alpha_{f,j}h_{f,j}(x)\right)}/m $ and \cref{eq:adaboosttypeargument1} we get that
      \begin{align*}
    \p_{(\rx,\ry)\sim S}\left[\ry f(\rx)\leq \theta \right]
    &\leq 
    \p_{(\rx,\ry)\sim S}\left[\ry \sum_{j=1}^{t} \alpha_{f,j}h_{f,j}(x)\leq \theta\sum_{j=1}^{t}{\alpha_{f,j}} \right]\\
    &\leq
     \sum_{(x,y)\in S}\exp{\left( \theta\sum_{j=1}^{t}{\alpha_{f,j}}-y\sum_{j=1}^{t} \alpha_{f,j}h_{f,j}(x)\right)}/m 
    \\
    &=\exp{\left( \theta\sum_{j=1}^{t}{\alpha_{f,j}}\right)}\sum_{(x,y)\in S}\exp{\left( -y\sum_{j=1}^{t} \alpha_{f,j}h_{f,j}(x)\right)}/m 
    \\
    &=\exp{\left( \theta\sum_{j=1}^{t}{\alpha_{f,j}}\right)}\prod_{j=1}^{t}Z_{f,j}=
    \prod_{j=1}^{t}Z_{f,j}\exp{\left( \theta{\alpha_{f,j}}\right)}
    \\
    &=\prod_{j=1}^{t} \Big((1/2+\gamma_{f,j})\exp(-\alpha_{f,j})+(1/2-\gamma_{f,j})\exp(\alpha_{f,j})\Big)\exp{\left( \theta{\alpha_{f,j}}\right)}
    \\
    &=
    \prod_{j=1}^{t} \left[(1/2+\gamma_{f,j})e^{(\theta-1)\alpha_{f,j}}+(1/2-\gamma_{f,j})e^{(\theta+1)\alpha_{f,j}}\right].
      \end{align*}
    As $\left(e^{(\theta-1)\alpha_{f,j}}-e^{(\theta+1)\alpha_{f,j}}\right)\leq 0$ since  $ \alpha_{f,l}= \frac{1}{2}\ln{\left(\frac{1+2\gamma}{1-2\gamma}\right)- \frac{1}{2}\ln{\left(\frac{1+\theta}{1-\theta} \right)}}>0$  (we assumed $ \theta< 2\gamma $)  we have that $\gamma_{f,j}\left(e^{(\theta-1)\alpha_{f,j}}-e^{(\theta+1)\alpha_{f,j}}\right)$ is a decreasing function of $\gamma_{f,j}$, which implies since we have $\gamma_{f,j}\geq \gamma$,  that  $\gamma_{f,j}\left(e^{(\theta-1)\alpha_{f,j}}-e^{(\theta+1)\alpha_{f,j}}\right)$ is less than or equal to $\gamma\left(e^{(\theta-1)\alpha_{f,j}}-e^{(\theta+1)\alpha_{f,j}}\right)$ 
    and combining this with the above inequality we get that  
    \begin{align*}
      \p_{(\rx,\ry)\sim S}\left[\ry f(\rx)\leq \theta \right]
      \leq 
      \prod_{j=1}^{t} \left[e^{(\theta-1)\alpha_{f,j}}(1/2+\gamma)+e^{(\theta+1)\alpha_{f,j}}(1/2-\gamma)\right],
        \end{align*}
    as for all $ j\in\{ 1:t \} $, we have by Line~\ref{alg:AdaBoostSamplealphaset}  that $\alpha=\alpha_{f,j}=\frac{1}{2}\ln{\left(\frac{1+2\gamma}{1-2\gamma}\right)- \frac{1}{2}\ln{\left(\frac{1+\theta}{1-\theta} \right)}}>0$ we conclude that
    \begin{align*}
      \p_{(\rx,\ry)\sim S}\left[\ry f(\rx)\leq \theta \right]
      \leq 
      \prod_{j=1}^{t} \left[e^{(\theta-1)\alpha}(1/2+\gamma)+e^{(\theta+1)\alpha}(1/2-\gamma)\right]\\
      =\big(\exp{((\theta-1)\alpha)}(1/2+\gamma)+\exp{((\theta+1)\alpha)}(1/2-\gamma)\big)^{t}.
        \end{align*}
    Further for the values $ \theta=3/4 $ and $ \gamma=9/20 $, which satisfies $ 3/8 =\theta/2 < \gamma=9/20 $  we get by numerical calculations that 
    \begin{align*}
      &\exp{\left(\left(\theta-1\right)\left(\frac{1}{2}\ln{\left(\frac{1+2\gamma}{1-2\gamma}\right)- \frac{1}{2}\ln{\left(\frac{1+\theta}{1-\theta} \right)}}\right)\right)}\left(\frac{1}{2}+\gamma\right)
      \\&+\exp{\left(\left(\theta+1\right)\left(\frac{1}{2}\ln{\left(\frac{1+2\gamma}{1-2\gamma}\right)- \frac{1}{2}\ln{\left(\frac{1+\theta}{1-\theta} \right)}}\right)\right)}\left(\frac{1}{2}-\gamma\right)\\
      &= \exp{\left(\left(\frac{3}{4}-1\right)\left(\frac{1}{2}\ln{\left(\frac{1+2\cdot\frac{9}{20}}{1-2\cdot\frac{9}{20}}\right)- \frac{1}{2}\ln{\left(\frac{1+\frac{3}{4}}{1-\frac{3}{4}} \right)}}\right)\right)}\left(\frac{1}{2}+\frac{9}{20}\right)
      \\&+\exp{\left(\left(\frac{3}{4}+1\right)\left(\frac{1}{2}\ln{\left(\frac{1+2\cdot\frac{9}{20}}{1-2\cdot\frac{9}{20}}\right)- \frac{1}{2}\ln{\left(\frac{1+\frac{3}{4}}{1-\frac{3}{4}} \right)}}\right)\right)}\left(\frac{1}{2}-\frac{9}{20}\right)
      \\&\leq 96/100=24/25
    \end{align*}
    which concludes the proof.
  \end{proof}

\end{document}